\documentclass[12pt,letterpaper]{article}

\def\showauthornotes{0}
\def\showtableofcontents{1}
\def\showkeys{0}
\def\showdraftbox{0}
\def\showcolorlinks{1}
\def\usemicrotype{1}
\def\showfixme{0}

\def\writemode{0}
\def\arxivmode{0}

\usepackage{etex}

\usepackage[l2tabu, orthodox]{nag}

\usepackage{xspace,enumerate}

\usepackage[dvipsnames]{xcolor}

\usepackage[T1]{fontenc}
\usepackage[full]{textcomp}

\usepackage[american]{babel}

\usepackage{mathtools}

\usepackage{amsthm}

\newtheorem{theorem}{Theorem}[section]
\newtheorem*{theorem*}{Theorem}

\newtheorem{proposition}[theorem]{Proposition}
\newtheorem*{proposition*}{Proposition}
\newtheorem{lemma}[theorem]{Lemma}
\newtheorem*{lemma*}{Lemma}
\newtheorem{corollary}[theorem]{Corollary}
\newtheorem*{conjecture*}{Conjecture}

\newtheorem*{fact*}{Fact}

\newtheorem*{hypothesis*}{Hypothesis}

\theoremstyle{definition}
\newtheorem{definition}[theorem]{Definition}

\theoremstyle{remark}

\newtheorem*{claim*}{Claim}
\newtheorem{remark}[theorem]{Remark}
\newtheorem*{remark*}{Remark}

\newtheorem*{observation*}{Observation}

\newtheorem*{algorithm*}{Algorithm}

\ifnum\writemode=1
\usepackage[
letterpaper,
top=1.2in,
bottom=1.2in,
left=1in,
right=1in]{geometry}

\fi

\ifnum\writemode=0
\usepackage[
letterpaper,
top=0.7in,
bottom=0.9in,
left=1in,
right=1in]{geometry}
\fi

\ifnum\arxivmode=0
\usepackage{newpxtext} %
\usepackage{textcomp} %
\usepackage[varg,bigdelims]{newpxmath}
\usepackage[scr=rsfso]{mathalfa}%
\usepackage{bm} %
\let\mathbb\varmathbb
\fi

\ifnum\arxivmode=1
\usepackage[varg]{pxfonts} %
\usepackage{textcomp} %
\usepackage[scr=rsfso]{mathalfa}%
\usepackage{bm} %
\fi

\ifnum\showkeys=1
\usepackage[color]{showkeys}
\fi

\ifnum\showcolorlinks=1
\usepackage[
pagebackref,
colorlinks=true,
urlcolor=blue,
linkcolor=blue,
citecolor=OliveGreen,
]{hyperref}
\fi

\ifnum\showcolorlinks=0
\usepackage[
pagebackref,
colorlinks=false,
pdfborder={0 0 0}
]{hyperref}
\fi

\usepackage[capitalise,nameinlink]{cleveref}
\crefname{lemma}{Lemma}{Lemmas}
\crefname{definition}{Definition}{Definitions}

\newcommand{\Sref}[1]{\hyperref[#1]{\S\ref*{#1}}}

\usepackage{nicefrac}

\ifnum\usemicrotype=1
\usepackage{microtype}
\fi

\ifnum\showauthornotes=1
\newcommand{\Authornote}[2]{{\sffamily\small\color{red}{[#1: #2]}}}
\newcommand{\Authornotecolored}[3]{{\sffamily\small\color{#1}{[#2: #3]}}}
\newcommand{\Authorcomment}[2]{{\sffamily\small\color{gray}{[#1: #2]}}}
\newcommand{\Authorstartcomment}[1]{\sffamily\small\color{gray}[#1: }

\newcommand{\Authorfnote}[2]{\footnote{\color{red}{#1: #2}}}
\newcommand{\Authorfixme}[1]{\Authornote{#1}{\textbf{??}}}
\newcommand{\Authormarginmark}[1]{\marginpar{\textcolor{red}{\fbox{\Large #1:!}}}}
\else
\newcommand{\Authornote}[2]{}
\newcommand{\Authornotecolored}[3]{}
\newcommand{\Authorcomment}[2]{}
\newcommand{\Authorstartcomment}[1]{}

\newcommand{\Authorfnote}[2]{}
\newcommand{\Authorfixme}[1]{}
\newcommand{\Authormarginmark}[1]{}
\fi

\newcommand{\Dnote}{\Authornote{D}}

\definecolor{forestgreen(traditional)}{rgb}{0.0, 0.27, 0.13}

\ifnum\showfixme=0

\fi

\usepackage{boxedminipage}

\newcommand{\paren}[1]{(#1)}
\newcommand{\Paren}[1]{\left(#1\right)}

\newcommand{\abs}[1]{\lvert#1\rvert}

\newcommand{\card}[1]{\lvert#1\rvert}

\newcommand{\norm}[1]{\lVert#1\rVert}
\newcommand{\Norm}[1]{\left\lVert#1\right\rVert}

\newcommand{\iprod}[1]{\langle#1\rangle}

\newcommand{\Esymb}{\mathbb{E}}
\newcommand{\Psymb}{\mathbb{P}}

\DeclareMathOperator*{\E}{\Esymb}

\DeclareMathOperator*{\ProbOp}{\Psymb}

\renewcommand{\Pr}{\ProbOp}

\newcommand{\textparen}[1]{\text{(#1)}}

\ifx\because\undefined
\newcommand{\because}[1]{\textparen{because #1}}
\else
\renewcommand{\because}[1]{\textparen{because #1}}
\fi

\newcommand{\super}[2]{#1^{\paren{#2}}}

\newcommand{\vbig}{\vphantom{\bigoplus}}

\newcommand{\from}{\colon}

\newcommand\bdot\bullet

\ifx\mathds\undefined %

\else
}
\fi

\DeclareMathOperator{\val}{val}

\newcommand{\R}{\mathbb R}

\newcommand{\cA}{\mathcal A}

\newcommand{\cP}{\mathcal P}

\newcommand{\cR}{\mathcal R}

\renewcommand{\leq}{\leqslant}
\renewcommand{\le}{\leqslant}
\renewcommand{\geq}{\geqslant}
\renewcommand{\ge}{\geqslant}

\ifnum\showdraftbox=1
\newcommand{\draftbox}{\begin{center}
  \fbox{%
    \begin{minipage}{2in}%
      \begin{center}%
          \Large\textsc{Working Draft}\\%
        Please do not distribute%
      \end{center}%
    \end{minipage}%
  }%
\end{center}
\vspace{0.2cm}}
\else
\newcommand{\draftbox}{}
\fi

\let\epsilon=\varepsilon

\numberwithin{equation}{section}

\newcommand\MYcurrentlabel{xxx}

\newcommand{\MYstore}[2]{%
  \global\expandafter \def \csname MYMEMORY #1 \endcsname{#2}%
}

\newcommand{\MYload}[1]{%
  \csname MYMEMORY #1 \endcsname%
}

\newcommand{\MYnewlabel}[1]{%
  \renewcommand\MYcurrentlabel{#1}%
  \MYoldlabel{#1}%
}

\newcommand{\MYdummylabel}[1]{}

\newcommand{\torestate}[1]{%
  \let\MYoldlabel\label%
  \let\label\MYnewlabel%
  #1%
  \MYstore{\MYcurrentlabel}{#1}%
  \let\label\MYoldlabel%
}

\newcommand{\restatetheorem}[1]{%
  \let\MYoldlabel\label
  \let\label\MYdummylabel
  \begin{theorem*}[Restatement of \cref{#1}]
    \MYload{#1}
  \end{theorem*}
  \let\label\MYoldlabel
}

\newcommand{\restatelemma}[1]{%
  \let\MYoldlabel\label
  \let\label\MYdummylabel
  \begin{lemma*}[Restatement of \cref{#1}]
    \MYload{#1}
  \end{lemma*}
  \let\label\MYoldlabel
}

\newcommand{\restateprop}[1]{%
  \let\MYoldlabel\label
  \let\label\MYdummylabel
  \begin{proposition*}[Restatement of \cref{#1}]
    \MYload{#1}
  \end{proposition*}
  \let\label\MYoldlabel
}

\newcommand{\restatefact}[1]{%
  \let\MYoldlabel\label
  \let\label\MYdummylabel
  \begin{fact*}[Restatement of \cref{#1}]
    \MYload{#1}
  \end{fact*}
  \let\label\MYoldlabel
}

\newcommand{\restate}[1]{%
  \let\MYoldlabel\label
  \let\label\MYdummylabel
  \MYload{#1}
  \let\label\MYoldlabel
}

\newcommand{\addreferencesection}{
  \phantomsection
  \addcontentsline{toc}{section}{References}
}

\newcommand{\e}{\epsilon}

\let\origparagraph\paragraph
\renewcommand{\paragraph}[1]{\origparagraph{#1.}}

\allowdisplaybreaks

\usepackage{paralist}

\usepackage{comment}
\usepackage{mdframed}
\newmdtheoremenv{framedalgorithm}[theorem]{Algorithm}

\newenvironment{keywords}{\noindent\textbf{Keywords:}}{}

\usepackage{enumitem}

\usepackage{relsize}
\usepackage[font=footnotesize]{caption}
\usepackage{appendix}

\DeclareMathOperator{\Id}{\mathrm{Id}}

\DeclareUrlCommand\email{}

\DeclareMathOperator*{\pE}{\tilde{\mathbb E}}

\newcommand*{\transpose}[1]{{#1}{}^{\mkern-4mu\intercal}}

\newcommand*{\dyad}[1]{#1#1{}^{\mkern-4mu\intercal}}

\title{Exact tensor completion with sum-of-squares}

\author{%
  Aaron Potechin\thanks{Institute for Advanced Study.
    Supported by the Simons Collaboration for Algorithms and Geometry and by the NSF under agreement No. CCF-1412958.
    Part of this work was done while at Cornell University. }
  \and
  David Steurer\thanks{Institute for Advanced Study and Cornell University,
    \protect\email{dsteurer@cs.cornell.edu}.
    Supported by a Microsoft Research Fellowship, a Alfred P. Sloan Fellowship, NSF awards (CCF-1408673,CCF-1412958,CCF-1350196), and the Simons Collaboration for Algorithms and Geometry.}}

\begin{document}

\maketitle
\draftbox
\thispagestyle{empty}

\begin{abstract}
We obtain the first polynomial-time algorithm for exact tensor completion that improves over the bound implied by reduction to matrix completion.
The algorithm recovers an unknown 3-tensor with $r$ incoherent, orthogonal components in $\mathbb R^n$ from $r\cdot \tilde O(n^{1.5})$ randomly observed entries of the tensor.  
This bound improves over the previous best one of $r\cdot \tilde O(n^{2})$ by reduction to exact matrix completion.
Our bound also matches the best known results for the easier problem of approximate tensor completion (Barak \& Moitra, 2015).

Our algorithm and analysis extends seminal results for exact matrix completion (Candes \& Recht, 2009) to the tensor setting via the sum-of-squares method.
The main technical challenge is to show that a small number of randomly chosen monomials are enough to construct a degree-3 polynomial with precisely planted orthogonal global optima over the sphere and that this fact can be certified within the sum-of-squares proof system.
\end{abstract}


\clearpage

\ifnum\showtableofcontents=1
{
\tableofcontents
\thispagestyle{empty}
 }
\fi

\clearpage

\setcounter{page}{1}

\section{Introduction}
\label{sec:introduction}

A basic task in machine learning and signal processing is to infer missing data from a small number of observations about the data.
An important example is \emph{matrix completiton} which asks to recover an unknown low-rank matrix from a small number of observed entries.
This problem has many interesting applications---one of the prominent original motivations was the Netflix Prize that sought improved algorithms for predicting user ratings for movies from a small number of user-provided ratings.
After an extensive research effort \cite{DBLP:journals/focm/CandesR09,DBLP:journals/tit/CandesT10,DBLP:conf/nips/KeshavanMO09,DBLP:conf/colt/SrebroS05}, efficient algorithms with almost optimal, provable recovery guarantees have been obtained:
In order to efficiently recover an unknown incoherent $n$-by-$n$ matrix of rank $r$ it is enough to observe $r\cdot \tilde O(n)$ random entries of the matrix \cite{DBLP:journals/tit/Gross11,DBLP:journals/jmlr/Recht11}.
One of the remaining challenges is to obtain algorithm for the more general and much less understood \emph{tensor completion problem} where the observations do not just consist of pairwise correlations but also higher-order ones.

Algorithms and analyses for matrix and tensor completion come in three flavors:
\begin{compactenum}%
\item algorithms analyzed by statistical learning tools like Rademacher complexity
\cite{DBLP:conf/colt/SrebroS05, DBLP:conf/colt/BarakM16}.
\item iterative algorithms like alternating minimization \cite{DBLP:conf/stoc/JainNS13,DBLP:conf/focs/Hardt14,DBLP:conf/colt/HardtW14}.
\item algorithms analyzed by constructing dual certificates for convex programming relaxations \cite{DBLP:journals/focm/CandesR09,DBLP:journals/tit/Gross11,DBLP:journals/jmlr/Recht11}.
\end{compactenum}
\Dnote{}
While each of these flavors have different benefits, typically only algorithms of the third flavor achieve exact recovery.
(The only exceptions to this rule we are aware of are a recent fast algorithm for matrix completion \cite{DBLP:conf/colt/0002N15} and a recent analysis \cite{DBLP:journals/corr/GeLM16} showing that the commonly used non-convex objective function for positive semidefinite matrix completion has no spurious local minima and thus stochastic gradient descent and other popular optimization programs can solve positive semidefinite matrix completion with arbitrary initialization.)
For all other algorithms, the analysis exhibits a trade-off between reconstruction error and the required number of observations (even when there is no noise in the input).\footnote{We remark that this trade-off is a property of the analysis and not necessarily the algorithm.
  For example, some algorithms of the first flavor are based on the same convex programming relaxations as exact recovery algorithms.
  Also for iterative algorithm, the trade-off between reconstruction error and number of sample comes from the requirement of the analysis that each iteration uses fresh samples.
  For these iterative algorithms, the number of samples depends only logarithmically on the desired accuracy, which means that these analyses imply exact recovery if the bit complexity of the entries is small.
}

In this work, we obtain the first algorithm for exact tensor completion that improves over the bounds implied by reduction to exact matrix completion.
The algorithm recovers an unknown 3-tensor with $r$ incoherent, orthogonal components in $\R^n$ from $r\cdot \tilde O(n^{1.5})$ randomly observed entries of the tensor.  
The previous best bound for exact recovery is $r\cdot \tilde O(n^{2})$, which is implied by reduction to exact matrix completion.
(The reduction views 3-tensor on $\R^n$ as an $n$-by-$n^2$ matrix.
We can recover rank-$r$ matrices of this shape from $r\cdot \tilde O(n^2)$ samples, which is best possible.)
Our bound also matches the best known results for the easier problem of approximate tensor completion \cite{DBLP:conf/nips/0002O14,DBLP:journals/corr/BhojanapalliS15,DBLP:conf/colt/BarakM16} (the results of the last work also applies to a wider range of tensors and does not require orthogonality).

A problem similar to matrix and tensor completion is matrix and tensor sensing.
The goal is to recover an unknown low rank matrix or tensor from a small number of linear measurements.
An interesting phenomenon is that for carefully designed measurements (which actually happen to be rank 1) it is possible to efficiently recover a $3$-tensor of rank $r$ with just $O(r^2\cdot n)$ measurements \cite{DBLP:conf/stoc/ForbesS12}, which is better than the best bounds for tensor completion when $r\ll n^{0.5}$.
We conjecture that for tensor completion from random entries the bound we obtain is up to logarithmic factors best possible among polynomial-time algorithms.

\paragraph{Sum-of-squares method}

Our algorithm is based on \emph{sum-of-squares} \cite{MR931698-Shor87,Parrilo00,MR1814045-Lasserre00}, a very general and powerful meta-algorithm studied extensively in many scientific communities (see for example the survey \cite{DBLP:journals/eccc/BarakS14}).
In theoretical computer science, the main research focus has been on the capabilities of sum-of-squares for approximation problems  \cite{DBLP:conf/stoc/BarakBHKSZ12}, especially in the context of Khot's Unique Games Conjecture \cite{DBLP:conf/stoc/Khot02a}.
More recently, sum-of-squares emerged as a general approach to inference problems that arise in machine learning and have defied other algorithmic techniques.  
This approach has lead to improved algorithms for tensor decomposition \cite{DBLP:conf/stoc/BarakKS15, DBLP:conf/approx/GeM15, DBLP:conf/stoc/HopkinsSSS16, DBLP:journals/corr/MaSS16}, dictionary learning \cite{DBLP:conf/stoc/BarakKS15,DBLP:journals/corr/HazanM16}, tensor principal component analysis \cite{DBLP:conf/colt/HopkinsSS15, DBLP:journals/corr/RaghavendraRS16, DBLP:journals/corr/BhattiproluGL16}, planted sparse vectors \cite{DBLP:conf/stoc/BarakKS14, DBLP:conf/stoc/HopkinsSSS16}.
An exciting direction is also to understand limitations of sum-of-squares for inference problems on concrete input distributions \cite{DBLP:conf/nips/MaW15, DBLP:conf/colt/HopkinsSS15,DBLP:journals/corr/BarakHKKMP16}.

An appealing feature of the sum-of-squares method is that its capabilities and limitations can be understood through the lens of a simple but surprisingly powerful and intuitive restricted proof system called sum-of-squares or Positivstellensatz system \cite{DBLP:journals/apal/GrigorievV01,DBLP:journals/cc/Grigoriev01,DBLP:journals/tcs/Grigoriev01}.
A conceptual contribution of this work is to show that seminal results for inference problem like compressed sensing and matrix completion have natural interpretations as \emph{identifiability proofs} in this system.
Furthermore, we show that this interpretation is helpful in order to analyze more challenging inference problems like tensor completion.
A promising future direction is to find more examples of inference problems where this lens on inference algorithms and identifiability proofs yields stronger provable guarantees.

A technical contribution of our work is that we develop techniques in order to show that sum-of-squares achieves exact recovery.
Most previous works only showed that sum-of-squares gives approximate solutions, which in some cases can be turned to exact solutions by invoking algorithms with local convergence guarantees~\cite{DBLP:conf/approx/GeM15,DBLP:conf/stoc/BarakKS14} or solving successive sum-of-squares relaxations~\cite{DBLP:journals/corr/MaSS16}.

\subsection{Results}

We say that a vector $v\in \R^n$ is $\mu$-incoherent with respect to the coordinate basis $e_1,\ldots,e_n$ if for every index $i\in[n]$,
\begin{equation}
  \iprod{e_i,v}^2 \le \tfrac \mu n \cdot \norm{v}^2\,.
\end{equation}

We say that a 3-tensor $X\in \R^{n}\otimes \R^{n}\otimes \R^{n}$ is orthogonal of rank $r$ if there are orthogonal vectors $\{u_i\}_{i\in [r]}\subseteq \R^{n}$, $\{v_i\}_{i\in [r]}\subseteq \R^{n}$, $\{w_i\}_{i\in [r]}\subseteq \R^{n}$ such that $X=\sum_{i=1}^r u_i\otimes v_i\otimes w_i$.
We say that such a 3-tensor $X$ is $\mu$-incoherent if all of the vectors $u_i,v_i,w_i$ are $\mu$-incoherent.

\begin{theorem}[main]
  \label{thm:main}
  There exists a polynomial-time algorithm that given at least $r\cdot \mu^{O(1)} \cdot \tilde O(n)^{1.5}$ random entries of an unknown orthogonal $\mu$-incoherent $3$-tensor $X\in \R^n\otimes \R^n\otimes \R^n$ of rank $r$, outputs all entries of $X$ with probability at least $1-n^{-\omega(1)}$.
\end{theorem}

We note that the analysis also shows that the algorithm is robust to inverse polynomial amount of noise in the input (resulting in inverse polynomial amount of error in the output).

We remark that the running time of the algorithm depends polynomially on the bit complexity on $X$.

\Dnote{}

\section{Techniques}
\label{sec:techniques}

Let $\{u_i\}_{i\in [r]},\{v_i\}_{i\in [r]},\{w_i\}_{i\in [r]}$ be three orthonormal sets in $\R^n$.
Consider a 3-tensor $X\in \R^n\otimes \R^n\otimes \R^n$ of the form $X=\sum_{i=1}^r \lambda _i \cdot u_i\otimes v_i \otimes w_i$ with $\lambda_1,\ldots,\lambda_n\ge 0$.
Let $\Omega\subseteq [n]^3$ be a subset of the entries of $X$.

Our goal is to efficiently reconstruct the unknown tensor $X$ from its restriction $X_\Omega$ to the entries in $\Omega$.
Ignoring computational efficiency, we first ask if this task is information-theoretically possible.
More concretely, for a given set of observations $X_\Omega$, how can we rule out that there exists another rank-$r$ orthogonal $3$-tensor $X'\neq X$ that would give rise to the same observations $X'_\Omega=X_\Omega$?\footnote{We emphasize that we ask here about the uniqueness of $X$ for a fixed set of entries $\Omega$.
  This questions differs from asking about the uniqueness for a random set of entries, which could be answered by suitably counting the number of low-rank $3$-tensors.}

A priori it is not clear how an answer to this information-theoretic question could be related to the goal of obtaining an efficient algorithm.
However, it turns out that the sum-of-squares framework allows us to systematically translate a uniqueness proof to an algorithm that efficiently finds the solution.
(In addition, this solution also comes with a short certificate for uniqueness.\footnote{This certificate is closely related to certificates in the form of dual solutions for convex programming relaxations that are used in the compressed sensing and matrix completion literature.})

\paragraph{Uniqueness proof}
\label{sec:uniqueness}
Let $\Omega\subseteq [n]^3$ be a set of entries and let $X=\sum_{i=1}^r \lambda_i \cdot u_i\otimes v_i\otimes w_i$ be a 3-tensor with $\lambda_1,\ldots,\lambda_r\ge 0$.

It turns out that the following two conditions are enough to imply that $X_\Omega$ uniquely determines $X$:
The first condition is that the vectors $\{(u_i\otimes v_i \otimes w_i)_\Omega\}$ are linearly independent.
The second condition is that exists a 3-linear form $T$ on $\R^n$ with the following properties:
\begin{enumerate}
\item in the monomial basis $T$ is supported on $\Omega$ so that  $T(x,y,z)=\sum_{(i,j,k)\in \Omega} T_{ijk} \cdot x_iy_jx_k$,
\item evaluated over unit vectors, the 3-form $T$ is exactly maximized at the points $(u_i,v_i,w_i)$ so that $T(u_1,v_1,w_1)=\dots=T(u_r,v_r,w_r)=1$ and ${T(x,y,z)}<1$ for all unit vectors $(x,y,z) \not\in \{(u_i,v_i,w_i)\mid i\in[r]\}$.
\end{enumerate}

We show that the two deterministic conditions above are satisfied with high probability if the vectors $\{u_i\},\{v_i\},\{w_i\}$ are incoherent and $\Omega$ is a random set of entries of size at least $r\cdot \tilde O(n^{1.5})$.

Let us sketch the proof that such a 3-linear form $T$ indeed implies uniqueness.
Concretely, we claim that if we let $X'$ be a 3-tensor of the form $\sum_{i=1}^{r'} {\lambda'_i} \cdot u'_i\otimes v'_i\otimes w'_i$ for $\lambda_1',\ldots,\lambda'_{r'}\ge 0$ and unit vectors $\{u'_i\}$, $\{v'_i\}$, $\{w'_i\}$ with $X'_\Omega=X_\Omega$ that minimizes $\sum_{i=1}^{r'} \abs{\lambda'_i}$ then $X'=X$ must hold.
We identify $T$ with an element of $\R^n\otimes \R^n\otimes \R^n$ (the coefficient tensor of $T$ in the monomial basis).
Let $X'$ be as before.
We are to show that $X=X'$.
On the one hand, using that $T(x,y,z)\le 1$ for all unit vectors $x,y,z$,
\begin{displaymath}
  \iprod{T,X'} = \sum_{i=1}^{r'} \lambda_i' \cdot T(u'_i,v'_i,w'_i)
  \le \sum_{i=1}^{r'}\lambda'_i\,.
\end{displaymath}
At the same time, using that $T$ is supported on $\Omega$ and the fact that $X_{\Omega}=X'_\Omega$,
\begin{displaymath}
  \iprod{T,X'}
  = \iprod{T,X}
  = \sum_{i=1}^r \lambda_i \cdot T(u_i,v_i,w_i) =\sum_{i=1}^r \lambda_i\,.
\end{displaymath}
Since $X'$ minimizes $\sum_{i=1}^{r'}\lambda'_i$, equality has to hold in the previous inequality.
It follows that every point $(u'_i,v'_i,w'_i)$ is equal to one of the points $(u_j,v_j,w_j)$, because $T$ is uniquely maximized at the points $\{(u_i,v_i,w_i)\mid i\in[r]\}$.
Since we assumed that $\{(u_i\otimes v_i \otimes w_i)_\Omega\}$ is linearly independent, we can conclude that $X=X'$.

When we show that such a 3-linear form $T$ exists, we will actually show something stronger, namely that the second property is not only true but also has a short certificate in form of a ``degree-4 sum-of-squares proof'', which we describe next.
This certificate also enables us to efficiently recover the missing tensor entries.

\paragraph{Uniqueness proof in the sum-of-squares system}

A degree-4 sos certificate for the second property of $T$ is an $(n+n^2)$-by-$(n+n^2)$ positive-semidefinite matrix $M$ (acting as a linear operator on $\R^n\oplus (\R^n \otimes \R^n)$) that represents the polynomial $\norm{x}^{2} + \norm{y}^2 \cdot \norm{z}^2 - 2 T(x,y,z)$, i.e.,
\begin{equation}
  \iprod{(x,y\otimes z), M (x,y\otimes z)} = \norm{x}^{2} + \norm{y}^2 \cdot \norm{z}^2 - 2 T(x,y,z)\,.
\end{equation}
Furthermore, we require that the kernel of $M$ is precisely the span of the vectors $\{(u_i,v_i\otimes w_i)\mid i\in [r]\}$.
Let's see that this matrix $M$ certifies that $T$ has the property that over unit vectors it is exactly maximized at the desired points $(u_i,v_i,w_i)$.
Let $u,v,w$ be unit vectors such that $(u,v,w)$ is not a multiple of one of the vectors $(u_i,v_i,w_i)$.
Then by orthogonality, both $(u,v\otimes w)$ and  $(-u,v\otimes w)$ have non-zero projection on the orthogonal complement of the kernel of $M$.
Therefore, the bounds $0<\iprod{(u,v\otimes w),M(u,v\otimes w)}=2-2p(u,v,w)$ and $0<\iprod{(-u,v\otimes w),M(-u,v\otimes w)}=2+2p(u,v,w)$ together give the desired conclusion that $\abs{T(u,v,w)}<1$.

\paragraph{Reconstruction algorithm based on the sum-of-squares system}

The existence of a positive semidefinite matrix $M$ as above not only means that reconstruction of $X$ from $X_\Omega$ is possible information-theoretically but also efficiently.
The sum-of-squares algorithm allows us to efficiently search over low-degree moments of objects called \emph{pseudo-distributions} that generalize probability distributions over real vector spaces.
Every pseudo-distribution $\mu$ defines \emph{pseudo-expectation values} $\pE_\mu f$ for all low-degree polynomial functions $f(x,y,z)$, which behave in many ways like expectation values under an actual probability distribution.
In order to reconstruct $X$ from the observations $X_\Omega$, we use the sum-of-squares algorithm to efficiently find a pseudo-distribution $\mu$ that satisfies\footnote{The viewpoint in terms of pseudo-distributions is useful to see how the previous uniqueness proof relates to the algorithm.
  We can also describe the solutions to the constraints \cref{eq:search-1,eq:search-2} in terms of linearly constrained positive semidefinite matrices. See alternative description of \cref{alg:tensor}
}
\begin{gather}
  \pE_{\mu(x,y,z)} \norm{x}^2 + \norm{y}^2\cdot \norm{z}^2 \le 1
  \label{eq:search-1}
 \\
 \Paren{\pE_{\mu(x,y,z)} x \otimes y \otimes z }_\Omega= X_\Omega
 \label{eq:search-2}
\end{gather}
Note that the distribution over the vectors $(u_i,v_i,w_i)$ with probabilities $\lambda_i$ satisfies the above conditions.
Our previous discussion about uniqueness shows that the existence of a positive semidefinite matrix $M$ as above implies no other distribution satisfies the above conditions.
It turns out that the matrix $M$ implies that this uniqueness holds even among pseudo-distributions in the sense that any pseudo-distribution that satisfies \cref{eq:search-1,eq:search-2} must satisfy $\pE_{\mu(x,y,z)} x\otimes y\otimes z=X$, which means that the reconstruction is successful.\footnote{The matrix $M$ can also be viewed as a solution to the dual of the convex optimization problem of finding a pseudo-distribution that satisfies conditions \cref{eq:search-1,eq:search-2}.}

\origparagraph{When do such uniqueness certificates exist?}
The above discussion shows that in order to achieve reconstruction it is enough to show that uniqueness certificates of the form above exist.
We show that these certificates exists with high probability if we choose $\Omega$ to be a large enough random subset of entries (under suitable assumptions on $X$).
Our existence proof is based on a randomized procedure to construct such a certificate heavily inspired by similar constructions for matrix completion \cite{DBLP:journals/tit/Gross11, DBLP:journals/jmlr/Recht11}.
(We note that this construction uses the unknown tensor $X$ and is therefore not ``constructive'' in the context of the recovery problem.)

Before describing the construction, we make the requirements on the 3-linear form $T$ more concrete.
We identify $T$ with the linear operator from $\R^n\otimes \R^n$ to $\R^n$ such that $T(x,y,z)=\iprod{x, T (y\otimes z)}$.
Furthermore, let $T_a$ be linear operators on $\R^n$ such that $T(x,y,z)=\sum_{a=1}^n x_a\cdot \iprod{y,T_a z}$.
Then, the following conditions on $T$ imply the existence of a uniqueness certificate $M$ (which also means that recover succeeds),
\begin{enumerate}
\item every unknown entry $(i,j,k)\not\in \Omega$ satisfies $\iprod{e_i,T (e_j\otimes e_k)}=0$,
 
\item every index $i\in [r]$ satisfies $u_i = T (v_i\otimes w_i)$,

\item the matrix $\sum_{a=1}^n T_a \otimes \transpose{T_a} - \sum_{i=1}^r \dyad{(v_i\otimes w_i)}$ has spectral norm at most $0.01$.
\end{enumerate}
We note that the uniqueness certificates for matrix completion \cite{DBLP:journals/tit/Gross11, DBLP:journals/jmlr/Recht11} have similar requirements.
The key difference is that we need to control the spectral norm of an operator that depends quadratically on the constructed object $T$ (as opposed to a linear dependence in the matrix completion case).
Combined with the fact that the construction of $T$ is iterative (about $\log n$ steps), the spectral norm bound unfortunately requires significant technical work.
In particular, we cannot apply general matrix concentration inequalities and instead apply the trace moment method.
(See \cref{normmethodssection}.)

We also note that the fact that the above requirements allow us to construct the certifcate $M$ is not immediate and requires some new ideas about matrix representations of polynomials, which might be useful elsewhere.
(See \cref{sec:zero-matching}.)

Finally, we note that the transformation applied to $T$ in order to obtain the matrix for the third condition above appears in many works about 3-tensors \cite{DBLP:conf/colt/HopkinsSS15, DBLP:conf/colt/BarakM16} with the earliest appearance in a work on refutation algorithms for random 3-SAT instances (see \cite{DBLP:journals/toc/FeigeO07}).

The iterative construction of the linear operator $T$ exactly follows the recipe from matrix completion \cite{DBLP:journals/tit/Gross11, DBLP:journals/jmlr/Recht11}.
Let $\cR_\Omega$ be the projection operator into the linear space of operators $T$ that satify the first requirement.
Let $\cP_T$ be the (affine) projection operator into the affine linear space of operators $T$ that satisfy the second reqirement.
We start with $\super T 0=X$.
At this point we satisfy the second condition.
(Also the matrix in the third condition is $0$.)
In order to enforce the first condition we apply the operator $\cR_\Omega$.
After this projection, the second condition is most likely no longer satisfied.
To enforce the second condition, we apply the affine linear operator $\cP_T$ and obtain $\super T 1 = \cP_T(\cR_\Omega X)$.
The idea is to iterate this construction and show that after a logarithmic number of iterations both the first and second condition are satisfied up to an inverse polynomially small error (which we can correct in a direct way).
The main challenge is to show that the iterates obtained in this way satisfy the desired spectral norm bound.
(We note that for technical reasons the construction uses fresh randomness $\Omega$ for each iteration like in the matrix completion case \cite{DBLP:journals/jmlr/Recht11,DBLP:journals/tit/Gross11}.
Since the number of iterations is logarithmic, the total number of required observations remains the same up to a logarithmic factor.)

\section{Preliminaries}
\label{sec:preliminaries}

Unless explicitly stated otherwise, $O(\cdot)$-notation hides absolute multiplicative constants.
Concretely, every occurrence of $O(x)$ is a placeholder for some function $f(x)$ that satisfies $\forall x\in \R.\, \abs{f(x)}\le C\abs{x}$ for some absolute constant $C>0$.
Similarly, $\Omega(x)$ is a placeholder for a function $g(x)$ that satisfies $\forall x\in \R.\, \abs{g(x)} \ge \abs{x}/C$ for some absolute constant $C>0$.

\newcommand{\support}{\mathrm{support}}

Our algorithm is based on a generalization of probability distributions over $\R^n$.
To define this generalization the following notation for the formal expectation of a function $f$ on $\R^n$ with respect to a finitely-supported function $\mu\from \R^n\to \R$,
$$
\pE_{\mu} f = \sum_{x\in \support(\mu)} \mu(x)\cdot f(x)\,.
$$
A \emph{degree-$d$ pseudo-distribution over $\R^n$} is a finitely-supported function $\mu\from \R^n\to \R$ such that $\pE_\mu 1 = 1$ and $\pE_\mu f^2\ge 0$ for every polynomial $f$ of degree at most $d/2$.

A key algorithmic property of pseudo-distributions is that their low-degree moments have an efficient separation oracle.
Concretely, the set of degree-$d$ moments $\pE_\mu (1,x)^{\otimes d}$ such that $\mu$ is a degree-$d$ pseudo-distributions over $\R^n$ has an $n^{O(d)}$-time separation oracle.
Therefore, standard convex optimization methods allow us to efficiently optimize linear functions over low-degree moments of pseudo-distributions (even subject to additional convex constraints that have efficient separation oracles) up to arbitrary numerical accuracy.

\section{Tensor completion algorithm}
\label{sec:algorithm}

In this section, we show that the following algorithm for tensor completion succeeds in recovering the unknown tensor from partial observations assuming the existence of a particular linear operator $T$.
We will state conditions on the unknown tensor that imply that such a linear operator exists with high probability if the observed entries are chosen at random.
We use essentially the same convex relaxation as in \cite{DBLP:conf/colt/BarakM16} but our analysis differs significantly.

\begin{framedalgorithm}[Exact tensor completion based on degree-4 sum-of-squares]\mbox{}\\
\label[algorithm]{alg:tensor}\textbf{Input:} locations $\Omega\subseteq[n]^3$ and partial observations $X_\Omega$ of an unknown $3$-tensor $X\in\R^n\otimes\R^n\otimes \R^n$.
  \\
  \textbf{Operation:}
  Find a degree-$4$ pseudo-distribution $\mu$ on $\R^n\oplus \R^n\oplus \R^n$ such that the third moment matches the observations $\Paren{\pE_{\mu(x,y,z)} x\otimes y \otimes z}_\Omega=X_\Omega$ so as to minimize
  \begin{displaymath}
    \pE_{\mu(x,y,z)}\norm{x}^2+\norm{y}^2\cdot \norm{z}^2\,.
  \end{displaymath}
  Output the $3$-tensor $\pE_{\mu(x,y,z)} x\otimes y \otimes z\in \R^n\otimes \R^n\otimes \R^n$.
  \\
  \textbf{Alternative description:}
  Output a minimum trace, positive semidefinite matrix $Y$ acting on $\R^n\oplus (\R^n \otimes \R^n)$ with blocks $Y_{1,1}$, $Y_{1,2}$ and $Y_{2,2}$ such that $(Y_{1,2})_\Omega=X_\Omega$ matches the observations, and $Y_{2,2}$ satisfies the additional symmetry constraints that each entry \begin{math}
    \iprod{e_j\otimes e_k,Y_{2,2} (e_{j'}\otimes e_{k'})}
  \end{math}
  only depends on the index sets $\{j,j'\},\{k,k'\}$.
\end{framedalgorithm}

Let $\{u_i\},\{v_i\},\{w_i\}$ be three orthonormal sets in $\R^n$, each of cardinality $r$.

We reason about the recovery guarantees of the algorithm in terms of the following notion of certifcate.

\begin{definition}
  \label[definition]{def:certificate}
  We say that a linear operator $T$ from $\R^{n}\otimes \R^n$ to $\R^n$ is a \emph{degree-4 certificate} for $\Omega$ and orthonormal sets $\{u_i\},\{v_i\},\{w_i\}\subseteq \R^n$ if the following conditions are satisfies
  \begin{enumerate}
  \item the vectors $\{(u_i\otimes v_j\otimes w_k)_\Omega \mid (i,j,k)\in S\}$ are linearly independent, where $S\subseteq[n]^3$ is the set of triples with at least two identical indices from $[r]$,
  \item every entry  $(a,b,c)\not\in \Omega$ satisfies $\iprod{e_a,T (e_b\otimes e_c)}=0$,
  \item If we view $T$ as a 3-tensor in $(\R^n)^{\otimes 3}$ whose $(a,b,c)$ entry is $\iprod{e_a,T (e_b\otimes e_c)}$, every index $i\in [r]$ satisfies $(\transpose u_i\otimes \transpose v_i \otimes \Id ) T = w_i$,  $(\transpose u_i\otimes \Id \otimes \transpose w_i ) T = v_i$, and $(\Id\otimes \transpose v_i \otimes\transpose  w_i ) T = u_i$. 
  \item  the following matrix has spectral norm at most $0.01$,
  \begin{displaymath}
    \sum_{a=1}^n T_a \otimes \transpose{T_a} - \sum_{i=1}^r \dyad{(v_i\otimes w_i)}\,,
  \end{displaymath}
  where $\{T_a\}$ are matrices such that $\iprod{x,T (y\otimes x)} =\sum_{a=1}^n x_a \cdot \iprod{y,T_a z}$.
  \end{enumerate}
\end{definition}

In \cref{sec:recovery}, we prove that existence of such certifcates implies that the above algorithm successfully recovers the unknown tensor, as formalized by the following theorem.

\begin{theorem}
  \label{thm:recovery}
  Let $X\in\R^n\otimes\R^n\otimes \R^n$ be any 3-tensor of the form $\sum_{i=1}^r \lambda_i\cdot u_i\otimes v_i\otimes w_i$ for $\lambda_1,\ldots,\lambda_r\in \R_+$.
  Let $\Omega\subseteq [n]^3$ be a subset of indices.
  Suppose there exists degree-4 certificate in the sense of \cref{def:certificate}.
  Then, given the observations $X_\Omega$ the above algorithm recovers the unknown tensor $X$ exactly.
\end{theorem}

In \cref{sec:certificate}, we show that degree-4 certificates are likely to exist when $\Omega$ is a random set of appropriate size.

\begin{theorem}
  \label{thm:certificate}
  Let $\{u_i\},\{v_i\},\{w_i\}$ be three orthonormal sets of $\mu$-incoherent vectors in $\R^n$, each of cardinality $r$.
  Let $\Omega\subseteq [n]^3$ be a random set of tensor entries of cardinality $m=r\cdot n^{1.5} (\mu \log n)^C$ for an absolute constant $C\ge 1$.
  Then, with probability $1-n^{-\omega(1)}$, there exists a linear operator $T$ that satisfies the requirements of \cref{def:certificate}.
\end{theorem}

Taken together the two theorems above imply our main result \cref{thm:main}.

\subsection{Simpler proofs via higher-degree sum-of-squares}
\label{sec:simpler}

Unfortunately the proof of \cref{thm:certificate} requires extremely technical spectral norm bounds for random matrices.

It turns out that less technical norm bounds suffice if we use degree 6 sum-of-squares relaxations.
For this more powerful algorithm, weaker certificates are enough to ensure exact recovery and the proof that these weaker certificates exist with high probability is considerably easier than the proof that degree-4 certificates exist with high probability.

In the following we describe this weaker notion of certificates and state their properties.
In the subsequent sections we prove properties of these certificates are enough to imply our main result \cref{thm:main}.

\begin{framedalgorithm}[Exact tensor completion based on higher-degree sum-of-squares]\mbox{}\\
\label[algorithm]{alg:tensor-higher}\textbf{Input:} locations $\Omega\subseteq[n]^3$ and partial observations $X_\Omega$ of an unknown $3$-tensor $X\in\R^n\otimes\R^n\otimes \R^n$.
  \\
  \textbf{Operation:}
  Find a degree-$6$ pseudo-distribution $\mu$ on $\R^n\oplus \R^n\oplus \R^n$ so as to minimize
  \begin{math}
    \pE_{\mu(x,y,z)}\norm{x}^2+\norm{z}^2
  \end{math}
  subject to the following constraints
  \begin{align}
    \Paren{\pE_{\mu(x,y,z)} x\otimes y \otimes z}_{\Omega}&=X_{\Omega}\,,\\
    \pE_{\mu(x,y,z)} (\norm{y}^2-1)\cdot p(x,y,z) &= 0 \text{ for all $p(x,y,z)\in \R[x,y,z]_{\le 4}$}\,.
  \end{align}
  Output the $3$-tensor $\pE_{\mu(x,y,z)} x\otimes y \otimes z\in \R^n\otimes \R^n\otimes \R^n$.
\end{framedalgorithm}

Let $\{u_i\},\{v_i\},\{w_i\}$ be three orthonormal sets in $\R^n$, each of cardinality $r$.
We reason about the recovery guarantees of the above algorithm in terms of the following notion of certificate.
The main difference to degree-4 certificate (\cref{def:certificate}) is that the spectral norm condition is replaced by a condition in terms of sum-of-squares representations. 

\begin{definition}
  \label[definition]{def:certificate-higher}
  We say that a 3-tensor $T\in (\R^n)^{\otimes 3}$ is a \emph{higher-degree certificate} for $\Omega$ and orthonormal sets $\{u_i\},\{v_i\},\{w_i\}\subseteq \R^n$ if the following conditions are satisfies
  \begin{enumerate}
  \item the vectors $\{(u_i\otimes v_i\otimes w_i)_\Omega\}_{i \in [r]}$ are linearly independent,
  \item every entry  $(a,b,c)\not\in \Omega$ satisfies $\iprod{T,(e_a \otimes e_b \otimes e_c)}=0$,
  \item every index $i\in [r]$ satisfies $(\transpose u_i\otimes \transpose v_i \otimes \Id ) T = w_i$,  $(\transpose u_i\otimes \Id \otimes \transpose w_i ) T = v_i$, and $(\Id\otimes \transpose v_i \otimes\transpose  w_i ) T = u_i$,
  \item  the following degree-4 polynomials in $\R[x,y,z]$ are sum of squares
  \begin{gather}
    \norm{x}^2 + \norm{y}^2\cdot \norm{z}^2 - 1/\e \cdot \iprod{T',x\otimes y \otimes z}\,,  \\
    \norm{y}^2 + \norm{x}^2\cdot \norm{z}^2-   1/\e \cdot \iprod{T',x\otimes y \otimes z}\,,\\
    \norm{z}^2 + \norm{x}^2\cdot \norm{y}^2-   1/\e \cdot \iprod{T',x\otimes y \otimes z}\,.
  \end{gather}
  where $T'=T-\sum_{i=1}^r u_i\otimes v_i\otimes w_i$ and $\e>0$ is an absolute constant (say $\e=10^{-6}$).
  \end{enumerate}
\end{definition}

In the following sections we prove that higher-degree certificates imply that \cref{alg:tensor-higher} successfully recovers the desired tensor and that they exist with high probability for random $\Omega$ of appropriate size.

\subsection{Higher-degree certificates imply exact recovery}
\label{sec:recovery-higher}

Let $\{u_i\}$, $\{v_i\}$, $\{w_i\}$ be orthonormal bases in $\R^n$.
We say that a degree-$\ell$ pseudo-distribution $\mu(x,y,z)$ satisfies the constraint $\norm{y}^2=1$, denoted $\mu\models\ \{\norm{y}^2=1\}$, if $\pE_{\mu(x,y,z)} p(x,y,z)\cdot (1-\norm{y}^2)=0$ for all polynomials $p\in \R[x,y,z]_{\le \ell-2}$

We are to show that a higher-degree certificate in the sense of \cref{def:certificate-higher} implies that \cref{alg:tensor-higher} reconstructs the partially observed tensor exactly.
A key step of this proof is the following lemma about expectation values of higher degree pseudo-distributions.

\begin{lemma}
  \label[lemma]{lem:pseudo-expectation}
  Let $T\in(\R^n)^{\otimes 3}$ be a higher-degree certificate as in \cref{def:certificate-higher} for the set $\Omega\subseteq[n]^3$ and the vectors $\{u_i\}_{i\in [r]}, \{v_i\}_{i\in [r]}, \{w_i\}_{i\in [r]}$.
  Then, every degree-6 pseudo-distribution $\mu(x,y,z)$ with $\mu \models \{ \norm{y}^2=1\}$ satisfies
  \begin{multline}
    \pE_{\mu(x,y,z)} T(x,y,z) \le \pE_{\mu(x,y,z)} \frac{\norm{x}^2+\norm{z}^2}2
    -\tfrac 1 {100} \cdot \sum_{i=r+1}^n \paren{\iprod{u_i,x}^2 + \iprod{w_i,z}^2}\\
    - \tfrac 1 {100} \cdot \sum_{i=1}^n\sum_{j\in [n]\setminus \{i\}} \iprod{v_i,y}^2 \cdot \Paren{\vbig\iprod{u_j,x}^2 + \iprod{w_j,z}^2 }
  \end{multline} 
\end{lemma}

To prove this lemma it will be useful to introduce the sum-of-squares proof system.
Before doing that let us observe that the lemma indeed allows us to prove that \cref{alg:tensor-higher} works.

\begin{theorem}[Higher-degree certificates imply exact recovery]
  Suppose there exists a higher-degree certificate $T$ in the sense of \cref{def:certificate-higher} for the set $\Omega\subseteq[n]^3$ and the vectors $\{u_i\}_{i\in [r]}, \{v_i\}_{i\in [r]}, \{w_i\}_{i\in [r]}$.
  Then, \cref{alg:tensor-higher} recovers the partially observed tensor exactly.
  In other words, if $X=\sum_{i=1}^r \lambda_i \cdot u_i\otimes v_i\otimes w_i$ with $\lambda_1,\ldots,\lambda_r\ge 0$ and $\mu(x,y,z)$ is a degree-6 pseudo-distribution with $\mu \models \{ \norm{y}^2=1\}$ that minimizes $\pE_{\mu(x,y,z)}\tfrac{1}{2}\paren{\norm{x}^2+\norm{z}^2}$ subject to $\paren{\pE_{\mu(x,y,z)} x\otimes y \otimes z}_{\Omega}=X_\Omega$, then $\pE_{\mu(x,y,z)} x\otimes y \otimes z= X$.
\end{theorem}

\begin{proof}
  Consider the distribution $\mu^*$ over vectors $(x,y,z)$ such that $(\sqrt {\lambda_in}\cdot u_i,  v_i, \sqrt{\lambda_in}\cdot w_i)$ has probability $1/n$.
  By construction, $\E_{\mu^*(x,y,z)} x\otimes y \otimes z=X$.
  We have
  \begin{displaymath}
    \pE_{\mu(x,y,z)} T(x,y,z) = \E_{\mu^*(x,y,z)} T(x,y,z) = \sum_{i=1}^r \lambda_i = \E_{\mu^*(x,y,z)}\tfrac{1}2 \paren{\norm{x}^2 + \norm{z}^2}\,.
  \end{displaymath}
  By \cref{lem:pseudo-expectation} and the optimality of $\mu$, it follows that 
  \begin{displaymath}
    \pE_{\mu(x,y,z)}\tfrac 1 {100} \cdot \sum_{i=r+1}^n \paren{\iprod{u_i,x}^2 + \iprod{w_i,z}^2}\\
    + \tfrac 1 {100} \cdot \sum_{i=1}^n\sum_{j\in [n]\setminus \{i\}} \iprod{v_i,y}^2 \cdot \Paren{\vbig\iprod{u_j,x}^2 + \iprod{w_j,z}^2 } = 0
  \end{displaymath}
  Since the summands on the left-hand side are squares it follows that each summand has pseudo-expectation $0$.
  It follows that $\pE_{\mu}\iprod{u_i,x}^2=\pE_{\mu}\iprod{v_i,y}^2=\pE_{\mu}\iprod{w_i,z}^2=0$ for all $i > r$ and $\pE_{\mu} \iprod{v_i,y}^2\iprod{u_j,x}^2 = \pE_{\mu} \iprod{v_i,y}^2\iprod{w_j,z}^2=0$ for all $i\neq j$.
  By the Cauchy--Schwarz inequality for pseudo-expectations, it follows that $\pE_{\mu(x,y,z)}\iprod{x\otimes y\otimes z, u_i\otimes v_j\otimes w_k}=0$ unless $i=j=k\in [r]$.
  Consequently, $\pE_{\mu(x,y,z)}x\otimes y\otimes z$ is a linear combination of the vectors $\{u_i\otimes v_i\otimes w_i \mid i\in [r]\}$.
  Finally, the linear independence of the vectors $\{(u_i\otimes v_i\otimes w_i)_\Omega \mid i\in [r]\}$ implies that $\pE_{\mu}x\otimes y\otimes z=X$ as desired.
\end{proof}

It remains to prove \cref{lem:pseudo-expectation}.
Here it is convenient to use formal notation for sum-of-squares proofs.
We will work with polynomials $\R[x,y,z]$ and the polynomial equation $\cA=\{\norm{y}^2=1\}$.
For $p\in \R[x,y,z]$, we say that there exists a degree-$\ell$ SOS proof that $\cA$  implies $p\ge 0$, denoted $\cA \vdash_\ell p\ge 0$, if there exists a polynomial $q\in \R[x,y,z]$ of degree at most $\ell-2$ such that $p+ q\cdot (1-\norm{y}^2)$ is a sum of squares of polynomials.
This notion proof allows us to reason about pseudo-distributions.
In particular, if $\cA \vdash_\ell p\ge 0$ then every degree-$\ell$ pseudo-distribution $\mu$ with $\mu\models\cA$ satisfies $\pE_{\mu} p\ge 0$.

We will change coordinates such that $u_i=v_i=w_i=e_i$ is the $i$-th coordinate vector for every $i\in [n]$.
Then, the conditions on $T$ in \cref{def:certificate-higher} imply that
\begin{equation}
  \label{eq:parts-of-T}
  \iprod{T,(x\otimes y \otimes z)} = \sum_{i=1}^r x_i y_i z_i + T'(x,y,z)\,,
\end{equation}
where $T'$ is a 3-linear form with the property that $T'(x,x,x)$ does not contain squares (i.e. is multilinear).
Furthermore, the conditions imply the following SOS proofs for $T'$:
\begin{compactenum}
\item $\emptyset \vdash_4 T'(x,y,z)\le \e \cdot \Paren{\norm{x}+\norm{y}^2\cdot \norm{z}^2}$,
\item $\emptyset \vdash_4 T'(x,y,z)\le \e \cdot \Paren{\norm{y}+\norm{x}^2\cdot \norm{z}^2}$,
\item $\emptyset \vdash_4 T'(x,y,z)\le \e \cdot \Paren{\norm{z}+\norm{x}^2\cdot \norm{y}^2}$.
\end{compactenum}

The following lemma gives an upper bound on one of the parts in \cref{eq:parts-of-T}.

\begin{lemma}
  \label[lemma]{lem:bound-on-first-part}
  For $\cA = \{\norm{y}^2=1\}$, the following inequality has a degree-6 sum-of-squares proof,
  \begin{multline}
    \cA  \vdash_6 \sum_{i=1}^r x_i y_i z_i
    \le \tfrac 12 \norm{x}^2 + \tfrac 12 \norm{z}^2 - \tfrac 1 4 \sum_{i=r+1}^n (x_i^2 + z_i^2)\\
    - \tfrac 18 \sum_{i\neq j} y_i^2 \cdot \Paren{\vbig x_j^2 + z_j^2 + y_j^2 \cdot (\norm{x}^2 + \norm{z}^2)}\,.
  \end{multline}
\end{lemma}

\begin{proof}
  We bound the left-hand side in the lemma as follows,
  \begin{align}
    \cA \vdash_6 \sum_{i=1}^{r}{{x_i}{y_i}{z_i}}
    & \leq \sum_{i=1}^{r}\paren{\tfrac 12x^2_i + \tfrac 12y^2_iz^2_i}\\
    & \le \tfrac12{||x||^2} - \tfrac{1}{2}\sum_{i>r}{x^2_i}
      + \tfrac{1}{2}\sum_{i=1}^n{{y^2_i}{z^2_i}}\,.
  \end{align}
  We can further bound $\sum_i y_i^2 z_i^2$ as follows,
  \begin{align}
    \cA \vdash_6 \sum_{i=1}^n{{y^2_i}{z^2_i}}
    & = \Paren{\sum_{i=1}^n y_i^2} \cdot \Paren{\sum_{i=1}^n z_i^2} - \sum_{i\neq j} y_i^2 \cdot z_j^2\\
    & = \Paren{\sum_{i=1}^n z_i^2} - \sum_{i\neq j} y_i^2 \cdot z_j^2\,.
  \end{align}
  We can prove a different bound on $\sum_i y_i^2 z_i^2$ as follows,
  \begin{align}
    \cA \vdash_6 \sum_{i=1}^n{{y^2_i}{z^2_i}}
    & \le \tfrac 12 \norm{z}^2 + \tfrac12 \sum_{i=1}^n y_i^4 z_i^2\\
    & \le \tfrac 12 \norm{z}^2 + \tfrac12 \sum_{i=1}^n y_i^4 \norm{z}^2\\
    & =  \tfrac 12 \norm{z}^2
      + \tfrac12 \Paren{\sum_{i=1}^n y_i^2} \cdot \Paren{\sum_{i=1}^n y_i^2\norm{z}^2}
      - \tfrac12 \sum_{i\neq j} y_i^2 \cdot y_j^2 \norm{z}^2\\
    & = \norm{z}^2 - \tfrac 12 \sum_{i\neq j} y_i^2 \cdot y_j^2 \norm{z}^2\,.
  \end{align}
  By combining these three inequalities, we obtain the inequality
  \begin{displaymath}
    \cA \vdash_6 \sum_{i=1}^{r}{{x_i}{y_i}{z_i}}
    \le \tfrac 12 \norm{x}^2 + \tfrac12 \norm{z}^2
    -\tfrac12 \sum_{i>r} x_i^2
    - \tfrac 12 \sum_{i\neq j} y_i^2 \cdot z_j^2
    - \tfrac 14 \sum_{i\neq j} y_i^2 \cdot y_j^2 \norm{z}^2\,.
  \end{displaymath}
  By symmetry between $z$ and $x$, the same inequality holds with $x$ and $z$ exchanged.
  Combining these symmetric inequalities, we obtain the desired inequality
  \begin{multline}
    \cA \vdash_6 \sum_{i=1}^{r}{{x_i}{y_i}{z_i}}
    \le \tfrac 12 \norm{x}^2 + \tfrac12 \norm{z}^2
    -\tfrac14 \sum_{i>r} (x_i^2 + z_i^2)\\
    - \tfrac 14 \sum_{i\neq j} y_i^2 \cdot (x_j^2 + z_j^2)
    - \tfrac 18 \sum_{i\neq j} y_i^2 \cdot y_j^2 (\norm{x}^2 + \norm{z}^2)\,.
  \end{multline}
\end{proof}

It remains to bound the second part in \cref{eq:parts-of-T}, which the following lemma achieves.

\begin{lemma}
  \label{boundonother}
  \label[lemma]{lem:bound-on-second-part}
$A \vdash_6 T'(x,y,z) \leq \frac{3\epsilon}{2}\sum_{i}{\sum_{j \neq i}{y^2_i\left(x^2_j + z^2_j + \frac{1}{2}y^2_j(||x||^2 + ||z||^2)\right)}}$
\end{lemma}
\begin{proof}
It is enough to show the following inequality for all $i\in [n]$,
\begin{displaymath}
  \cA \vdash_6 {y^2_i}T'(x,y,z) \leq \epsilon\sum_{j \neq i}{\left(\frac{3}{2}y^2_i(x^2_j + z^2_j) + \frac{1}{2}{y^2_i}{y^2_j}(||x||^2 + ||z||^2)\right)}
\end{displaymath}
By symmetry it suffices to consider the case $i=1$.
Let $x' = x - x_1\cdot e_1$, $y' = y - y_1\cdot e_1$, and $z' = z - z_1\cdot e_1$.
We observe that 
\begin{align*}
\cA \vdash_4 T'(x,y,z) &= T'(x_1e_1 + x',y_1e_1 + y',z_1e_1 + z') \\
&= T'({x_1}e_1,y',z') + T'(x',{y_1}e_1,z') \\
&+ T'(x',y',{z_1}e_1) + T'(x',y',z')
\end{align*}
We now apply the following inequalities
\begin{enumerate}
\item $\cA \vdash_4 T'({x_1}e_1,y',z') \leq \frac{\epsilon}{2}\left(x^2_1||y'||^2 + ||z'||^2\right) \leq \frac{\epsilon}{2}\sum_{j \neq 1}(y^2_j||x||^2 + z^2_j)$
\item $\cA \vdash_4 T'(x',y_1{e_1},z') \leq \frac{\epsilon}{2}\left(||x'||^2{y^2_1} + ||z'||^2\right) \leq \frac{\epsilon}{2}\sum_{j \neq 1}(x^2_j + z^2_j)$
\item $\cA \vdash_4 T'(x',y',z_1{e_1}) \leq \frac{\epsilon}{2}\left(z^2_1||y'||^2 + ||x'||^2\right) \leq \frac{\epsilon}{2}\sum_{j \neq 1}(y^2_j||z||^2 + x^2_j)$
\item $\cA \vdash_4 T'(x',y',z') \leq \frac{\epsilon}{2}\left(||x'||^2||y'||^2 + ||z'||^2\right) \leq \frac{\epsilon}{2}\sum_{j \neq 1}(x^2_j + z^2_j)$
\end{enumerate}
\end{proof}

We can now prove \cref{lem:pseudo-expectation}.

\begin{proof}[Proof of \cref{lem:pseudo-expectation}]
  Taken together, \cref{lem:bound-on-first-part,lem:bound-on-second-part} imply
  \begin{multline}
    \cA \vdash_6 \iprod{T,x\otimes y\otimes z} 
    \le \tfrac 12 \norm{x}^2 + \tfrac 12 \norm{z}^2 - (\tfrac 1 4 -O(\e))\sum_{i=r+1}^n (x_i^2 + z_i^2)\\
    - (\tfrac 18 - O(\e)) \sum_{i\neq j} y_i^2 \cdot \Paren{\vbig x_j^2 + z_j^2 + y_j^2 \cdot (\norm{x}^2 + \norm{z}^2)}\,,
  \end{multline}
  where the absolute constant hidden by $O(\cdot)$ notation is at most $10$.
  Therefore for $\e<1/100$, as we assumed in \cref{def:certificate-higher},
  we get a SOS proof of the inequality,
  \begin{multline}
    \cA \vdash_6 \iprod{T,x\otimes y\otimes z} 
    \le \tfrac 12 \norm{x}^2 + \tfrac 12 \norm{z}^2 - \tfrac 1 8 \sum_{i=r+1}^n (x_i^2 + z_i^2)\\
    - \tfrac 1{16} \sum_{i\neq j} y_i^2 \cdot \Paren{\vbig x_j^2 + z_j^2 + y_j^2 \cdot (\norm{x}^2 + \norm{z}^2)}\,,
  \end{multline}
  This SOS proof implies that that every degree-6 pseudo-distribution $\mu(x,y,z)$ with $\mu\models \cA$ satisfies the desired inequality,
  \begin{multline}
    \pE_{\mu(x,y,z)}\iprod{T,x\otimes y\otimes z} 
    \le \pE_{\mu(x,y,z)} \tfrac 12 \norm{x}^2 + \tfrac 12 \norm{z}^2 - \tfrac 1 8 \sum_{i=r+1}^n (x_i^2 + z_i^2)\\
    - \tfrac 1{16} \sum_{i\neq j} y_i^2 \cdot \Paren{\vbig x_j^2 + z_j^2 + y_j^2 \cdot (\norm{x}^2 + \norm{z}^2)}\,,
  \end{multline}
\end{proof}

\subsection{Constructing the certificate $T$}
\label{sec:certificate-higher}
In this section we give a procedure for constructing the certificate $T$. This construction is directly inspired by the construction of the dual certificate in \cite{DBLP:journals/tit/Gross11,DBLP:journals/jmlr/Recht11} (sometimes called quantum golfing). We will then prove that $T$ satisfies all of the conditions for a higher-degree certificate of $\Omega$. In Section \ref{sec:certificate} we will show that $T$ also satisfies the conditions for a degree-4 certificate for $\Omega$. 

Let $\{u_i\},\{v_i\},\{w_i\}\subseteq \R^n$ be three orthonormal bases, with all vectors $\mu$-incoherent.
Let $X=\sum_{i=1}^r u_i\otimes v_i\otimes w_i$.
Let $\Omega\subseteq[n]^3$ chosen at random such that each element is included independently with probability $m/n^{1.5}$ (so that $\card{\Omega}$ is tightly concentrated around $m$).

Let $P$ be the projector on the span of the vectors $u_i\otimes v_j \otimes w_k$ such that an index in $[r]$ appears at least twice in $(i,j,k)$ (i.e., at least one of the conditions $i=j\in [r]$, $i=k\in[r]$, $j=k\in[r]$ is satisfied).
Let $R_\Omega$ be the linear operator on $\R^n\otimes \R^n\otimes \R^n$ that sets all entries outside of $\Omega$ to $0$ (so that $(R_\Omega[ T])_\Omega = R_\Omega[T]$) and is scaled such that $\E_\Omega R_\Omega=\Id$.
Let $\bar R_\Omega$ be $\Id-R_\Omega$.

Our goal is to construct $T\in \R^n\otimes \R^n \otimes \R^n$ such that $P[T]=X$, $(T)_\Omega=T$, and the spectral norm condition in \cref{def:certificate} is satisfied.
The idea for constructing $T$ is to start with $T=X$.
Then, move to closest point $T'$ that satisfies $R_\Omega[T']=T'$
Then, move to closest point $T''$ that satisfies $P [T''] =X$ and repeat.
To implement this strategy, we define
\begin{equation}
  \label{eq:construction}
  \super T k = \sum_{j=0}^{k-1} (-1)^{j}R_{\Omega_{j+1}} (P \bar R_{\Omega_{j}})\cdots (P \bar R_{\Omega_1}) [X]\,,
\end{equation}
where $\Omega_1,\ldots,\Omega_k$ are iid samples from the same distribution as $\Omega$.

By induction, we can show the following lemma about linear constraints that the constructed stensors $\super T k$ satisfy.

\begin{lemma}
  For every $k\ge 1$, the tensor $\super T k$ satisfies $(T)_\Omega=T$ and
  \begin{displaymath}
    P [\super T k] + (-1)^{k} P(P\bar R_{\Omega_k})\cdots(P\bar R_{\Omega_1}) [X] = X\,.
  \end{displaymath}
\end{lemma}

Here, $P(P\bar R_{\Omega_k})\cdots(P\bar R_{\Omega_1}) [X]$ is an error term that decreases geometrically.
In the parameter regime of \cref{thm:certificate}, the norm of this term is $n^{-\omega(1)}$ for some $k=(\log n)^{O(1)}$.

The following lemma shows that it is possible to correct such small errors.
This lemma also implies that the linear independence condition in \cref{def:certificate} is satisfied with high probability.
(Therefore, we can ignore this condition in the following.)

\begin{lemma}
  \label[lemma]{lem:correct-error}
  Suppose $m\ge r n \mu \cdot (\log n)^{O(1)}$.
  Then, with probability $1-n^{-\omega(1)}$ over the choice of $\Omega$, the following holds:
  For every $E\in\R^n\otimes\R^n\otimes\R^n$ with $P[E]=E$, there exists $Y$ with $(Y)_\Omega=Y$ such that $P[Y]=E$ and $\norm{Y}_F\le O(1)\cdot \norm{E}_F$.
\end{lemma}

\begin{proof}
  Let $S\subseteq [n]^3$ be such that $P$ is the projector to the vectors $u_i\otimes v_j\otimes w_k$ with $(i,j,k)\in S$.
  By construction of $P$ we have $\card{S}\le 3 r n$.
  In order to show the conclusion of the lemma it is enough to show that the vectors $(u_i\otimes v_j\otimes w_k)_{\Omega}$ with $(i,j,k)\in S$ are well-conditioned in the sense that the ratio of the largest and smallest singular value is $O(1)$.
  This fact follows from standard matrix concentration inequalities.
  See \cref{lem:simple-spectral}.
\end{proof}

The main technical challenge is to show that the construction satisfies the condition that the following degree 4 polynomials are sums of squares (where $T' = T - X$).   
\begin{gather}
    \norm{x}^2 + \norm{y}^2\cdot \norm{z}^2 - 1/\e \cdot \iprod{T',x\otimes y \otimes z}\,,  \\
    \norm{y}^2 + \norm{x}^2\cdot \norm{z}^2-   1/\e \cdot \iprod{T',x\otimes y \otimes z}\,,\\
    \norm{z}^2 + \norm{x}^2\cdot \norm{y}^2-   1/\e \cdot \iprod{T',x\otimes y \otimes z}\,.
\end{gather}
We show how to prove the first statement, the other statements can be proved with symmetrical arguments. To prove the first statement, we decompose $T'$ into pieces of the form $(\bar{R}_{\Omega_l}P)\cdots(\bar{R}_{\Omega_1}P)(\bar{R}_{\Omega_0}X)$, $P'(\bar{R}_{\Omega_l}P)\cdots(\bar{R}_{\Omega_1}P)(\bar{R}_{\Omega_0}X)$ (where $P'$ is a part of $P$), or $E$. For each piece $A$, we prove a norm bound $\norm{\sum_{a}{A_a \otimes A_a^T}} \leq B$. Since $\sum_{a}{A_a \otimes A_a^T}$ represents the same polynomial as $\transpose{A}A$, this proves that $B\norm{y}^2\norm{z}^2 - \transpose{(y \otimes z)}\transpose{A}A(y \otimes z)$ is a degree 4 sum of squares. Now note that $\transpose{(y \otimes z)}\transpose{A}A(y \otimes z) - \sqrt{B}\transpose{x}A(y \otimes z) - \sqrt{B}\transpose{(y \otimes z)}\transpose{A}x + B\norm{x}^2$ is also a sum of squares. Combining these equations and scaling we have that $\norm{x}^2 + \norm{y}^2\norm{z}^2 - \frac{2}{\sqrt{B}}\transpose{x}A(y \otimes z)$ is a degree 4 sum of squares.

Thus, it is sufficient to prove norm bounds on $\norm{\sum_{a}{A_a \otimes A_a^T}}$. We have an appropriate bound in the case when $A = E$ because $E$ has very small Frobenius norm. For the cases when $A = (\bar{R}_{\Omega_l}P)\cdots(\bar{R}_{\Omega_1}P)(\bar{R}_{\Omega_0}X)$ or $A = P(\bar{R}_{\Omega_l}P)\cdots(\bar{R}_{\Omega_1}P)(\bar{R}_{\Omega_0}X)$, we use the following theorem
\begin{theorem}
  \label{thm:squarenormboundtheorem}
  Let $A = (\bar{R}_{\Omega_l}P)\cdots(\bar{R}_{\Omega_1}P)(\bar{R}_{\Omega_0}X)$ or $P'(\bar{R}_{\Omega_l}P)\cdots(\bar{R}_{\Omega_1}P)(\bar{R}_{\Omega_0}X)$ where $P'$ is a part of $P$.
  There is an absolute constant $C$ such that for any $\alpha > 1$ and $\beta > 0$, 
  $$\Pr\left[\Norm{\sum_{a}{A_a \otimes A_a^T}} > \alpha^{-(l+1)}\right] < n^{-\beta}$$
  as long as $m > C\alpha\beta\mu^{\frac{3}{2}}r n^{1.5}\cdot\log(n)$ and
  $m > C\alpha\beta\mu^{2}rn\log(n)$.
\end{theorem}
\begin{proof}
This theorem follows directly from combining Proposition \ref{actualtensorbounds}, Theorem \ref{noncrosstermtheorem}, Theorem \ref{iterativebound}, and Theorem \ref{noncrosstermtheoremwithP}.
\end{proof}

\subsubsection{Final correction of error terms}
\label{sec:bernstein}

In this section, we prove a spectral norm bound that allows us to correct error terms that are left at the end of the construction.
The proof uses the by now standard Matrix Bernstein concentration inequality.
Similar proofs appear in the matrix completion literature \cite{DBLP:journals/tit/Gross11,DBLP:journals/jmlr/Recht11}.

Let $\{u_i\},\{v_i\},\{w_i\}\subseteq \R^n$ be three orthonormal bases, with all vectors $\mu$-incoherent.
Let $\Omega\subseteq[n]^3$ be $m$ entires sampled uniformly at random with replacement.
(This sampling model is different from what is used in the rest of the proof.
However, it is well known that the models are equivalent in terms of the final recovery problem.)

\begin{lemma}
  \label[lemma]{lem:simple-spectral}
  Let $S\subseteq [n]^3$.
  Suppose $m= \mu \card{S} (\log n)^C$ for an absolute constant $C\ge 1$.
  Then with probability $1-n^{\omega(1)}$ over the choice of $\Omega$, the vectors $(u_i\otimes v_j\otimes w_k)_\Omega$ for $(i,j,k)\in S$ are well-conditioned in the sense that the ratio between the largest and smallest singular value is at most $1.1$.
\end{lemma}

\begin{proof}
  For $s=(i,j,k)\in S$, let $y_s = u_i\otimes v_j \otimes w_k$.
  Let $\Omega=\{\omega_1,\ldots,\omega_m\}$, where $\omega_1,\ldots,\omega\in [n]^3$ are sampled uniformly at random with replacement.
  Let $A$ be the $S$-by-$S$ Gram matrix of the vectors $(y_s)_\Omega$.
  Then, $A$ is the sum of $m$ identically distributed rank-1 matrices $A_i$,
  $$
  A = \sum_{i=1}^m A_i \quad\text{with}\quad (A_i)_{s,s'} = (y_s)_{\omega_i}\cdot (y_{s'})_{\omega_i}\,.
  $$
  Each $A_i$ has expectation $\E A_i = n^{-1.5} \Id$ and spectral norm at most $\card{S}\cdot\mu/n^{1.5}$.
  Standard matrix concentration inequalities \cite{DBLP:journals/focm/Tropp12} show that $m\ge O(\card{S}\mu^2\log n)$ is enough to ensure that the sum is spectral close to its expectation $(m/n^{1.5})\Id$ in the sense that $0.99 A \preceq (m/n^{1.5}) \Id \preceq 1.1 A$.
\end{proof}

\subsection{Degree-4 certificates imply exact recovery}
\label{sec:recovery}

In this section we prove \cref{thm:recovery}.
We need the following technical lemma, which we prove in \cref{sec:zero-matching}.

\begin{lemma}

\label[lemma]{lem:zero-matching}
  Let $R$ be self-adjoint linear operator $R$ on $\R^n\otimes \R^n$.
  Suppose $\iprod{(v_j\otimes w_k), R (v_{i}\otimes w_{i})}=0$ for all indices $i,j,k\in [r]$ such that $i\in \{j,k\}$.
  Then, there exists a self-adjoint linear operator $R'$ on $\R^n\otimes \R^n$ such that $R' (v_i\otimes w_i)=0$ for all $i\in [r]$, the spectral norm of $R'$ satisfies $\norm{R'}\le 10\norm{R}$, and $R'$ represents the same polynomial in $\R[y,z]$,
  \begin{displaymath}
    \iprod{(y\otimes z), R' (y\otimes z)} = \iprod{(y\otimes z), R (y\otimes z)}\,.
  \end{displaymath}
\end{lemma}

We can now prove that certificates in the sense of \cref{def:certificate} imply that our algorithm successfully recoves the unknown tensor.

\begin{proof}[Proof of \cref{thm:recovery}]
  Let $T$ be a certificate in the sense of \cref{def:certificate}.
  
  Our goal is to construct a positive semidefinite matrix $M$ on $\R^n\oplus \R^n\otimes \R^n$ that represents the following polynomial
  \begin{displaymath}
    \iprod{(x,y\otimes z), M (x,y\otimes z)} = \norm{x}^2 + \norm{y}^2\cdot \norm{z}^2 - 2 \iprod{x,T (y\otimes z)}\,.
  \end{displaymath}
  Let $T_a$ be matrices such that $\iprod{x,T(x\otimes y)}=\sum_a x_a \cdot T_a(y,z)$.
  Since $\norm{x}^2 + \sum_{a=1}^n T_a(y,z)^2  - 2\iprod{x,T(y\otimes z)}=\norm{x-T (y\otimes z)}$ is a sum of squares of polynomials, it will be enough to find a positive semidefinite matrix that represents the polynomial $\norm{y}^2\cdot \norm{z}^2 - \sum_{a=1}^n T_a(y,z)^2$.
  (This step is a polynomial version of the Schur complement condition for positive semidefiniteness.)
  Let $R$ be the following linear operator
  \begin{displaymath}
    R= \sum_{a=1}^n T_a \otimes \transpose{T_a} - \sum_{i=1}^r \dyad{(v_i\otimes w_i)}\,,
  \end{displaymath}
  \begin{lemma}

\label[lemma]{lem:checkingRcondition}
  $R$ satisfies the requirement of \cref{lem:zero-matching}.
\end{lemma}
\begin{proof}
Consider $\iprod{(v_j\otimes w_k), R (v_{j}\otimes w_{j})}$. Since $v_j$ is repeated, the value of this expression will be the same if we replace $R$ by an $R_2$ which represents the same polynomial. Thus, we can replace $R$ by $R_2 = \sum_{a=1}^n{\transpose{T_a}T_a} - \sum_{i=1}^r \dyad{(v_i\otimes w_i)} = \transpose{T}T - \sum_{i=1}^r \dyad{(v_i\otimes w_i)}$

We now observe that $\iprod{(v_j\otimes w_k), R_2 (v_{j}\otimes w_{j})} = 
\iprod{(v_j\otimes w_k),\transpose{T}(u_j) - (v_{j}\otimes w_{j})} = 0$.
By a symmetrical proof, $\iprod{(v_j\otimes w_k), R (v_{k}\otimes w_{k})} = 0$ as well.
\end{proof}
  By Lemma \cref{lem:zero-matching}, there exists a self-adjoint linear operator $R'$ that represents the same polynomial as $R$, has spectral norm $\norm{R'}\le 10\norm{R}\le 0.1$, and sends all vectors $v_i\otimes w_i$ to $0$.
  Since $R'$ sends all vectors $v_i\otimes w_i$ to $0$ and $\norm{R'}\le 0.1$, the following matrix
  \begin{displaymath}
    R''=\sum_{i=1}^r \dyad{(v_i\otimes w_i)} + R'
  \end{displaymath}
  has $r$ eigenvalues of value $1$ (corresponding to the space spanned by $v_i\otimes w_i$) and all other eigenvalues are at most $0.1$ (because the non-zero eigenvalues of $R'$ have eigenvectors orthogonal to all $v_i\otimes w_i$).
  At the same time, $R''$ represents the following polynomial,
  \begin{displaymath}
    \iprod{(y\otimes z),R''(y\otimes z)} = \sum_{a=1}^n T_a(y,z)^2\,.
  \end{displaymath}
  Let $P$ be a positive semidefinite matrix that represents the polynomial $\norm{x}^2 + \sum_{a=1}^n T_a(y,z)^2  - 2\iprod{x,T(y\otimes z)}$ (such a matrix exists because the polynomial is a sum of squares).
  We choose $M$ as follows
  \begin{displaymath}
    M =
    \Paren{\begin{matrix}
        \Id & -T \\
        \transpose {(T)} & \transpose {(T)} T 
      \end{matrix}}
    + \Paren{\begin{matrix}
        0 & 0 \\
        0 & \Id -R''
      \end{matrix}}
  \end{displaymath}
  Since $R''\preceq \Id$, this matrix is positive semidefinite.
  Also, $M$ represents $\norm{x}^2+\norm{y}^2\cdot \norm{z}^2-2\iprod{x,T(y \otimes z)}$.
  Since $u_i=T (v_i\otimes w_i)$ for all $i\in [r]$ and the kernel of $\Id-R''$ only contains span of $v_i\otimes w_i$, the kernel of $M$ is exactly the span of the vectors $(u_i,v_i\otimes w_i)$.
  
  Next, we show that the above matrix $M$ implies that \cref{alg:tensor} recovers the unknown tensor $X$.
  Recall that the algorithm on input $X_\Omega$ finds a pseudo-distribution $\mu(x,y,z)$ so as to minimize $\pE_{\mu}\norm{x}^2 + \norm{y}^2\cdot \norm{z}^2$ such that $(\pE_{\mu}x\otimes y \otimes z)_\Omega=X_\Omega$.
  Since everything is scale invariant, we may assume that $X=\sum_{i=1}^r \lambda_i \cdot u_i\otimes v_i\otimes w_i$ for $\lambda_1,\ldots,\lambda_r\ge 0$ and $\sum_i \lambda_i = 1$.
  Then, a valid pseudo-distribution would be the probability distribution over $(u_1,v_1,w_1),\ldots,(u_r,v_r,w_r)$ with probabilities $\lambda_1,\ldots,\lambda_r$.
  Let $\mu$ be the pseudo-distribution computed by the algorithm.
  By optimality of $\mu$, we know that the objective value satisfies $\pE_{\mu} \norm{x}^2+\norm{y}^2\cdot \norm{z}^2\le \E_{i\sim \lambda} \norm{u_i}^2+\norm{v_i}^2\cdot \norm{w_i}^2=2$.
  Then, if we let $Y=\E_{\mu} \dyad{(x, y \otimes z)}$,
  \begin{align*}
    0\le \iprod{M,Y} 
    &= \pE_{\mu(x,y,z)} \norm{x}^2 + \norm{y}^2\cdot \norm{z}^2 - 2\iprod{x,T (y\otimes z)}\\
     &\le 2 - 2\pE_{\mu(x,y,z)}\iprod{x,T (y\otimes z)}\\
     &= 2 - 2\E_{i\sim \lambda}\iprod{u_i,T (v_i\otimes w_i)}\\ 
        & = 0 
  \end{align*}
  The first step uses that $M$ and $Y$ are psd.
  The second step uses that $M$ represents the polynomial $\norm{x}^2 + \norm{y}^2\cdot \norm{z}^2 - 2\iprod{x,T (y\otimes z)}$.
  The third step uses that $\mu$ minimizes the objective function.
  The fourth step uses that the entries of $T$ are $0$ outside of $\Omega$ and that $\mu$ matches the observations $(\pE_{\mu} x\otimes y \otimes z)_\Omega=X_\Omega$.
  The last step uses that $u_i=T (v_i\otimes w_i)$ for all $i\in [r]$.

  We conclude that $\iprod{M,Y}=0$, which means that the range of $Y$ is contained in the kernel of $M$.
  Therefore, $Y=\sum_{i,j=1}^r \gamma_{i,j}\cdot (u_i, v_i\otimes w_i)\transpose{(u_j, v_j\otimes w_j)}$ for scalars $\{\gamma_{i,j}\}$.
  We claim that the multipliers must satisfy $\gamma_{i,i}=\lambda_i$ and $\gamma_{i,j}=0$ for all $i\neq j\in [r]$.
  Indeed since $\mu$ matches the observations in $\Omega$, 
  \begin{displaymath}
    0 = \sum_{i,j=1}^n (\lambda_i - \gamma_{i,j}\delta_{ij}) \cdot (u_i\otimes v_j\otimes w_j)_\Omega \,.
  \end{displaymath}
  Since the vectors $(u_i\otimes v_j\otimes w_j)_\Omega$ are linearly independent, we conclude that $\gamma_{i,j}=\lambda_{i}\cdot \delta_{ij}$ as desired.
  (This linear independence was one of the requirements of the certificate in \cref{def:certificate}.)
\end{proof}

\subsection{Degree-4 certificates exist with high probability}
\label{sec:certificate}

In this section we show that our certificate $T$ in fact satisfies the conditions for a degree-4 certificate, proving \cref{thm:certificate}. 

We use the same construction as in \cref{sec:certificate-higher}.
The main, remaining technical challenge for \cref{thm:certificate} is to show that the construction satisfies the spectral norm condition of \cref{def:certificate}.
This spectral norm bound follows from the following theorem which we give a proof sketch for in \cref{sec:full-trace-power}.
\begin{theorem}
  \label{thm:normboundtheorem}
  Let $A = (\bar{R}_{\Omega_l}P)\cdots(\bar{R}_{\Omega_1}P)(\bar{R}_{\Omega_0}X)$ or $P(\bar{R}_{\Omega_l}P)\cdots(\bar{R}_{\Omega_1}P)(\bar{R}_{\Omega_0}X)$ and let $B = (\bar{R}_{\Omega_{l'}}P)\cdots(\bar{R}_{\Omega_1}P)(\bar{R}_{\Omega_0}X)$ or $P(\bar{R}_{\Omega_{l'}}P)\cdots(\bar{R}_{\Omega_1}P)(\bar{R}_{\Omega_0}X)$. 
  There is an absolute constant $C$ such that for any $\alpha > 1$ and $\beta > 0$, 
  $$\Pr\left[\Norm{\sum_{a}{A_a \otimes B_a^T}} > \alpha^{-(l+l'+2)}\right] < n^{-\beta}$$
  as long as $m > C\alpha\beta\mu^{\frac{3}{2}}r n^{1.5}\cdot\log(n)$ and
  $m > C\alpha\beta\mu^{2}rn\log(n)$.
\end{theorem}
\begin{remark}
If it were true in general that $||\sum_{a}{A_{a} \otimes B^T_a}|| \leq \sqrt{||\sum_{a}{A_{a} \otimes A^T_{a}}||}\sqrt{||\sum_{a}{B_{a} \otimes B^T_{a}}||}$ then it would be sufficient to use Theorem \ref{thm:squarenormboundtheorem} and we would not need to prove Theorem \ref{thm:normboundtheorem}. Unfortunately, this is not true in general.

That said, it may be possible to show that even if we do not know directly that $||\sum_{a}{A_{a} \otimes B^T_a}||$ is small, since $||\sum_{a}{A_{a} \otimes A^T_{a}}||$ and $||\sum_{a}{B_{a} \otimes B^T_{a}}||$ are both small there must be some alternative matrix representation of $\sum_{a}{A_{a} \otimes B^T_a}$ which has small norm, and this is sufficient. We leave it as an open problem whether this can be done.
\end{remark}

We have now all ingredients to prove \cref{thm:certificate}.

\begin{proof}[Proof of \cref{thm:certificate}]
  Let $k=(\log n)^C$ for some absolute constant $C\ge 1$.
  Let $E= (-1)^k P(P\bar R_{\Omega_k})\cdots(P\bar R_{\Omega_1}) [X]$.
  By \cref{lem:correct-error} there exists $Y$ with $(Y)_\Omega=Y$ and $P[Y]=E$ such that $\norm{Y}_F\le O(1)\norm{E}$.
  We let $T=\super T k + Y$.
  This tensor satisfies the desired linear constraints $(T)_\Omega=T$ and $P[T]=X$.
  Since $E$ has the form of the matrices in \cref{thm:normboundtheorem}, the bound in \cref{thm:normboundtheorem} implies $\norm{E}_F\le 2^{-k}\cdot n^{10}\le n^{-C+10}$.
  (Here, we use that the norm in the  conclusion of \cref{thm:normboundtheorem} is within a factor of $n^{10}$ of the Frobenius norm.)
  
  We are to prove that the following matrix has spectral norm bounded by $0.01$,
  \begin{displaymath}
    \sum_{a=1}^n (T)_a \otimes (T)_a^T - \sum_{a=1}^n X_a \otimes X_a^T\,.
  \end{displaymath}
  We expand the sum according to the definition of $\super T \ell$ in \cref{eq:construction}.
  Then, most terms that appear in the expansion have the form as in \cref{thm:normboundtheorem}.
  Since those terms decrease geometrically, we can bound their contribution by $0.001$ with probability $1-n^{-\omega(1)}$.
  The terms that involve the error correction $Y$ is smaller than $0.001$ because $Y$ has polynomially small norm $\norm{Y}_F\le n^{-C+10}$.
  The only remaining terms are cross terms between $X$ and a tensor of the form as in \cref{thm:normboundtheorem}.
  We can bound the total contribution of these terms also bounded by at most $0.001$ using \cref{xcrosstermnormboundtheorem}.
\end{proof}

\section{Matrix norm bound techniques}\label{normmethodssection}
In this section, we describe the techniques that we will use to prove probabilistic norm bounds on matrices of the form $Y = \sum_{a}{(\bar{R}_{\Omega}A)_a \otimes (\bar{R}_{\Omega}A)^T_a}$. We will prove these norm bounds using the trace moment method, which obtains probabilistic bounds on the norm of a matrix $Y$ from bounds on the expected value of $tr((YY^T)^q)$ for sufficiently large $q$. This will require analyzing $tr((YY^T)^q)$, which will take the form of a sum of products, where the terms in the product are either entries of $A$ or terms of the form $\bar{R}_{\Omega}(a,b,c)$ where $\bar{R}_{\Omega}(a,b,c) = \frac{n^3}{m} - 1$ if $(a,b,c) \in \Omega$ and $-1$ otherwise. To analyze $tr((YY^T)^q)$, we will group products together which have the same expected behavior on the $\bar{R}_{\Omega}(a,b,c)$ terms, forming smaller sums of products. For each of these sums, we can then use the same bound on the expected behavior of the $\bar{R}_{\Omega}(a,b,c)$ terms for each product in the sum. This allows us to move this bound outside of the sum, leaving us with a sum of products of entries of $A$. We will then bound the value of these sums by carefully choosing the order in which we sum over the indices.

In the reainder of this section and in the next two sections, we allow for our tensors to have asymmetric dimensions. We account for this with the following definitions.
\begin{definition}
We define $n_1$ to the dimension of the $u$ vectors, $n_2$ to be the dimension of the $v$ vectors, and $n_3$ to be the dimension of the $w$ vectors. We define $n_{max} = \max{\{n_1,n_2,n_3\}}$
\end{definition}
\subsection{The trace moment method}
We use the trace moment method through the following proposition and corollary.
\begin{proposition}\label{powertonormboundprop}
For any random matrix $Y$, for any integer $q \geq 1$ and any $\epsilon > 0$, 
$$Pr\left[||Y|| > \sqrt[2q]{\frac{E\left[tr((YY^T)^q)\right]}{\epsilon}}\right] < \epsilon$$
\end{proposition}
\begin{proof}
By Markov's inequality, for all integers $q \geq 1$ and all $\epsilon > 0$
$$Pr\left[tr((YY^T)^q) > \frac{E\left[tr((YY^T)^q)\right]}{\epsilon}\right] < \epsilon$$
The result now follows from the observation that if 
$||Y|| > \sqrt[2q]{\frac{E\left[tr((YY^T)^q)\right]}{\epsilon}}$ then $tr((YY^T)^q) > \frac{E\left[tr((YY^T)^q)\right]}{\epsilon}$.
\end{proof}
\begin{corollary}\label{powertonormboundcorollary}
For a given $p \geq 1$, $r \geq 0$, $n > 0$, and $B > 0$, for a random matrix $Y$, if $E\left[tr\left((YY^T)^q\right)\right] \leq ({q^p}B)^{2q}n^r$ for all integers $q > 1$ then for all $\beta > 0$, 
$$Pr\left[||Y|| > Be^p\left(\frac{(r+\beta)}{2p}\ln{n} + 1\right)^p\right] < n^{-\beta}$$
\end{corollary}
\begin{proof}
We take $\epsilon = n^{-\beta}$ and we choose $q$ to minimize $\sqrt[2q]{\frac{({q^p}B)^{2q}n^r}{\epsilon}} = B{q^p}n^{\frac{r+\beta}{2q}}$. Setting the derivative of this expression to $0$ we obtain that $(\frac{p}{q}-\frac{r+\beta}{2q^2}\ln{n})B{q^p}n^{\frac{r+\beta}{2q}} = 0$, so we want $q = \frac{r+\beta}{2p}\ln{n}$. However, $q$ must be an integer, so we instead take $q = \lceil{\frac{r+\beta}{2p}\ln{n}}\rceil$. With this $q$, we have that
$$B{q^p}n^{\frac{r+\beta}{2q}} \leq B\left(\frac{r+\beta}{2p}\ln{n} + 1\right)^p{n^{\frac{p}{\ln{n}}}} = Be^p\left(\frac{(r+\beta)}{2p}\ln{n} + 1\right)^p$$
Applying Proposition \ref{powertonormboundprop} with $q$, we obtain that
$$Pr\left[||Y|| > \sqrt[2q]{\frac{E\left[tr((YY^T)^q)\right]}{\epsilon}}\right] 
\leq Pr\left[||Y|| > Be^p\left(\frac{(r+\beta)}{2p}\ln{n} + 1\right)^p\right] < n^{-\beta}$$
\end{proof}
\subsection{Partitioning by intersection pattern}
As discussed at the beginning of the section, $E\left[tr((YY^T)^q)\right]$ will be a sum of products, where part of these products will be of the form $\prod_{i=1}^{2q'}{\bar{R}_{\Omega}(a_i,b_i,c_i)}$. Here, $q'$ may or may not be equal to $q$, in fact we will often have $q' = 2q$ because each $Y$ will contribute two terms of the form $\bar{R}_{\Omega}(a,b,c)$ to the product. To handle this part of the product, we partition the terms of our sum based on the intersection pattern of which triples $(a_i,b_i,c_i)$ are equal to each other. Fixing an intersection pattern determines the expected value of $\prod_{i=1}^{2q'}{\bar{R}_{\Omega}(a_i,b_i,c_i)}$.
\begin{definition}
We define an intersection pattern to be a set of equalities and inequalities satisfying the following conditions
\begin{enumerate}
\item All of the equalities and inequalities are of the form $(a_{i_1},b_{i_1},c_{i_1}) = (a_{i_2},b_{i_2},c_{i_2})$ or $(a_{i_1},b_{i_1},c_{i_1}) \neq (a_{i_2},b_{i_2},c_{i_2})$, respectively.
\item For every $i_1,i_2$, either $(a_{i_1},b_{i_1},c_{i_1}) = (a_{i_2},b_{i_2},c_{i_2})$ is in the intersection pattern or $(a_{i_1},b_{i_1},c_{i_1}) \neq (a_{i_2},b_{i_2},c_{i_2})$ is in the intersection pattern
\item All of the equalities and inequalities are consistent with each other, i.e. there exist values of $(a_{1},b_{1},c_{1}),\cdots,(a_{2q},b_{2q},c_{2q})$ satisfying all of the equalities and inequalities in the intersection pattern.
\end{enumerate}
\end{definition}
\begin{proposition}
For a given $(a,b,c)$,
\begin{enumerate}
\item $E\left[\bar{R}_{\Omega}(a,b,c)\right] = 0$
\item For all $k > 1$, $E\left[\left(\bar{R}_{\Omega}(a,b,c)\right)^k\right] \leq \left(\frac{{n_1}{n_2}{n_3}}{m}\right)^{k-1}$
\end{enumerate}
\end{proposition}
\begin{corollary}\label{patterncorollary}
For a given intersection pattern, if there is any triple $(a,b,c)$ which appears exactly once, 
$E\left[\prod_{i=1}^{2q'}{\bar{R}_{\Omega}(a_i,b_i,c_i)}\right] = 0$. Otherwise, letting $z$ be the number of distinct triples, $E\left[\prod_{i=1}^{2q'}{\bar{R}_{\Omega}(a_i,b_i,c_i)}\right] \leq \left(\frac{{n_1}{n_2}{n_3}}{m}\right)^{2q'-z}$
\end{corollary}
\begin{proof}
for a given intersection pattern, let $(a_{i_1},b_{i_1},c_{i_1}),\cdots,(a_{i_z},b_{i_z},c_{i_z})$ be the distinct triples and let $c_j$ be the number of times the triple $(a_{i_j},b_{i_j},c_{i_j})$ appears. We have that 
$$E\left[\prod_{i=1}^{2q'}{\bar{R}_{\Omega}(a_i,b_i,c_i)}\right] = \prod_{j=1}^{z}{E\left[\left(\bar{R}_{\Omega}(a_{i_j},b_{i_j},c_{i_j})\right)^{c_j}\right]}$$
If $c_j = 1$ for any $j$ then this expression is $0$. Otherwise, 
$$\prod_{j=1}^{z}{E\left[\left(\bar{R}_{\Omega}(a_{i_j},b_{i_j},c_{i_j})\right)^{c_j}\right]}
\leq \prod_{j=1}^{z}{\left(\frac{{n_1}{n_2}{n_3}}{m}\right)^{c_j-1}} = \left(\frac{{n_1}{n_2}{n_3}}{m}\right)^{\left(\sum_{j=1}^{z}{c_j}\right) - z} = \left(\frac{{n_1}{n_2}{n_3}}{m}\right)^{2q'-z}$$
\end{proof}
\subsection{Bounding sums of products of tensor entries}
In this subsection, we describe how to bound the sum of products of tensor entries we obtain for a given intersection pattern after moving our bound on the expected value of the $\bar{R}_{\Omega}(a,b,c)$ terms outside the sum. We represent such a product with a hypergraph as follows.
\begin{definition}
Given a set of distinct indices and a set of tensor entries on those indices, let $H$ be the hypergraph with one vertex for each distinct index and one hyperedge for each tensor entry, where the hyperedge consists of all indices contained in the tensor entry. If the tenor entry appears to the pth power, we take this hyperedge with multiplicity $p$.
\end{definition}
With this definition in mind, we will first preprocess our products.
\begin{enumerate}
\item We will preprocess the tensor entries so that every entry appears to an even power using the inequality $|ab| \leq \frac{1}{2}(a^2 + b^2)$. This has the effect of taking two hyperedges of our choice in $H$ and replacing them with one doubled hyperedge or the other (we have to consider both possibilities). Note that this step makes all of our terms positive and can only increase their magnitude, so the result will be an upper bound on our actual sum.
\item We will add the missing terms to our sum so that for we sum over every possibility for the distinct indices (even the possibilities which make several of these indices equal and would put us in a different intersection pattern). Note that this can only increase our sum.
\end{enumerate}
\begin{remark}
It is important that we first bound the expected value of the $\bar{R}_{\Omega}(a,b,c)$ terms and move this bound outside of our sum before adding the missing terms to the sum.
\end{remark}
After preprocessing our products, our strategy will be as follows. We will sum over the indices, removing the corresponding vertices from $H$. As we do this, we will apply appropriate bounds on squared tensor entries, removing the corresponding doubled hyperedge from $H$. To obtain these bounds, we observe that we can bound the average square of our tensor entries in terms of the number of indices we are averaging over.
\begin{definition}
We say that an order 3 tensor $A$ of dimensions $n_1 \times n_2 \times n_3$ is $(B,r,\mu)$-bounded if the following bounds are true
\begin{enumerate}
\item $\max_{a,b,c}{\{A^2_{abc}\}} \leq Br$
\item $\max{\{\max_{b,c}{\{\frac{1}{n_1}\sum_{a}{A^2_{abc}}\}},
\max_{a,c}{\{\frac{1}{n_2}\sum_{b}{A^2_{abc}}\}},
\max_{a,b}{\{\frac{1}{n_3}\sum_{c}{A^2_{abc}}\}}\}} \leq \frac{B}{\mu}$
\item $\max{\{\max_{c}{\{\frac{1}{{n_1}{n_2}}\sum_{a,b}{A^2_{abc}}\}},
\max_{b}{\{\frac{1}{{n_1}{n_3}}\sum_{a,c}{A^2_{abc}}\}},
\max_{a}{\{\frac{1}{{n_2}{n_3}}\sum_{b,c}{A^2_{abc}}\}}\}} \leq \frac{B}{\mu^2}$
\item $\frac{1}{{n_1}{n_2}{n_3}}\sum_{a,b,c}{A^2_{abc}} \leq \frac{B}{\mu^3}$
\end{enumerate}
More generally, we say that a tensor $A$ is $(B,r,\mu)$-bounded if the following is true
\begin{enumerate}
\item The maximum value of an entry of $A$ squared is at most $Br$
\item Every index which we average over decreases our upper bound by a factor of $\mu$
\item If we are averaging over at least one index then we can delete the factor of $r$ in our bound.
\end{enumerate}
Since $r$ and $\mu$ will always be the same, we write $B$-bounded rather than $(B,r,\mu)$-bounded
\end{definition}
To give a sense of why these are the correct type of bounds to use, we now show that $X$ is $\left(\frac{r\mu^3}{{n_1}{n_2}{n_3}}\right)$-bounded. In Section \ref{iterativesection}, we will use an iterative argument to show that with high probability, similar bounds hold for all of the tensors $A$ we will be considering.
\begin{proposition}\label{actualtensorbounds}
$X$ is $\left(\frac{r\mu^3}{{n_1}{n_2}{n_3}}\right)$-bounded
\end{proposition}
\begin{proof}
Recall that $X = \sum_{i=1}^{r}{u_i \otimes v_i \otimes w_i}$ where the vectors $\{u_i\}$ are orthonormal, the vectors $\{v_i\}$ are orthonormal, and the vectors $\{w_i\}$ are orthonormal. Also recall that for all $i,a,b,c$, $u^2_{ia} \leq \frac{\mu}{n_1}$, $v^2_{ib} \leq \frac{\mu}{n_2}$, and $w^2_{ic} \leq \frac{\mu}{n_3}$. We now have the following bounds:
\begin{enumerate}
\item $$\max_{a,b,c}{\{X^2_{abc}\}} = \max_{a,b,c}{\left\{\sum_{i=1}^{r}{\sum_{i'=1}^{r}{u_{ia}v_{ib}w_{ic}u_{i'a}v_{i'b}w_{i'c}}}\right\}} \leq \frac{r^2\mu^3}{{n_1}{n_2}{n_3}}$$
\item 
\begin{align*}
\max_{b,c}{\left\{\frac{1}{n_1}\sum_{a}{X^2_{abc}}\right\}} &= \frac{1}{n_1}\max_{b,c}{\left\{\sum_{a}{\sum_{i=1}^{r}{\sum_{i'=1}^{r}{u_{ia}v_{ib}w_{ic}u_{i'a}v_{i'b}w_{i'c}}}}\right\}} \\
&= \frac{1}{n_1}\max_{b,c}{\left\{\sum_{i=1}^{r}{\sum_{i'=1}^{r}{\left(\sum_{a}{u_{ia}u_{i'a}}\right)v_{ib}w_{ic}v_{i'b}w_{i'c}}}\right\}} \\
&= \frac{1}{n_1}\max_{b,c}{\left\{\sum_{i=1}^{r}{v^2_{ib}w^2_{ic}}\right\}} \\
&\leq \frac{r\mu^2}{{n_1}{n_2}{n_3}}
\end{align*}
The other bounds where we sum over one index follow by symmetrical arguments.
\item 
\begin{align*}\max_{c}{\left\{\frac{1}{{n_1}{n_2}}\sum_{a,b}{X^2_{abc}}\right\}} &= \frac{1}{{n_1}{n_2}}\max_{c}{\left\{\sum_{a,b}{\sum_{i=1}^{r}{\sum_{i'=1}^{r}{u_{ia}v_{ib}w_{ic}u_{i'a}v_{i'b}w_{i'c}}}}\right\}} \\
&= \frac{1}{{n_1}{n_2}}\max_{c}{\left\{\sum_{i=1}^{r}{\sum_{i'=1}^{r}{\left(\sum_{a}{u_{ia}u_{i'a}}\right)\left(\sum_{b}{v_{ib}v_{i'b}}\right)w_{ic}w_{i'c}}}\right\}} \\
&=\frac{1}{{n_1}{n_2}}\max_{c}{\left\{\sum_{i=1}^{r}{w^2_{ic}}\right\}} \\
&\leq \frac{r\mu}{{n_1}{n_2}{n_3}}
\end{align*}
The other bounds where we sum over two indices follow by symmetrical arguments.
\item 
\begin{align*}\frac{1}{{n_1}{n_2}{n_3}}\sum_{a,b,c}{X^2_{abc}} &= \frac{1}{{n_1}{n_2}{n_3}}\sum_{a,b,c}{\sum_{i=1}^{r}{\sum_{i'=1}^{r}{u_{ia}v_{ib}w_{ic}u_{i'a}v_{i'b}w_{i'c}}}} \\
&= \frac{1}{{n_1}{n_2}{n_3}}\sum_{i=1}^{r}{\sum_{i'=1}^{r}{\left(\sum_{a}{u_{ia}u_{i'a}}\right)\left(\sum_{b}{v_{ib}v_{i'b}}\right)\left(\sum_{c}{w_{ic}w_{i'c}}\right)}} \\
&=\frac{1}{{n_1}{n_2}{n_3}}\sum_{i=1}^{r}{1} \\
&= \frac{r}{{n_1}{n_2}{n_3}}
\end{align*}
\end{enumerate}
\end{proof}
With these kinds of bounds in mind, we bound sums of products of tensor entries as follows. We note that we can always apply the entrywise bound for a squared tensor entry. However, to apply any of the other bounds, we must be able to sum over an index or indices where the only term in our product which depends on this index or indices is the squared tensor entry. This can be described in terms of the hypergraph $H$ as follows.
\begin{definition}
Given a hyperedge $e$ in $H$, define $b(e)$ to the the minimal $B$ such that the tensor entry corresponding to $e$ is $B$-bounded.
\end{definition}
\begin{definition}
We say that a vertex is free in $H$ if it contained in only one hyperedge and this hyperedge appears with multiplicity two.
\end{definition}
We can apply our bounds in the following ways.
\begin{enumerate}
\item We can always choose a hyperedge $e$ of $H$, use the entrywise bound of $rb(e)$ on the corresponding squared tensor entry (note the extra factor of $r$), and reduce the multiplicity of $e$ by two.
\item If there is a free vertex incident with a doubled hyperedge $e$ in $H$, we can sum over all free vertices which are incident with $e$ using the corresponding bound then delete these vertices and the doubled hyperedge $e$ from $H$. When we do this, we obtain a factor of 
$$b(e)\left(\frac{n_1}{\mu}\right)^{\text{\# of deleted a vertices}}
\left(\frac{n_2}{\mu}\right)^{\text{\# of deleted b vertices}}\left(\frac{n_3}{\mu}\right)^{\text{\# of deleted c vertices}}$$
The factors of $n_1,n_2,n_3$ appear because we are summing over these indices and the factors of $\frac{1}{\mu}$ appear because each index we sum over reduces the bound on the average value by a factor of $\mu$.
\end{enumerate}
If we apply these bounds repeatedly until there are no tensor entries/hyperedges left to bound, our final bound on a single sum of products of tensor entries will be
$$\left(\prod_{e \in H}{\sqrt{b(e)}}\right)\left(\frac{n_1}{\mu}\right)^{\# \text{ of } a \text{ indices}}
\left(\frac{n_2}{\mu}\right)^{\# \text{ of } b \text{ indices}}
\left(\frac{n_3}{\mu}\right)^{\# \text{ of } c \text{ indices}}r^{\# \text{ of entrywise bounds used}}$$
To prove our final upper bound, we will argue that we can always apply these bounds in such a way that the number of times we need to use an entrywise bound is sufficiently small.
\subsection{Counting intersection patterns}
There will be one more factor in our final bound. This factor will come from the number of possible intersection patterns with a given number $z$ of distinct triples $(a,b,c)$.
\begin{lemma}\label{patterncounting}
The total number of intersection patterns on $2q'$ triples with $z$ distinct triples $(a,b,c)$ such that every triple $(a,b,c)$ has multiplicity at least two 
is at most ${\binom{2q'}{z}}z^{2q'-z} \leq 2^{2q'}{q'}^{2q'-z}$
\end{lemma}
\begin{proof}
To determine which triples $(a,b,c)$ are equal to each other, it is sufficient to decide which triples are distinct from all previous triples (there are ${\binom{2q'}{z}}$ choices for this) and for the remaining $2q'-z$ triples, which of the $z$ distinct triples they are equal to (there are $z^{2q'-z}$ choices for this). 
\end{proof}

\section{Trace Power Calculation for $\bar{R}_{\Omega}A \otimes (\bar{R}_{\Omega}A)^T$}\label{squaretracepowercalculationsection}
In this section, we implement the techniques described in Section \ref{normmethodssection} to probabilistically bound $||\bar{R}_{\Omega}A \otimes (\bar{R}_{\Omega}A)^T||$. In particular, we prove the following theorem.
\begin{theorem}\label{noncrosstermtheorem}
If $A$ is $B$-bounded, $C \geq 1$, and 
\begin{enumerate}
\item $m > 10000C(2+\beta)^2{n_{max}}r\mu^2\ln{n_{max}}$
\item $m > 10000C(2 + \beta)^2{r}\sqrt{n_1}\max{\{n_2,n_3\}}\mu^{\frac{3}{2}}\ln{n_{max}}
\geq 10000C(2 + \beta)^2{r}\sqrt{{n_1}{n_2}{n_3}}\mu^{\frac{3}{2}}\ln{n_{max}}$
\item ${\mu}r \leq \min{\{n_1,n_2,n_3\}}$
\end{enumerate}
then defining $Y = \bar{R}_{\Omega}A \otimes (\bar{R}_{\Omega}A)^T$,
$$Pr\left[||Y|| > \frac{B{n_1}{n_2}{n_3}}{Cr\mu^3}\right] < 4n^{-(\beta+1)}$$
\end{theorem}
\begin{corollary}
If $C \geq 1$ and 
\begin{enumerate}
\item $m > 10000C(2+\beta)^2{n_{max}}r\mu^2\ln{n_{max}}$
\item $m > 10000C(2 + \beta)^2{r}\sqrt{n_1}\max{\{n_2,n_3\}}\mu^{\frac{3}{2}}\ln{n_{max}}
\geq 10000C(2 + \beta)^2{r}\sqrt{{n_1}{n_2}{n_3}}\mu^{\frac{3}{2}}\ln{n_{max}}$
\item ${\mu}r \leq \min{\{n_1,n_2,n_3\}}$
\end{enumerate}
then $$Pr\left[||\bar{R}_{\Omega}X \otimes (\bar{R}_{\Omega}X)^T|| > \frac{1}{C}\right] < 4n^{-(\beta+1)}$$
\end{corollary}
\begin{proof}
This follows immediately from Theorem \ref{noncrosstermtheorem} and the fact that $X$ is $\left(\frac{r\mu^3}{{n_1}{n_2}{n_3}}\right)$-bounded.
\end{proof}
To prove Theorem \ref{noncrosstermtheorem}, we break up $Y$ into four parts and then prove probabilistic norm bounds for each part.
\begin{definition} \ 
\begin{enumerate}
\item Define $(Y_1)_{bcb'c'} = Y_{bcb'c'}$ if $b = b'$, $c = c'$ and $0$ otherwise.
\item Define $(Y_2)_{bcb'c'} = Y_{bcb'c'}$ if $b = b'$, $c \neq c'$ and $0$ otherwise.
\item Define $(Y_3)_{bcb'c'} = Y_{bcb'c'}$ if $b \neq b'$, $c = c'$ and $0$ otherwise.
\item Define $(Y_4)_{bcb'c'} = Y_{bcb'c'}$ if $b \neq b'$, $c \neq c'$ and $0$ otherwise.
\end{enumerate}
\end{definition}
\subsection{Structure of $tr((Y_j{Y_j^T})^q)$}
We have that $Y_{bcb'c'} = \sum_{a}{\bar{R}_{\Omega}(a,b,c')\bar{R}_{\Omega}(a,b',c)A_{abc'}A_{ab'c}}$. To see the structure of $({Y_j}Y_j^T)^q$, we now compute ${Y_j}Y_j^T$. 
\begin{align*}
&({Y_j}Y_j^T)_{{b_1}{c_1}{b_2}{c_2}} = \\
&\sum_{a_1,a_2,b',c'}{\bar{R}_{\Omega}({a_1},{b_1},c')\bar{R}_{\Omega}({a_1},b',c_1)\bar{R}_{\Omega}({a_2},{b_2},c')\bar{R}_{\Omega}({a_2},b',c_2)A_{{a_1}{b_1}c'}A_{{a_1}b'{c_1}}A_{{a_2}{b_2}c'}A_{{a_2}b'{c_2}}}
\end{align*}
where the sum is taken over $b',c'$ which satisfy the appropriate constraints. The $\bar{R}_{\Omega}$ terms will not be part of our hypergraph $H$ (as their expected behavior is determined by the intersection pattern). We can view the first two terms $A_{{a_1}{b_1}c'}$ and $A_{{a_1}b'{c_1}}$ as an hourglass with upper triangle $(b_1,a_1,c')$ and lower triangle $(c_1,a_1,b')$ (where the vertices in each triangle are listed from left to right). Similarly, we can view the last two terms $A_{{a_2}{b_2}c'}$ and $A_{{a_2}b'{c_2}}$ as an hourglass with upper triangle $(c',a_2,b_2)$ and lower triangle $(b',a_2,c_2)$. Thus, the hypergraph $H$ corresponding to $tr((Y_j{Y_j^T})^q)$ will be $2q$ hourglasses glued together where the top vertices of the hourglass alternate between $b$ and $c'$ indices, the bottom vertices of the hourglass alternate between $c$ and $b'$ indices, and the middle vertices of the hourglass are the $a$ indices.
\begin{remark}
While there is no real difference between the $b$ and $b'$ indices and between the $c$ and $c'$ indices, we will keep track of this to make it easier to see the structure of $H$.
\end{remark}
As described in Section \ref{normmethodssection}, we split up $E\left[tr((Y_j{Y^T_j})^q)\right]$ based on the intersection pattern of which of the $4q$ triples of the form $(a,b,c')$ or $(a,b',c)$ are equal to each other. We only need to consider patterns where each triple and thus each hyperedge appears at least twice, as otherwise the terms in the sum will have expected value $0$. In all cases, letting $z$ be the number of distinct triples in a given intersection pattern, by Corollary \ref{patterncorollary} our bound on the expected value of the $\bar{R}_{\Omega}$ terms will be $\left(\frac{{n_1}{n_2}{n_3}}{m}\right)^{4q-z}$
\subsection{Bounds on $||Y_1||$}
Consider $E\left[tr((Y_1{Y^T_1})^q)\right]$. The constraints that $b' = b$ and $c' = c$ in every $Y$ force all of the $b$ and $b'$ indices to be equal and all of the $c$ and $c'$ indices to be equal, so our hypergraph $H$ consists of a single vertex $b$, a single vertex $c$, and two copies of the hyperedge $(a_i,b,c)$ for each $i \in [1,2q]$. For all intersection patterns, the number of distinct triples $z$ is equal to the number of distinct $a$ indices, which can be anywhere from $1$ to $2q$.

We apply our bounds on $H$ as follows.
\begin{enumerate}
\item In our preprocessing step, when there are two hyperedges $e_1$ and $e_2$ which appear with odd multiplicity, we double one of these hyperedges or the other. Thus, we can assume that all hyperedges appear with even multiplicity.
\item We will apply an entrywise bound $2q-z$ times on hyperedges of multiplicity $\geq 4$, reducing the multiplicity by $2$ each time.
\item After applying these entrywise bounds, all of the distinct $a$ vertices will be free and we can sum up over these indices one by one.
\end{enumerate}
Recall that the bound from the $R_{\Omega}$ terms is $\left(\frac{{n_1}{n_2}{n_3}}{m}\right)^{4q-z}$ and our bound for the other terms is
$$\left(\prod_{e \in H}{\sqrt{b(e)}}\right)\left(\frac{n_1}{\mu}\right)^{\# \text{ of } a \text{ entries}}
\left(\frac{n_2}{\mu}\right)^{\# \text{ of } b \text{ entries}}
\left(\frac{n_3}{\mu}\right)^{\# \text{ of } c \text{ entries}}r^{\# \text{ of entrywise bounds used}}$$
where $b(e) = B$ for all our hyperedges.
Summing over all $z \in [1,2q]$ and all intersection patterns using Lemma \ref{patterncounting}, our final bound is
$$2q \cdot 2^{4q}\max_{z \in [1,2q]}{\left\{(2q)^{4q-z}\left(\frac{{n_1}{n_2}{n_3}}{m}\right)^{4q-z}B^{2q}\left(\frac{n_1}{\mu}\right)^{z}\left(\frac{n_2}{\mu}\right)\left(\frac{n_3}{\mu}\right)r^{2q-z}\right\}}$$
The inner expression will either be maximized at $z = 2q$ or $z = 1$ and we will always take $q$ to be between $\frac{\ln{n_{max}}}{2}$ and $\frac{n_{max}}{2}$, so our final bound on $E\left[tr((Y_1{Y^T_1})^q)\right]$ is at most
$$(4q)^{4q}\max{\left\{n_{max}\left(\frac{{n^2_1}{n_2}{n_3}B}{m\mu\ln{n_{max}}}\right)^{2q}\left(\frac{n_2}{\mu}\right)\left(\frac{n_3}{\mu}\right),\left(\frac{{n^2_1}{n^2_2}{n^2_3}rB}{m^2}\right)^{2q}\frac{m}{r\mu^3}\right\}}$$
Since $m > 10000C(2+\beta)^2{n_{max}}r\mu^2\ln{n_{max}}$ and $m > 10000C(2 + \beta)^2{r}\sqrt{{n_1}{n_2}{n_3}}\mu^{\frac{3}{2}}\ln{n_{max}}$, we have that 
$$E\left[tr((Y_1{Y^T_1})^q)\right] < (16q^2)^{2q}
\left(\frac{{n_1}{n_2}{n_3}B}{10000C(2 + \beta)^2r\mu^3(\ln{n_{max}})^2}\right)^{2q}n^3_{max}$$
(note that $m < n^3_{max}$ as otherwise the tensor completion problem is trivial). We now recall Corollary \ref{powertonormboundcorollary}, which says that for a given $p \geq 1$, $r \geq 0$, $n > 0$, and $B > 0$, for a random matrix $Y$, if $E\left[tr\left((YY^T)^q\right)\right] \leq ({q^p}B)^{2q}n^r$ for all integers $q > 1$ then for all $\beta > 0$, 
$$Pr\left[||Y|| > Be^p\left(\frac{(r+\beta)}{2p}\ln{n} + 1\right)^p\right] < n^{-\beta}$$
Using Corollary \ref{powertonormboundcorollary} with the appropriate parameters, we can show that for all $\beta > 0$,
$$P\left[||Y_1|| > \frac{16{e^2}B{n_1}{n_2}{n_3}}{10000r\mu^3}\right] < n_{max}^{-(\beta+1)}$$

\subsection{Bounds on $||Y_2||$ and $||Y_3||$}
Consider $E\left[tr((Y_2{Y^T_2})^q)\right]$. The constraint that $b' = b$ in every $Y$ forces all of the $b$ and $b'$ indices to be equal, so our hypergraph $H$ consists of a single vertex $b$ and $4q$ total hyperedges of the form $(a,b,c)$ or $(a,b,c')$. Ignoring the $b$ vertex (which is part of all the hyperedges), the $(a,c)$ and $(a,c')$ edges form a single connected component. We only need to consider intersection patterns where each triple $(a,b,c)$ or $(a,b,c')$ (and thus each edge $(a,c)$ or $(a,c')$) appears with multiplicity at least two. For a given intersection pattern, let $z$ be the number of distinct edges.

We apply our bounds on $H$ as follows.
\begin{enumerate}
\item In our preprocessing step, when there are two edges $e_1$ and $e_2$ which appear with odd multiplicity, we double one of these edges or the other. Thus, we can assume that all edges appear with even multiplicity.
\item We will apply an entrywise bound $2q-z$ times on edges of multiplicity $\geq 4$, reducing the multiplicity by $2$ each time.
\item After applying these entrywise bounds, all of our edges will have multiplicity $2$. We now sum over a free $a$, $c$, or $c'$ vertex in $H$ whenever such a vertex exists. Otherwise, there must be a cycle, in which case we use the entrywise bound on one edge of the cycle and delete it.
\end{enumerate}
\begin{definition}
Let $x$ be the number of times we delete an edge in a cycle using the entrywise bound.
\end{definition}
\begin{lemma}\label{failureslemmaone}
The total number of vertices in $H$ (excluding $b$) is $z+1-x$
\end{lemma}
\begin{proof}
Observe that neither deleting a free vertex nor deleting an edge in a cycle can disconnect $H$. Also, except for the final edge where both of its vertices will be free, every edge which has a free vertex has exactly one free vertex. Thus, we delete an edge in a cycle $x$ times, removing $0$ vertices each time, we delete an edge with one free vertex $z-x-1$ times, removing $1$ vertex each time, and we delete the final edge once, removing the final two vertices. This adds up to $z+1-x$ vertices in $H$.  
\end{proof}
Recall that the bound from the $R_{\Omega}$ terms is $\left(\frac{{n_1}{n_2}{n_3}}{m}\right)^{4q-z}$ and our bound for the other terms is
$$\left(\prod_{e \in H}{\sqrt{b(e)}}\right)\left(\frac{n_1}{\mu}\right)^{\# \text{ of } a \text{ entries}}
\left(\frac{n_2}{\mu}\right)^{\# \text{ of } b \text{ entries}}
\left(\frac{n_3}{\mu}\right)^{\# \text{ of } c \text{ or } c'  \text{ entries}}r^{\# \text{ of entrywise bounds used}}$$
where $b(e) = B$ for all our hyperedges.
Summing over all $z \in [1,2q]$ and all intersection patterns using Lemma \ref{patterncounting}, our final bound is
$$2q \cdot 2^{4q}\max_{z \in [1,2q],x \in [0,z-1]}
{\left\{(2q)^{4q-z}\left(\frac{{n_1}{n_2}{n_3}}{m}\right)^{4q-z}B^{2q}\left(\frac{n_{max}}{\mu}\right)^{z-1-x}\left(\frac{{n_1}{n_2}{n_3}}{\mu^3}\right)r^{2q-z+x}\right\}}$$
Since ${\mu}r \leq n_{max}$, the inner expression will either be maximized when $z = 2q$ and $x = 0$ or when $z = 1$ and $x = 0$. Again, we will always take $q$ to be between $\frac{\ln{n_{max}}}{2}$ and $\frac{n_{max}}{2}$, so our final bound on $E\left[tr((Y_2{Y^T_2})^q)\right]$ is at most
$$(4q)^{(4q)}\max{\left\{\left(\frac{{n_1}{n_2}{n_3}{n_{max}}B}{m\mu\ln{n_{max}}}\right)^{2q}\left(\frac{{n_1}{n_2}{n_3}}{\mu^2}\right),\left(\frac{{n^2_1}{n^2_2}{n^2_3}rB}{m^2}\right)^{2q}\frac{m}{r\mu^3}\right\}}$$
Since $m > 10000C(2+\beta)^2{n_{max}}r\mu^2\ln{n_{max}}$ and $m > 10000C(2 + \beta)^2{r}\sqrt{{n_1}{n_2}{n_3}}\mu^{\frac{3}{2}}\ln{n_{max}}$, we have that 
$$E\left[tr((Y_2{Y^T_2})^q)\right] < (16q^2)^{2q}
\left(\frac{{n_1}{n_2}{n_3}B}{10000C(2 + \beta)^2{r}\mu^3(\ln{n_{max}})^2}\right)^{2q}n^3_{max}$$
Using Corollary \ref{powertonormboundcorollary} with the appropriate parameters (in fact the same ones as before), we can show that for all $\beta > 0$,
$$P\left[||Y_2|| > \frac{16{e^2}B{n_1}{n_2}{n_3}}{10000r\mu^3}\right] < n_{max}^{-(\beta+1)}$$
By a symmetrical argument, we can obtain the same probabilistic bound on $||Y_3||$.

\subsection{Bounds on $||Y_4||$}
Consider $E\left[tr((Y_4{Y^T_4})^q)\right]$. Our hypergraph $H$ consists of $2q$ hyperedges of the form $(b,a,c')$ or $(c',a,b)$ from the top triangles of the hourglasses and $2q$ hyperedges of the form $(c,a,b')$ or $(b',a,c)$ from the bottom triangles of the hourglasses. We only need to consider intersection patterns where each triple (and thus each hyperedge) appears with multiplicity at least two. For a given intersection pattern, let $z$ be the number of distinct hyperedges.

Ignoring the $a$ vertices for now, we can think of $H$ as a graph on the $b$, $b'$, $c$, and $c'$ vertices. Note that the $(b,c')$ and $(c',b)$ edges are part of a single connected component and the $(c,b')$ and $(b',c)$ edges are part of a single connected component (these connected components may or may not be the same).

We apply our bounds on $H$ as follows.
\begin{enumerate}
\item In our preprocessing step, when there are two hyperedges $e_1$ and $e_2$ which appear with odd multiplicity, we double one of these hyperedges or the other. Thus, we can assume that all hyperedges appear with even multiplicity.
\item We will apply an entrywise bound $2q-z$ times on hyperedges of multiplicity $\geq 4$, reducing the multiplicity by $2$ each time.
\item After applying these entrywise bounds, all of our hyperedges will have multiplicity $2$. We now sum over a free $b$,$b'$,$c$, or $c'$ vertex in $H$ whenever such a vertex exists. Otherwise, there must be a cycle on the $(b,c')$ and $(b',c)$ parts of the hyperedges, in which case we use the entrywise bound on one hyperedge of the cycle and delete it.
\end{enumerate}
\begin{definition}
Let $x$ be the number of times we delete a hyperedge in a cycle using the entrywise bound.
\end{definition}
\begin{lemma}\label{failureslemmatwo}
Let $k$ be the number of connected components of $H$. The total number of $b$,$b'$,$c$, and $c'$ vertices in $H$ is $z+ k -x \leq z+2-x$
\end{lemma}
\begin{proof}
The proof is similar to the proof of Lemma \ref{failureslemmaone}. Observe that neither deleting a free vertex nor deleting an edge in a cycle can disconnect a connected component of $H$. Also, except for the final edge of a connected component where both of its vertices will be free, every edge which has a free vertex has exactly one free vertex. Thus, we delete an edge in a cycle $x$ times, removing $0$ vertices each time, we delete an edge with one free vertex $z-x-k$ times, removing $1$ vertex each time, and we delete the final edge of a connected component $k$ times, removing the final $2k$ vertices. This adds up to $z+k-x$ vertices in $H$. For the inequality, recall that $H$ has at most 2 connected components, one for the $(b,c')$ edges and one for the $(c,b')$ edges.
\end{proof}
Finally, we bound the number of distinct $a$ indices
\begin{proposition}
The number of distinct $a$ indices is at most $\frac{z}{2}$.
\end{proposition}
\begin{proof}
Note that by the definition of $Y_4$, every $a$ index must be part of at least two distinct hyperedges.
\end{proof}
Recall that the bound from the $R_{\Omega}$ terms is $\left(\frac{{n_1}{n_2}{n_3}}{m}\right)^{4q-z}$ and our bound for the other terms is
$$\left(\prod_{e \in H}{\sqrt{b(e)}}\right)\left(\frac{n_1}{\mu}\right)^{\# \text{ of } a \text{ entries}}
\left(\frac{n_2}{\mu}\right)^{\# \text{ of } b \text{ or } b' \text{ entries}}
\left(\frac{n_3}{\mu}\right)^{\# \text{ of } c \text{ or } c' \text{ entries}}r^{\# \text{ of entrywise bounds used}}$$
where $b(e) = B$ for all our hyperedges.
Summing over all $z \in [2,2q]$ and all intersection patterns using Lemma \ref{patterncounting}, our final bound is
$$2q \cdot 2^{4q}\max_{z \in [1,2q],x \in [0,z-2]}
{\left\{(2q)^{4q-z}\left(\frac{{n_1}{n_2}{n_3}}{m}\right)^{4q-z}B^{2q}\left(\frac{n_1}{\mu}\right)^{\frac{z}{2}}\left(\frac{\max{\{n_2,n_3\}}}{\mu}\right)^{z-2-x}\left(\frac{{n^2_2}{n^2_3}}{\mu^4}\right)r^{2q-z+x}\right\}}$$
Since ${\mu}r \leq \min{\{n_1,n_2,n_3\}}$, the inner expression will either be maximized when $z = 2q$ and $x = 0$ or when $z = 2$ and $x = 0$. Again, we will always take $q$ to be between $\frac{\ln{n_{max}}}{2}$ and $\frac{n_{max}}{2}$, so our final bound on $E\left[tr((Y_2{Y^T_2})^q)\right]$ is at most
$$(4q)^{(4q)}\max{\left\{\left(\frac{{n_1}{n_2}{n_3}{\sqrt{n_1}\max{\{n_2,n_3\}}}B}{m\mu^{\frac{3}{2}}\ln{n_{max}}}\right)^{2q}\left(\frac{n^3_{max}}{\mu^2}\right),\left(\frac{{n^2_1}{n^2_2}{n^2_3}rB}{m^2}\right)^{2q}\frac{m^2}{r^2\mu^5{n_1}}\right\}}$$
Since $m > 10000C(2 + \beta)^2{r}\sqrt{{n_1}}\max{\{{n_2},{n_3}\}}\mu^{\frac{3}{2}}\ln{n_{max}}$, we have that 
$$E\left[tr((Y_4{Y^T_4})^q)\right] < (16q^2)^{2q}
\left(\frac{{n_1}{n_2}{n_3}B}{10000C(2 + \beta)^2{r}\mu^3(\ln{n_{max}})^2}\right)^{2q}n^6_{max}$$
Using Corollary \ref{powertonormboundcorollary} with the appropriate parameters, we can show that for all $\beta > 0$,
$$P\left[||Y_4|| > \frac{16{e^2}B{n_1}{n_2}{n_3}}{10000r\mu^3}\right] < n_{max}^{-(\beta+1)}$$
Putting our four bounds together with a union bound, for all $\beta > 0$, 
$$P\left[||Y|| > \frac{B{n_1}{n_2}{n_3}}{r\mu^3}\right] \leq P\left[||Y|| > \frac{64{e^2}B{n_1}{n_2}{n_3}}{10000r\mu^3}\right] < 4n_{max}^{-(\beta+1)}$$
as needed.

\section{Iterative tensor bounds}\label{iterativesection}
In this section, we show that with high probability, applying the operator $P\bar{R}_{\Omega}$ to an order 3 tensor $A$ improves our bounds on it, where we are assuming that $\Omega$ is chosen independently of $A$.
\begin{theorem}\label{iterativebound}
If $A$ is a $B$-bounded tensor, $C \geq 1$, $\beta > 0$, and  
\begin{enumerate}
\item $m > 10000C(2+\beta)^2{n_{max}}r\mu^2\ln{n_{max}}$
\item $m > 10000C(2 + \beta)^2{r}\sqrt{n_1}\max{\{n_2,n_3\}}\mu^{\frac{3}{2}}\ln{n_{max}}
\geq 10000C(2 + \beta)^2{r}\sqrt{{n_1}{n_2}{n_3}}\mu^{\frac{3}{2}}\ln{n_{max}}$
\item ${\mu}r \leq \min{\{n_1,n_2,n_3\}}$
\end{enumerate}
then
$$Pr\left[P\bar{R}_{\Omega}A \text{ is not } 
\left(\frac{B}{C}\right) \text{-bounded}\right] < 100n_{max}^{-(\beta+1)}$$
\end{theorem}
\begin{proof}
We first consider how $P$ acts on a tensor
\begin{definition} \ 
\begin{enumerate}
\item Define $P^{UV}$ to be the projection onto $span\{u_i \otimes v_i \otimes w: i \in [1,r]\}$.
\item Define $P^{UW}$ to be the projection onto $span\{u_i \otimes v \otimes w_i: i \in [1,r]\}$.
\item Define $P^{VW}$ to be the projection onto $span\{u \otimes v_i \otimes w_i: i \in [1,r]\}$.
\item Define $P^{UVW}$ to be the projection onto $span\{u_i \otimes v_i \otimes w_i: i \in [1,r]\}$.
\end{enumerate}
\end{definition}
\begin{proposition}
$P = P_{UV} + P_{UW} + P_{VW} - 2P_{UVW}$
\end{proposition}
With this in mind, we break up the tensor $W = P\bar{R}_{\Omega}A$ into four parts and then obtain probabilistic bounds for each part. 
Theorem \ref{iterativebound} will then follow from the union bound and the inequality $(a+b+c-2d)^2 \leq 5(a^2 + b^2 + c^2 + 2d^2)$. 
\begin{definition} \ 
\begin{enumerate}
\item Define $W^{UV} = P_{UV}\bar{R}_{\Omega}A$.
\item Define $W^{UW} = P_{UW}\bar{R}_{\Omega}A$.
\item Define $W^{VW} = P_{VW}\bar{R}_{\Omega}A$.
\item Define $W^{UVW} = P_{UVW}\bar{R}_{\Omega}A$.
\end{enumerate}
\end{definition}
To analyze these parts, we reexpress $P_{UV},P_{UW},P_{VW},P_{UVW}$ in terms of matrices $UV,UW,VW,UVW$.
\begin{definition} \ 
\begin{enumerate}
\item Define $UV_{aba'b'} = \sum_{i=1}^{r}{u_{ia}v_{ib}u_{ia'}v_{ib'}}$
\item Define $UW_{aca'c'} = \sum_{i=1}^{r}{u_{ia}w_{ic}u_{ia'}w_{ic'}}$
\item Define $VW_{bcb'c'} = \sum_{i=1}^{r}{v_{ib}w_{ic}v_{ib'}w_{ic'}}$
\item Define $UVW_{abca'b'c'} = \sum_{i=1}^{r}{u_{ia}v_{ib}w_{ic}u_{ia'}v_{ib'}w_{ic'}}$
\end{enumerate}
\end{definition}
\begin{proposition}\label{projectionmatrixbounds} \ 
\begin{enumerate}
\item $UV$ is $\left(\frac{r\mu^4}{{n^2_1}{n^2_2}}\right)$-bounded.
\item $UW$ is $\left(\frac{r\mu^4}{{n^2_1}{n^3_2}}\right)$-bounded.
\item $VW$ is $\left(\frac{r\mu^4}{{n^2_2}{n^2_3}}\right)$-bounded.
\item $UVW$ is $\left(\frac{r\mu^6}{{n^2_1}{n^2_2}{n^2_3}}\right)$-bounded.
\end{enumerate}
\end{proposition}
\begin{proof}
These bounds can be proved in the same way as Proposition \ref{actualtensorbounds}.
\end{proof}
\begin{proposition} \ 
\begin{enumerate}
\item $W^{UV}_{abc} = \sum_{a',b'}{UV_{aba'b'}\bar{R}_{\Omega}(a',b',c)A_{a'b'c}}$
\item $W^{UW}_{abc} = \sum_{a',c'}{UV_{aca'c'}\bar{R}_{\Omega}(a',b,c')A_{a'bc'}}$
\item $W^{VW}_{abc} = \sum_{b',c'}{UV_{bcb'c'}\bar{R}_{\Omega}(a,b',c')A_{ab'c'}}$
\item $W^{UVW}_{abc} = \sum_{a',b',c'}{UV_{abca'b'c'}\bar{R}_{\Omega}(a',b',c')A_{a'b'c'}}$
\end{enumerate}
\end{proposition}
\begin{proposition} \ 
\begin{enumerate}
\item $\left(W^{UV}_{abc}\right)^2 = \sum_{a'_1,b'_1,a'_2,b'_2}{UV_{ab{a'_1}{b'_1}}UV_{ab{a'_2}{b'_2}}\bar{R}_{\Omega}(a'_1,b'_1,c)\bar{R}_{\Omega}(a'_2,b'_2,c)A_{{a'_1}{b'_1}c}A_{{a'_2}{b'_2}c}}$
\item $\left(W^{UW}_{abc}\right)^2 = \sum_{a'_1,c'_1,a'_2,c'_2}{UW_{ac{a'_1}{c'_1}}UW_{ac{a'_2}{c'_2}}\bar{R}_{\Omega}(a'_1,b,c'_1)\bar{R}_{\Omega}(a'_2,b,c'_2)A_{{a'_1}{b}{c'_1}}A_{{a'_2}{b}{c'_2}}}$
\item $\left(W^{VW}_{abc}\right)^2 = \sum_{b'_1,c'_1,b'_2,c'_2}{VW_{bc{b'_1}{c'_1}}VW_{bc{b'_2}{c'_2}}\bar{R}_{\Omega}(a,b'_1,c'_1)\bar{R}_{\Omega}(a,b'_2,c'_2)A_{a{b'_1}{c'_1}}A_{a{b'_2}{c'_2}}}$
\item \begin{align*}
&\left(W^{UVW}_{abc}\right)^2 = \\
&\sum_{a'_1,b'_1,c'_1,a'_2,b'_2,c'_2}{UVW_{abc{a'_1}{b'_1}{c'_1}}UVW_{abc{a'_2}{b'_2}{c'_2}}\bar{R}_{\Omega}(a'_1,b'_1,c'_1)\bar{R}_{\Omega}(a'_2,b'_2,c'_2)A_{{a'_1}{b'_1}{c'_1}}A_{{a'_2}{b'_2}{c'_2}}}
\end{align*}
\end{enumerate}
\end{proposition}
We need to probabilistically bound the expressions $\sum_{\text{subset of } \{a,b,c\}}{(W^{UV,UW,VW, \text{ or } UVW}_{a,b,c})^2}$.
For each expression which we need to probabilistically bound, we can obtain this bound by analyzing the expected value of its qth power using the techniques in Section \ref{normmethodssection} and then using a result similar to Corollary \ref{powertonormboundcorollary}. We begin by probabilistically bounding $\left(W^{UVW}_{abc}\right)^2$. As the remaining bounds will all be very similar, rather than giving a full proof of the remaining bounds we will only describe the few differences and what effect they have.
\begin{lemma}
For all $a,b,c$ and all $\beta > 0$, if $m > 10000C(2+\beta)^2{n_{max}}r\mu^2\ln{n_{max}}$ then $$P\left[\left(W^{UVW}_{abc}\right)^2 > \frac{32{e^2}Br\mu}{10000Cn_{max}}\right] < n_{max}^{-(\beta+4)}$$
\end{lemma}
\begin{proof}
Similar to before, we partition our sum based on the intersection pattern of which $(a'_i,b'_i,c'_i)$ are equal. Letting $z$ be the number of distinct triples $(a'_i,b'_i,c'_i)$, the contribution from the $\bar{R}_{\Omega}(a'_i,b'_i,c'_i)$ terms will be at most a factor of 
$\left(\frac{{n_1}{n_2}{n_3}}{m}\right)^{2q-z}$. Recall that for a given intersection pattern, our bound on the remaining terms is 
$$\left(\prod_{e \in H}{\sqrt{b(e)}}\right)\left(\frac{n_1}{\mu}\right)^{\# \text{ of } a \text{ or } a' \text{ indices}}
\left(\frac{n_2}{\mu}\right)^{\# \text{ of } b \text{ or } b'  \text{ indices}}
\left(\frac{n_3}{\mu}\right)^{\# \text{ of } c \text{ or } c'  \text{ indices}}r^{\# \text{ of entrywise bounds used}}$$
\begin{remark}
Here we will only be summing over $a'$,$b'$, and $c'$, indices, but for other expressions we will be summing over $a$, $b$, and $c$ indices as well.
\end{remark}
In our hypergraph $H$, we will have hyperedges $(a'_i,b'_i,c'_i)$ corresponding to the tensor entries $A_{a'_i,b'_i,c'_i}$ and we will have hyperedges $(a,b,c,a'_i,b'_i,c'_i)$ corresponding to the matrix entries $UVW_{a,b,c,a'_i,b'_i,c'_i}$. We have that 
$$\prod_{e \in H}{\sqrt{b(e)}} = {B^q}\left(\frac{r\mu^6}{{n^2_1}{n^2_2}{n^2_3}}\right)^q$$
We apply our techniques to $H$ as follows.
\begin{enumerate}
\item Recall that in our preprocessing step, we can take a pair of hyperedges $e_1,e_2$ and replace them with either a doubled copy of $e_1$ or a doubled copy of $e_2$. Using this, we ensure that every hyperedge appears with even multiplicity.

Here, we start with hyperedges $(a',b',c')$ where every distinct $(a',b',c')$ has multiplicity at least two and hyperedges $(a,b,c,a',b',c')$ where every distinct $(a,b,c,a',b',c')$ has multiplicity at least two ($a,b,c$ are the same for all of these hyperedges). Thus, in our preprocessing step, we can ensure that all of the hyperedges $(a',b',c')$ and $(a,b,c,a',b',c')$ occur with even multiplicity and every distinct hyperedge has multiplicity at least two. 
\item We apply an entrywise bound $q$ times to the $(a,b,c,a',b',c')$ hyperedges.
\item We will apply an entrywise bound $q-z$ times on hyperedges $(a',b',c')$ of multiplicity $\geq 4$, reducing the multiplicity by $2$ each time. After doing this, all our hyperedges will have multiplicity $2$. We now ignore the $c'$ vertices and consider the graph on the $a',b'$ vertices. We then sum over a free $a'$ or $b'$ vertex in $H$ whenever such a vertex exists. Otherwise, there must be a cycle (which could be a duplicated edge if we have hyper-edges $(a',b',c'_1)$ and $(a',b',c'_2)$), in which case we use the entrywise bound on one edge of the cycle and delete it.
\end{enumerate}
\begin{definition}
Let $x$ be the number of times we delete an edge in a cycle using the entrywise bound.
\end{definition}
\begin{lemma}
Let $k$ be the number of connected components of $H$. The total number of $a'$ and $b'$ vertices in $H$ is $z+k-x \leq 2z-2x$
\end{lemma}
\begin{proof}
The first part can be proved in exactly the same way as Lemma \ref{failureslemmatwo}. For the inequality, we need to show that $k \leq z-x$. To see this, note that there are at most $z$ distinct edges and every time we delete an edge in a cycle, this removes one edge without reducing the number of connected components. After removing all cycles (and no other edges), we must have at least as many edges left as we have connected components, so $z-x \geq k$, as needed.
\end{proof}
Summing over all $z \in [1,2q]$ and all intersection patterns using Lemma \ref{patterncounting} and noting that there are at most $z$ $a'$,$b'$,$c'$ indices but we must have two fewer $a'$ or $b'$ indices for each time we delete an edge in a cycle using an entrywise bound, our final bound on $E\left[\left(\left(W^{UVW}_{abc}\right)^2\right)^q\right]$ is
$$2q \cdot 2^{2q}\max_{z \in [1,2q],x \in [0,z-1]}
{\left\{q^{2q-z}\left(\frac{{n_1}{n_2}{n_3}}{m}\right)^{2q-z}\left(\frac{Br^2\mu^6}{{n^2_1}{n^2_2}{n^2_3}}\right)^q\left(\frac{{n_1}{n_2}{n_3}}{\mu^3}\right)^{z}\left(\frac{\mu}{\min{\{n_1,n_2\}}}\right)^{2x}r^{q-z+x}\right\}}$$
Since $m >> rq$ and ${\mu}r \leq \min{\{n_1,n_2,n_3\}}$, the inner expression will be maximized when $z = q$ and $x = 0$. Again, we will take $q$ to be between $\frac{\ln{n_{max}}}{2}$ and $\frac{n_{max}}{2}$ so our final bound on $E\left[\left(\left(W^{UVW}_{abc}\right)^2\right)^q\right]$ is at most
$$(2q)^{(2q)}\left(\frac{2Br^2\mu^3}{m\ln{n_{max}}}\right)^{q}n_{max}$$
Since $m > 10000C(2 + \beta)^2{r}n_{max}\mu^{2}\ln{n_{max}}$, we have that for all $a,b,c$.
$$E\left[\left(\left(W^{UVW}_{abc}\right)^2\right)^q\right] < q^{2q}
\left(\frac{8Br\mu}{10000C(2 + \beta)^2{n_{max}}(\ln{n_{max}})^2}\right)^{q}n_{max}$$
To obtain our final probabilistic bound, we adapt Corollary \ref{powertonormboundcorollary} for non-negative scalar expressions.
\begin{corollary}\label{adaptedpowertonormbound}
For a given $p \geq 1$, $r \geq 0$, $n > 0$, and $B > 0$, for a non-negative scalar expression $Z$, if $E[Z^q] \leq ({q^p}B)^{2q}n^r$ for all integers $q > 1$ then for all $\beta > 0$, 
$$Pr\left[|Z| > B^2e^{2p}\left(\frac{(r+\beta)}{2p}\ln{n} + 1\right)^{2p}\right] < n^{-\beta}$$
\end{corollary}
\begin{proof}
This can be proved in the same way as Corollary \ref{powertonormboundcorollary} except that $|Z|$ takes the place of $||YY^T||$ which is why the bound of Corollary \ref{powertonormboundcorollary} is squared.
\end{proof}
Using Corollary \ref{adaptedpowertonormbound} with the appropriate parameters, for all $a,b,c$
$$P\left[\left(W^{UVW}_{abc}\right)^2 > \frac{32{e^2}Br\mu}{10000Cn_{max}}\right] < n_{max}^{-(\beta+4)}$$
\end{proof}
The remaining bounds can be proved in a similar way, though there are a few differences. We now consider the remaining bounds involving $W^{UVW}$. When we average over at least one coordinate, our analysis is as follows:
\begin{enumerate}
\item The $(a,b,c,a',b',c')$ hyperedges no longer all have the same $(a,b,c)$. In fact, since the intersection patterns only specify which $(a'_i,b'_i,c'_i)$ are equal to each other, we treat all of the different $a,b,c$ as distinct indices.
\item For each $(a,b,c)$, we begin with two $(a,b,c,a',b',c')$ hyperedges which have this $(a,b,c)$ (though their $(a',b',c')$ may be different) To handle this, in our preprocessing step we take each such pair of $(a,b,c,a',b',c')$ hyperedges and double one or the other.
\item Averaging over the $a$, $b$, or $c$ indices, we avoid using entrywise bounds for any of the doubled $(a,b,c,a',b',c')$ hyperedges.
\item The analysis of the $(a',b',c')$ hyperedges is exactly the same
\end{enumerate}
Taking $Z$ to be the appropriate expression (for example, $Z = \frac{1}{n_1}\sum_{a}{\left(W^{UVW}_{abc}\right)^2}$ if we are only averaging over the $a$ index), our bound on $E[Z^q]$ is affected as follows:
\begin{enumerate}
\item Avoiding using the entrywise bounds on the $(a,b,c,a',b',c')$ hyperedges reduces our bound on $E[Z^q]$ by a factor of $r^q$.
\item If we average over $a$, this gives us $q$ additional $a$ indices to sum over, increasing our bound on $E[Z^q]$ by a factor of $\left(\frac{n_1}{\mu}\right)^q$, but this also gives us a factor of $\frac{1}{n^q_1}$ so the net effect is to reduce our bound on $E[Z^q]$ by a factor of $\mu^q$. Similar logic applies to $b$ and $c$, so each index we average over (including the first) reduces our bound on $E[Z^q]$ by a factor of $\mu^q$.
\end{enumerate}
This implies that each index we average over (including the first) reduces our final bound by a factor of $\mu$ and averaging over at least one index reduces our final bound by a further factor of $r$, as needed.

At this point, we just need to consider the bounds involving $W^{UV}$, as the remaining cases are symmetric. When we analyze $Z = \left(W^{UV}_{abc}\right)^2$ rather than $\left(W^{UVW}_{abc}\right)^2$, our analysis differs as follows. Instead of having $\left(\frac{Br^2\mu^6}{{n^2_1}{n^2_2}{n^2_3}}\right)^q$ in our bound on $E[Z^q]$ from the $(a,b,c,a',b',c')$ hyperedges, we will have $\left(\frac{Br^2\mu^4}{{n^2_1}{n^2_2}}\right)^q$ from $(a,b,a',b')$ hyperedges, increasing our bound on $E[Z^q]$ by a factor of $\left(\frac{n^2_3}{\mu^2}\right)^q$. However, this is partially counteracted by the fact that we are either no longer summing over the $c'$ indices separately from the $c$ indices because we always have that $c'_i = c_i$. This removes a factor of $(\frac{n_3}{\mu})^z$ from our bound on $E[Z^q]$. Thus, our bound on $E[Z^q]$ is now
$$2q \cdot 2^{2q}\max_{z \in [1,2q],x \in [0,z-1]}
{\left\{q^{2q-z}\left(\frac{{n_1}{n_2}{n_3}}{m}\right)^{2q-z}\left(\frac{Br^2\mu^4}{{n^2_1}{n^2_2}}\right)^q\left(\frac{{n_1}{n_2}}{\mu^2}\right)^{z}\left(\frac{\mu}{\min{\{n_1,n_2\}}}\right)^{2x}r^{q-z+x}\right\}}$$
We check that it is still optimal to take $z = q$ and $x = 0$. Since $r\mu \leq \min{\{n_1,n_2,n_3\}}$, it is always optimal to take $x = 0$. Now if we reduce $z$ by $1$, this gives us a factor of at most $\frac{q{n_1}{n_2}{n_3}}{m}\cdot\frac{r\mu^2}{{n_1}{n_2}} = \frac{qr\mu^2{n_3}}{m}$. We will take $q \leq 10(1+\beta)\ln{n_{max}}$ and we have that $m > 10000C(2+\beta)^2{n_{max}}r\mu^2\ln{n_{max}}$, so it is indeed still optimal to take $z = q$ and $x = 0$. Thus, the net effect of the differences is a factor of $\left(\frac{n_3}{\mu}\right)^q$ in our bound on $E[Z^q]$ which gives us a factor of $\frac{n_3}{\mu}$ in our final bound. This gives us the following bound.
\begin{lemma}
For all $a,b,c$ and all $\beta > 0$, if $m > 10000C(2+\beta)^2{n_{max}}r\mu^2\ln{n_{max}}$ then $$P\left[\left(W^{UV}_{abc}\right)^2 > \frac{32{e^2}Br}{10000C}\right] < n_{max}^{-(\beta+4)}$$
\end{lemma}
Finally, we consider what happens if we average over one or more of the $a$, $b$, and $c$ indices. If we average over the $a$ indices or average over the $b$ indices, then instead of using entrywise bounds on the $(a,b,a',b')$ hyperedges, the index or indices we average over will create free vertices, allowing us to bound the $(a,b,a',b')$ hyperedges without using any entrywise bounds. We can now use the same reasoning as before. The final case is if we only average over the $c$ indices.
\begin{lemma}
For all $a,b$ and all $\beta > 0$, if $m > 10000C(2+\beta)^2{n_{max}}r\mu^2\ln{n_{max}}$ and then $$P\left[\frac{1}{n_3}\sum_{c}{\left(W^{UV}_{abc}\right)^2} > \frac{32{e^2}B}{10000C\mu}\right] < n_{max}^{-(\beta+4)}$$
\end{lemma}
\begin{proof}
In this case, our preprocessing ensures that all distinct hyperedges appear with multiplicity which is even and at least two. Now instead of first bounding the $(a,b,a',b')$ hyperedges and then bounding the $(a',b',c')$ hyperedges, we will first bound the $(a',b',c')$ hyperedges using all of the distinct $c$ indices and bound the $(a,b,a',b')$ hyperedges using the distinct $a',b'$ indices.

Letting $y = z - (\text{\# of distinct } c')$, we have the following bounds on the number of $a',b',c$ indices and the number of times we will use an entrywise bound
\begin{enumerate}
\item There are at most $z$ $a'$ indices and there are at most $z$ $b'$ indices.
\item There are at most $2z-2x$ $a'$ and $b'$ indices, where $x$ is the number of times we use an entrywise bound on $(a,b,a',b')$ hyperedges because of an $(a',b')$ edge in a cycle.
\item There are $(z-y)$ $c$ indices.
\item The total number of times that we will use an entrywise bound is $2q - 2z + x + y$
\end{enumerate}
Taking $Z = \frac{1}{n_3}\sum_{c}{\left(W^{UV}_{abc}\right)^2}$, this gives us a bound of 
$$2q \cdot 2^{2q}\max_{z \in [1,2q], \atop x,y \in [0,z-1]}
{\left\{q^{2q-z}\left(\frac{{n_1}{n_2}{n_3}}{m}\right)^{2q-z}\left(\frac{Br\mu^4}{{n^2_1}{n^2_2}{n_3}}\right)^q\left(\frac{{n_1}{n_2}{n_3}}{\mu^3}\right)^{z}\left(\frac{\mu}{\min{\{n_1,n_2\}}}\right)^{2x}
\left(\frac{\mu}{n_3}\right)^{y}r^{2q-2z+x+y}\right\}}$$
on $E[Z^q]$. We check that it is optimal to take $z = q$, $x = 0$, and $y = 0$. Since $r\mu \leq \min{\{n_1,n_2,n_3\}}$, it is always optimal to take $x = y = 0$. Now if we reduce $z$ by $1$, this gives us a factor of at most $\frac{q{n_1}{n_2}{n_3}}{m}\cdot\frac{r^2\mu^3}{{n_1}{n_2}{n_3}} = \frac{qr^2\mu^3}{m}$. We will take $q \leq 10(1+\beta)\ln{n_{max}}$ and we have that $r\mu \leq \min{\{n_1,n_2,n_3\}}$ and $m > 10000C(2+\beta)^2{n_{max}}r\mu^2\ln{n_{max}}$, so it is indeed optimal to take $z = q$, $x = 0$, and $y = 0$. 

Comparing the resulting bound to our bound on $E\left[\left(W^{UV}_{abc}\right)^{2q}\right]$, it is smaller by a factor of $(r\mu)^q$, so our final bound is smaller by a factor of $r\mu$, as needed.
\end{proof}
We now have all of our needed probabilistic bounds. Theorem \ref{iterativebound} follows from the inequality 
$W_{abc}^2 \leq 5\left((W^{UV}_{abc})^2 + (W^{UW}_{abc})^2 + (W^{VW}_{abc})^2 + 2(W^{UVW}_{abc})^2\right)$ and union bounds.
\end{proof}

\section{Trace Power Calculation for $P'\bar{R}_{\Omega}A \otimes (P'\bar{R}_{\Omega}A)^T$}\label{squaretermboundwithPsection}
In this section, we prove the following theorem.
\begin{theorem}\label{noncrosstermtheoremwithP}
If $A$ is $B$-bounded, $C \geq 1$, and 
\begin{enumerate}
\item $m > 10000C(2+\beta)^2{n_{max}}r\mu^2\ln{n_{max}}$
\item $m > 10000C(2 + \beta)^2{r}\sqrt{n_1}\max{\{n_2,n_3\}}\mu^{\frac{3}{2}}\ln{n_{max}}
\geq 10000C(2 + \beta)^2{r}\sqrt{{n_1}{n_2}{n_3}}\mu^{\frac{3}{2}}\ln{n_{max}}$
\item ${\mu}r \leq \min{\{n_1,n_2,n_3\}}$
\end{enumerate}
then 
$$Pr\left[||Y|| > \frac{B{n_1}{n_2}{n_3}}{Cr\mu^3}\right] < 4n^{-(\beta+1)}$$
 whenever $Y$ is any of the following:
\begin{enumerate}
\item $Y = P^{UV}\bar{R}_{\Omega}A \otimes (P^{UV}\bar{R}_{\Omega}A)^T$
\item $Y = P^{UW}\bar{R}_{\Omega}A \otimes (P^{UW}\bar{R}_{\Omega}A)^T$
\item $Y = P^{VW}\bar{R}_{\Omega}A \otimes (P^{VW}\bar{R}_{\Omega}A)^T$
\item $Y = P^{UVW}\bar{R}_{\Omega}A \otimes (P^{UVW}\bar{R}_{\Omega}A)^T$
\end{enumerate}
\end{theorem}
\begin{proof}
This can be proved using the techniques of Sections \ref{normmethodssection} and \ref{squaretracepowercalculationsection} with one additional trick. We first consider the $P^{UVW}$ case and then describe the differences for the other cases. In all of these cases, we will show that the bound we obtain on $E\left[tr((Y{Y^T})^q)\right]$ is much less than the bound we obtained for $E\left[tr((Y_4{Y^T_4})^q)\right]$ in section \ref{squaretracepowercalculationsection}, which was 
$$2q \cdot 2^{4q}\max_{z \in [1,2q],x \in [0,z-2]}
{\left\{(2q)^{4q-z}\left(\frac{{n_1}{n_2}{n_3}}{m}\right)^{4q-z}B^{2q}\left(\frac{n_1}{\mu}\right)^{\frac{z}{2}}\left(\frac{\max{\{n_2,n_3\}}}{\mu}\right)^{z-2-x}\left(\frac{{n^2_2}{n^2_3}}{\mu^4}\right)r^{2q-z+x}\right\}}$$
Thus, for simplicity, for the remainder of the section, we will absorb constants, functions of only $q$, and logarithms into an $\tilde{O}$. Doing this and taking $z = 2q,x=0$, the above bound becomes $\left(\tilde{O}\left(\frac{{n_1}{n_2}{n_3}\sqrt{n_1}\max{\{n_2,n_3\}}B}{m\mu^{\frac{3}{2}}}\right)\right)^{2q}n^2_{max}$

When $Y = P^{UVW}\bar{R}_{\Omega}A \otimes (P^{UVW}\bar{R}_{\Omega}A)^T$, the structure of $tr\left((YY^T)^q\right)$ is as follows. We have $(a',b',c')$ hyperedges and we have hyperedges $(a,b,c,a',b',c')$ which we can view as an outer triangle $(a,b,c)$ and an inner triangle $(a',b',c')$. The outer triangles form hourglasses as before while the inner triangles sit inside the outer triangles.

The $\bar{R}_{\Omega}$ terms only involve the $(a',b',c')$ triples so our intersection patterns only describe these indices. Thus, we sum over all of the $a,b,c$ indices freely. We now use the following additional trick. We decompose each $UVW_{abca'b'c'}$ as 
$\sum_{i=1}^{r}{u_{ia}v_{ib}w_{ic}u_{ia'}v_{ib'}w_{ic'}}$. Now observe that every vertex in the outer triangles appears in two hyperedges. When we sum over that vertex, we get a term such as $\sum_{a}{u_{{i_1}a}u_{{i_2}a}}$. This is $0$ unless $i_1 = i_2$ and is $1$ if $i_1 = i_2$. This in fact forces a global choice for $i$ among the $UVW$ terms, giving a single factor of $r$ for the choices for this global $i$. This also means that the vertices in the outer triangles give a factor of exactly $1$, so they can be ignored! For the remaining terms of $UVW$, we use the bounds $u^2_{ia'} \leq \frac{\mu}{n_1}$, $v^2_{ib'} \leq \frac{\mu}{n_2}$, and $w^2_{ic'} \leq \frac{\mu}{n_3}$, obtaining a factor of $\left(\frac{\mu^3}{{n_1}{n_2}{n_3}}\right)^{2q}$ 

We now consider the contribution from summing over the $a',b',c'$ vertices, the contribution from the $\bar{R}_{\Omega}$ terms, and the contribution from the entries of $A$. Letting $z$ be the number of distinct triples $(a',b',c')$ in the given intersection pattern, the contribution from the $\bar{R}_{\Omega}$ terms will be $\left(\frac{{n_1}{n_2}{n_3}}{m}\right)^{4q-z}$. The contribution from the entries of $A$ from the $b(e)$ is $B^{2q}$. Letting $x$ be the number of times that we have to use an entrywise bound on a doubled edge because it is in a cycle, we have the following bounds on the number of indices and the number of times we use an entrywise bound.
\begin{enumerate}
\item The number of distinct $a$ indices, the number of distinct $b$ indices, and the number of distinct $c$ indices are all at most $z$
\item The total number of distinct indices is at most $3z-x$.
\item The number of times we use an entrywise bound is $2q-z+x$
\end{enumerate}
Putting everything together, we obtain a bound of 
$$\tilde{O}\left(\max_{z \in [1,2q], x \in [0,z-1]}{\left\{r\left(\frac{{\mu^3}B}{{n_1}{n_2}{n_3}}\right)^{2q}\left(\frac{{n_1}{n_2}{n_3}}{m}\right)^{4q-z}\left(\frac{{n_1}{n_2}{n_3}}{\mu^3}\right)^z\left(\frac{\mu}{\min{\{n_1,n_2,n_3\}}}\right)^x{r^{2q-z+x}}\right\}}\right)$$
for $E\left[tr((Y{Y^T})^q)\right]$.
This is maximized when $z = 2q$ and $x = 0$, leaving us with a bound of $\left(\tilde{O}(\frac{{n_1}{n_2}{n_3}B}{m})\right)^{2q}r$ which is much less than the bound of  
$\left(\tilde{O}\left(\frac{{n_1}{n_2}{n_3}\sqrt{n_1}\max{\{n_2,n_3\}}B}{m\mu^{\frac{3}{2}}}\right)\right)^{2q}n^2_{max}$ which we had for $E\left[tr((Y_4{Y^T_4})^q)\right]$, so we get a correspondingly smaller norm bound, as needed.

The analysis is the same for the $P^{UV}$, $P^{UW}$, and $P^{VW}$ cases except for the following differences which increase the bound on $E\left[tr((Y{Y^T})^q)\right]$, but still makes it much less than we had for $E\left[tr((Y_4{Y^T_4})^q)\right]$.
\begin{enumerate}
\item In the $P^{VW}$ case there are now two global indices, one for the top of the outer hourglasses and one for the bottom of the outer hourglasses. This gives us a global factor of $r^2$ rather than $r$.
\item Since one of the outer indices is now merged with the corresponding inner index, instead of the $UVW$ terms giving us factors of $\left(\frac{\mu}{n_1}\right)^{2q}$, $\left(\frac{\mu}{n_2}\right)^{2q}$, and $\left(\frac{\mu}{n_3}\right)^{2q}$ for the inner indices, we will only have two of these factors. This increases our bound on $E\left[tr((Y{Y^T})^q)\right]$ by a factor of at most $\left(\frac{n_{max}}{\mu}\right)^{2q}$ 
\end{enumerate}
Putting these differences together, our bound will now be $\left(\tilde{O}(\frac{{n_1}{n_2}{n_3}{n_{max}}B}{m\mu})\right)^{2q}r^2$ which is still much less than the bound we had for $E\left[tr((Y_4{Y^T_4})^q)\right]$.
\end{proof}

\addreferencesection

\bibliographystyle{amsalpha}

\bibliography{bib/mathreview,bib/dblp,bib/scholar,bib/custom}

\appendix

\section{Controlling the kernel of matrix representations}
\label{sec:zero-matching}

We prove the following lemma in this section which was an ingredient of the proof of \cref{thm:recovery}.
Let $\{u_i\},\{v_i\},\{w_i\}$ be three orthonormal bases of $\R^n$.

\begin{lemma*}[Restatement \cref{lem:zero-matching}]

  Let $R$ be self-adjoint linear operator $R$ on $\R^n\otimes \R^n$.
  Suppose $\iprod{(v_j\otimes w_k), R (v_{i}\otimes w_{i})}=0$ for all indices $i,j,k\in [r]$ such that $i\in \{j,k\}$.
  Then, there exists a self-adjoint linear operator $R'$ on $\R^n\otimes \R^n$ such that $R' (v_i\otimes w_i)=0$ for all $i\in [r]$, the spectral norm of $R'$ satisfies $\norm{R'}\le 10\norm{R}$, and $R'$ represents the same polynomial in $\R[y,z]$,
  \begin{displaymath}
    \iprod{(y\otimes z), R' (y\otimes z)} = \iprod{(y\otimes z), R (y\otimes z)}\,.
  \end{displaymath}
\end{lemma*}

\noindent
\emph{Proof.}
We write
$$R(v_i \otimes w_i) = \sum_{jk}{c_{iijk}v_j \otimes w_k}$$
Then the condition on the bilinear form of $R$ implies that
for all $i,j,k$, $c_{iiik} = 0$ and $c_{iiji} = 0$.

We now take $Z$ to be the following matrix
\begin{align*}
Z &= \sum_{i,j,k}{c_{iijk}\left((v_j \otimes w_k)(v_i \otimes w_i)^T - (v_j \otimes w_i)(v_i \otimes w_k)^T 
- (v_i \otimes w_k)(v_j \otimes w_i)^T + (v_i \otimes w_i)(v_j \otimes w_k)^T\right)} \\
&+ \sum_{ij}{\frac{c_{iijj}}{2}\left((v_j \otimes w_j)(v_i \otimes w_i)^T - (v_j \otimes w_i)(v_i \otimes w_j)^T 
- (v_i \otimes w_j)(v_j \otimes w_i)^T + (v_i \otimes w_i)(v_j \otimes w_j)^T\right)}
\end{align*}
It can be verified directly that $Z$ represents the $0$ polynomial and has the same behavior on each of the $(v_i \otimes w_i)$ as $R$. The factor of $\frac{1}{2}$ in the second sum comes from the fact that $c_{jjii} = c_{iijj}$ and the fourth term for $c_{jjii}$ matches the first term for $c_{iijj}$

We choose $R'=R-Z$. In order to show the bound $\norm{R'}\le 10\norm{R}$ it is enough to show that $\norm{Z}\le 9\norm{R}$

We analyze the norm of $Z$ as follows. We break $Z$ into parts according to each type of term and analyze each part separately. Define $X$ to be the subspace spanned by the $(v_i \otimes w_i)$, define $P_X$ to be the projection onto $X$ and define $P^{\perp}_X$ to be the projection onto the subspace orthogonal to $X$. 

For the part $\sum_{ijk}{c_{iijk}(v_j \otimes w_k)(v_i \otimes w_i)^T}$, note that $\sum_{ijk}{c_{iijk}(v_j \otimes w_k)(v_i \otimes w_i)^T} = {P^{\perp}_X}R'{P_X}$ so it has norm at most $||R||$.

For the part $\sum_{ijk}{c_{iijk}(v_j \otimes w_i)(v_i \otimes w_k)^T}$, note that under a change of basis this is equivalent to a block-diagonal matrix with blocks $\sum_{jk}{c_{iijk}{v_j}w^T_k}$.
The norm of each such block is at most its Frobenius norm, which is the norm of $\sum_{jk}c_{iijk}(v_j \otimes w_k) = R'(v_i \otimes w_i)$.
Thus, this part also has norm at most $||R||$.
Using similar arguments, we can bound the norm of the other parts by $||R||$ as well, obtaining that $||Z|| \leq 8||R||$.\qed

\section{Full Trace Power Calculation}
\label{sec:full-trace-power}
In this section, we analyze $||\sum_{a}{A_a \otimes B^T_a}||$ where 
$A = (\bar{R}_{\Omega_l}P_l)\cdots(\bar{R}_{\Omega_1}P_1)(\bar{R}_{\Omega_0}X)$ or $A = P_{l+1}(\bar{R}_{\Omega_l}P_l)\cdots(\bar{R}_{\Omega_1}P_1)(\bar{R}_{\Omega_0}X)$ for some projection operators $P_1,\cdots,P_l,P_{l+1}$ and $B = (\bar{R}_{\Omega_{l'}}P_{l'})\cdots(\bar{R}_{\Omega_1}P'_1)(\bar{R}_{\Omega_0}X)$ or $B = P_{l'+1}(\bar{R}_{\Omega_{l'}}P_{l'})\cdots(\bar{R}_{\Omega_1}P'_1)(\bar{R}_{\Omega_0}X)$ for some projection operators $P'_1,\cdots,P'_{l'},P'_{l'+1}$. In particular, we prove the following theorem using the trace power method.
\begin{theorem}\label{normboundtheorem}
There is an absolute constant $C$ such that for any $\alpha > 1$ and $\beta > 0$, 
$$Pr\left[||\sum_{a}{A_a \otimes B^T_a}|| > \alpha^{-(l+l'+2)}\right] < n^{-\beta}$$
as long as 
\begin{enumerate}
\item $r\mu \leq \min{\{n_1,n_2,n_3\}}$
\item $m > C\alpha\beta\mu^{\frac{3}{2}}r\sqrt{n_1}\max{\{n_2,n_3\}}log(\max{\{n_1,n_2,n_3\}})$
\item $m > C\alpha\beta\mu^{2}r\max{\{n_1,n_2,n_3\}}log(\max{\{n_1,n_2,n_3\}})$
\end{enumerate}
\end{theorem}
\begin{remark}
In this draft, we only sketch the case where we do not have projection operators in front. To handle the cases where there are projection operators in front, we can use the same ideas that are sketched out in Section \ref{squaretermboundwithPsection},
\end{remark}
\subsection{Term Structure}
When we expand out the sums in $tr\left(\left((\sum_{a}{A_a \otimes B^T_a})(\sum_{a}{A_a \otimes B^T_a})^T\right)^{q}\right)$, our terms will have the following structure. We label the indices so that each $\bar{R}_{\Omega_j}$ operator has its own indices $(a_{ij},b_{ij},c_{ij})$ or $(a'_{ij},b'_{ij},c'_{ij})$. Many of these indices will be equal.
\begin{enumerate}
\item For all $i \in [0,l]$ and all $j \in [1,2q]$ we have indices $(a_{ij},b_{ij},c_{ij})$ and a corresponding term $\bar{R}_{\Omega_i}(a_{ij},b_{ij},c_{ij})$ in the product.
\item For all $i \in [0,l']$ and all $j \in [1,2q]$ we have indices $(a'_{ij},b'_{ij},c'_{ij})$ and a corresponding term $\bar{R}_{\Omega_i}(a'_{ij},b'_{ij},c'_{ij})$ in the product.
\item For all $j \in [1,2q]$ we have a term $X_{{a_0}{b_0}{c_0}}$ and a term $X_{{a'_0}{b'_0}{c'_0}}$ in the product.
\item For all $i \in [0,l]$ and all $j \in [1,2q]$ we have a term $P_i(a_{ij},b_{ij},c_{ij},a_{(i-1)j},b_{(i-1)j},c_{(i-1)j})$ in the product.
\item For all $i \in [0,l']$ and all $j \in [1,2q]$ we have a term $P'_i(a'_{ij},b'_{ij},c'_{ij},a'_{(i-1)j},b'_{(i-1)j},c'_{(i-1)j})$ in the product.
\end{enumerate}
We represent the terms in the product graphically as follows.
\begin{definition} \ 
\begin{enumerate}
\item For all $i$, we represent the terms $\bar{R}_{\Omega_i}(a_{ij},b_{ij},c_{ij})$ and $\bar{R}_{\Omega_i}(a'_{ij},b'_{ij},c'_{ij})$ by triangles. We call these triangles $R_{i}$-triangles and $R'_{i}$-triangles respectively.
\item For all $i$ and $j$, 
\begin{enumerate}
\item If $P_i = P_{UV}$ then we represent $P_i(a_{ij},b_{ij},c_{ij},a_{(i-1)j},b_{(i-1)j},b_{(i-1)j})$ by a hyperedge $(a_{ij},b_{ij},a_{(i-1)j},b_{(i-1)j})$. We call this hyperedge a $UV$-hyperedge.
\item If 
$P_i = P_{UW}$ then we represent $P_i(a_{ij},b_{ij},c_{ij},a_{(i-1)j},b_{(i-1)j},b_{(i-1)j})$ by a hyperedge $(a_{ij},c_{ij},a_{(i-1)j},c_{(i-1)j})$. We call this hyperedge a $UW$-hyperedge.
\item If 
$P_i = P_{VW}$ then we represent $P_i(a_{ij},b_{ij},c_{ij},a_{(i-1)j},b_{(i-1)j},b_{(i-1)j})$ by a hyperedge $(b_{ij},c_{ij},b_{(i-1)j},c_{(i-1)j})$. We call this hyperedge a $VW$-hyperedge.
\item If $P_i = P_{UVW}$ then we represent $P_i(a_{ij},b_{ij},c_{ij},a_{(i-1)j},b_{(i-1)j},b_{(i-1)j})$ by a hyperedge $(a_{ij},b_{ij},c_{ij},a_{(i-1)j},b_{(i-1)j},c_{(i-1)j})$. We call this hyperedge a $UVW$-hyperedge.
\end{enumerate}
We represent the $P'_i$ terms by hyperedges in a similar manner.
\item For all $j \in [1,2q]$, we represent the term $X_{{a_{0j}}{b_{0j}}{c_{0j}}}$ with a hyperedge $({a_{0j}},{b_{0j}},{c_{0j}})$ and we represent the term $X_{{a'_{0j}}{b'_{0j}}{c'_{0j}}}$ with a hyperedge $({a'_{0j}},{b'_{0j}},{c'_{0j}})$. We call these hyperedges $X$-hyperedges.
\end{enumerate}
\end{definition}
We have the following equalities among the indices:
\begin{enumerate}
\item For all $j \in [1,2q]$, $a_{lj} = a'_{l'j}$
\item For all $j \in [1,2q]$, if $j$ is even then $b_{lj} = b_{l(j+1)}$ and $c'_{l'j} = c'_{l'(j+1)}$
\item For all $j \in [1,2q]$, if $j$ is odd then $c_{lj} = c_{l(j+1)}$ and $b'_{l'j} = b'_{l'(j+1)}$
\item For all $i \in [1,l]$ and all $j \in [1,2q]$, if $P_i = P_{UV}$ then $c_{ij} = c_{(i-1)j}$, if $P_i = P_{UW}$ then $b_{ij} = b_{(i-1)j}$, and if $P_i = P_{VW}$ then $a_{ij} = a_{(i-1)j}$
\item For all $i \in [1,l]$ and all $j \in [1,2q]$, if $P'_i = P_{UV}$ then $c'_{ij} = c'_{(i-1)j}$, if $P'_i = P_{UW}$ then $b'_{ij} = b'_{(i-1)j}$, and if $P'_i = P_{VW}$ then $a'_{ij} = a'_{(i-1)j}$
\end{enumerate}
\subsection{Techniques}
In this section, we describe how to bound the expected value of 
$$tr\left(\left((\sum_{a}{A_a \otimes B_a})(\sum_{a}{A_a \otimes B_a})^T\right)^{q}\right)$$
We first consider the $\bar{R}_{\Omega_i}$ terms, which for a given choice of the indices are as follows:
$$\left(\prod_{i=0}^{l}{\prod_{j=1}^{2q}{\bar{R}_{\Omega_i}(a_{ij},b_{ij},c_{ij})}}\right)
\left(\prod_{i=0}^{l'}{\prod_{j=1}^{2q}{\bar{R}_{\Omega_i}(a'_{ij},b'_{ij},c'_{ij})}}\right)$$
For a given choice of the indices, the expected value of this part can be bounded as follows
\begin{definition}
For all $i$, let $z_i$ be the number of distinct $R_{i}$-triangles and let $z'_i$ be the number of distinct $R'_{i}$-triangles. If a triangle appears as both an $R_{i}$-triangle and as an $R'_{i}$-triangle then it contributes $\frac{1}{2}$ to both $z_i$ and $z'_i$ (so the total number of distinct triangles at level $i$ is $z_i + z'_i$)
\end{definition}
\begin{lemma}
For a given choice of the indices $\{a_{ij},b_{ij},c_{ij}\}$ and $\{a'_{ij},b'_{ij},c'_{ij}\}$
\begin{enumerate}
\item If any triangle appears exactly once at some level $i$ then 
$$E\left[\left(\prod_{i=0}^{l}{\prod_{j=1}^{2q}{\bar{R}_{\Omega_i}(a_{ij},b_{ij},c_{ij})}}\right)
\left(\prod_{i=0}^{l'}{\prod_{j=1}^{2q}{\bar{R}_{\Omega_i}(a'_{ij},b'_{ij},c'_{ij})}}\right)\right]=0$$
\item If for all $i$, all of the triangles which appear at level $i$ appear at least twice then
\begin{align*}
0 &< E\left[\left(\prod_{i=0}^{l}{\prod_{j=1}^{2q}{\bar{R}_{\Omega_i}(a_{ij},b_{ij},c_{ij})}}\right)
\left(\prod_{i=0}^{l'}{\prod_{j=1}^{2q}{\bar{R}_{\Omega_i}(a'_{ij},b'_{ij},c'_{ij})}}\right)\right] \\
&\leq \left(\frac{{n_1}{n_2}{n_3}}{m}\right)^{\sum_{i=0}^{l}{(2q-z_i)}+\sum_{i=0}^{l'}{(2q-z'_i)}}
\end{align*}
\end{enumerate}
\end{lemma}
\begin{proof}
If there is any triangle $(a,b,c)$ which appears exactly once in level $i$ then $\bar{R}_{\Omega_i}(a_i,b_i,c_i)$ has expectation $0$ and is independent of every other term in the product so the entire product has value $0$. Otherwise, note that for $k > 1$, $0 < E\left[(R_{\Omega_i}(a,b,c))^k\right] \leq \left(\frac{{n_1}{n_2}{n_3}}{m}\right)^{k-1}$. Further note that $R_{\Omega_i}(a,b,c)$ terms with either different $i$ or different $a,b,c$ are independent of each other. Thus, using this bound, each copy of a triangle beyond the first gives us a factor of $\left(\frac{{n_1}{n_2}{n_3}}{m}\right)$. The total number of factors which we obtain is the total number of triangles minus the number of distinct triangles (where triangles at different levels are automatically distinct) and the result follows.
\end{proof}
We now note that this bound holds for all sets of indices that follow the same intersection pattern of which $R_{i}$-triangles and $R'_{i}$-triangles are equal to each other. Thus, we can group all terms which have the same intersection pattern together, using this bound on all of them. 

Each such intersection pattern forces additional equalities between the indices. After taking these equalities into account, we must sum over the remaining distinct indices. We now analyze what happens with the remaining terms of the product as we sum over these indices. We begin by considering how well we can bound the sum of entries of $X$ squared if we sum over $0$,$1$,$2$, or all $3$ indices.
\begin{lemma} \ 
\begin{enumerate}
\item $\max_{abc}{\{X^2_{abc}\}} \leq \frac{r^2\mu^3}{{n_1}{n_2}{n_3}}$
\item $\max_{bc}{\{\sum_{a}{X^2_{abc}}\}} \leq \frac{r\mu^2}{{n_2}{n_3}}$
\item $\max_{c}{\{\sum_{a,b}{X^2_{abc}}\}} \leq \frac{r\mu}{{n_3}}$
\item $\sum_{a,b,c}{X^2_{abc}} = r$
\end{enumerate}
\end{lemma}
\begin{proof}
For the first statement, 
$$X^2_{abc} = \sum_{i,j}{u_{ia}v_{ib}w_{ic}u_{ja}v_{jb}w_{jc}} \leq 
r^2\max_{i}\{u^2_{ia}v^2_{ib}w^2_{ic}\} \leq \frac{r^2\mu^3}{{n_1}{n_2}{n_3}}$$
For the second statement,
$$\sum_{a}{X^2_{abc}} = 
\sum_{i,j}{\left(\sum_{a}{u_{ia}u_{ja}}\right)v_{ib}w_{ic}v_{jb}w_{jc}} = 
\sum_{i,a}{u^2_{ia}v^2_{ib}w^2_{ic}} \leq \frac{\mu^2}{{n_2}{n_3}}\sum_{a,i}{u^2_{ia}} = \frac{r\mu^2}{{n_2}{n_3}}$$
For the third statement,
$$\sum_{a,b}{X^2_{abc}} = 
\sum_{i,j,b}{\left(\sum_{a}{u_{ia}u_{ja}}\right)v_{ib}w_{ic}v_{jb}w_{jc}} = 
\sum_{i,a,b}{u^2_{ia}v^2_{ib}w^2_{ic}} \leq \frac{\mu}{{n_3}}\sum_i{\left(\sum_{a}{u^2_{ia}}\sum_{b}{v^2_{ib}}\right)} = \frac{r\mu}{{n_3}}$$
The final statement can be proved in a similar way.
\end{proof}
Note that every index we sum over reduces the average value by $\mu$. Further note that if we do not sum over any indices, there is an extra factor of $r$ in our bound. Following similar logic, similar statements hold for the $P_i$ and $P'_i$ terms. 

We utilize this as follows. We start with a hypergraph $H$ which represents the current terms in our product. We first preprocess our product using the inequality $|ab| \leq \frac{x}{2}a^2 + \frac{b^2}{2x}$ (carefully choosing each $a$, $b$, and $x$) to make all of our hyperedges have even multiplicity. Note that when doing this, we cannot fully control which doubled hyperedges we will have; if we apply this on hyperedges $e_1$ and $e_2$ we could end up with two copies of $e_1$ or two copies of $e_2$.
 
Now if we have a hyperedge with multiplicity 4 or more, we use the entrywise bound to reduce its multiplicity by $2$. For example, if our sum was $\sum_{a}{X^4_{abc}}$ then we would use the inequality 
$$\sum_{a}{X^4_{abc}} \leq \max_{abc}{\{X^2_{abc}\}}\sum_{a}{X^2_{abc}}$$ 
to bound this sum. 

Once every hyperedge appears with power 2, we choose an ordering for how we will bound the hyperedges. For each hyperedge, we sum over all indices which are currently only incident with that hyperedge, take the appropriate bound, and then delete the hyperedge and these indices from our current hypergraph $H$. We account for all of this with the following definitions:
\begin{definition} \ 
\begin{enumerate}
\item Define the base value of an $X$-hyperedge $e$ to be $v(e) = \sqrt{\frac{r\mu^3}{{n_1}{n_2}{n_3}}}$
\item Define the base value of a $UV$-hyperedge $e$ to be $v(e) = \sqrt{\frac{r\mu^4}{{n^2_1}{n^2_2}}}$
\item Define the base value of a $UW$-hyperedge $e$ to be $v(e) = \sqrt{\frac{r\mu^4}{{n^2_1}{n^2_3}}}$
\item Define the base value of a $VW$-hyperedge $e$ to be $v(e) = \sqrt{\frac{r\mu^4}{{n^2_2}{n^2_3}}}$
\item Define the base value of an $UVW$-hyperedge $e$ to be $v(e) = \sqrt{\frac{r\mu^6}{{n^2_1}{n^2_2}{n^2_3}}}$
\end{enumerate}
\end{definition}
\begin{definition}
We say that a index in our hypergraph $H$ is free if it is incident with at most one hyperedge.
\end{definition}
For a given intersection pattern, assuming that every vertex is incident with at least one hyperedge after the preprocessing, our final bound will be
\begin{align*}&\left(\frac{{n_1}{n_2}{n_3}}{m}\right)^{\sum_{i=0}^{l}{(2q-z_i)}+\sum_{i=0}^{l'}{(2q-z'_i)}}
\left(\prod_{e}{v(e)}\right)\left(\frac{n_1}{\mu}\right)^{\text{\# of a indices}} \\
&\left(\frac{n_2}{\mu}\right)^{\text{\# of b indices}}
\left(\frac{n_3}{\mu}\right)^{\text{\# of c indices}}r^{\text{\# of doubled hyperedges we bound with no free index}}
\end{align*}
To see this, note that from the discussion above, when an index $a$,$b$, or $c$ is free and we sum over it, we obtain $n_1$, $n_2$, or $n_3$ terms respectively but this also reduces the current bound we are using by a factor of $\mu$. This will happen precisely one time for every index which is incident to at least one edge. Thus, for this part the ordering doesn't really matter. However, there is an extra factor of $r$ whenever we bound a doubled hyperedge with no free index (including when this hyperedge has multiplicity 4 or higher and we reduce its multiplicity by 2). We want to avoid this extra factor of $r$ as much as possible. We describe how to do this in subsection \ref{choosingordering}.
\begin{remark}
When summing over an index, we may not acutally sum over all possiblities because this could create equalities between triangles which should not be equal according to the intersection pattern. However, adding in these missing terms can only increase the sum, so it is still an upper bound.
\end{remark} 
\subsection{Bounding the number of indices}
In this subsection, we describe bounds on the number of each type of index for a given intersection pattern. We then define a coefficient $\Delta$ which is the discrepency between the our bounds and the actual number of indices and reexpress our boun in terms of $\Delta$.
 
We make the following simplifying assumption about our sums.
\begin{enumerate}
\item For all $i \in [0,l']$, we either have that $a'_{ij} = a_{ij}$ for all $j \in [1,2q]$ or $a'_{ij} \neq a_{ij}$ for all $j \in [1,2q]$.
\item For all $i \in [0,l']$, we either have that $b'_{ij} = b_{ij}$ for all $j \in [1,2q]$ or $b'_{ij} \neq b_{ij}$ for all $j \in [1,2q]$.
\item For all $i \in [0,l']$, we either have that $c'_{ij} = c_{ij}$ for all $j \in [1,2q]$ or $c'_{ij} \neq c_{ij}$ for all $j \in [1,2q]$.
\end{enumerate}
Moreover, all of these choices are fixed beforehand. We justify this assumption with a random partitioning argument in subsection \ref{randompartitionsection}.

With this setup, we first bound the number of each type of index which appears.
\begin{definition} 
For all $i$,
\begin{enumerate}
\item We define $x_{ia}$ to be the number of distinct indices $a_{ij}$ which do not appear at a higher level, we define $x_{ib}$ to be the number of distinct indices $b_{ij}$ which do not appear at a higher level, and we define $x_{ic}$ to be the number of distinct indices $c_{ij}$ which do not appear at a higher level.
\item We define $x'_{ia}$ to be the number of distinct indices $a'_{ij}$ which do not appear at a higher level, we define $x'_{ib}$ to be the number of distinct indices $b'_{ij}$ which do not appear at a higher level, and we define $x'_{ic}$ to be the number of distinct indices $c'_{ij}$ which do not appear at a higher level.
\end{enumerate}
In the case where we have an equality $a'_{ij} = a_{ij}$ and this index does not appear at a higher level, we instead count it as $\frac{1}{2}$ for $x_{ia}$ and $\frac{1}{2}$ for $x'_{ia}$ (and similarly for $b$ and $c$).
\end{definition}
Recall that we defined $z_{i}$ to be the number of distinct $R_i$-triangles and we defined $z'_{i}$ to be the nmber of distinct $R'_{i}$-triangles. $z_{i}$ and $z'_i$ give the following bounds on the coefficients
\begin{proposition} \ 
\begin{enumerate}
\item For all $i < l$, $x_{ia} \leq z_i$, $x_{ib} \leq z_i$, and $x_{ic} \leq z_i$.
\item For all $i < l'$, $x'_{ia} \leq z'_i$, $x'_{ib} \leq z'_i$, and $x'_{ic} \leq z'_i$.
\item If $l' \neq l$, $b'_{lj} \neq b_{lj}$, or $c'_{lj} \neq c_{lj}$, 
\begin{enumerate}
\item $x_{la} + x'_{l'a} \leq \min{\{z_l,z'_{l'}\}}$
\item $x_{lb} + x_{lc} \leq z_l + 1$
\item $x'_{l'b} + x'_{l'c} \leq z'_{l'} + 1$
\end{enumerate}
\item In the special case that $l' = l$, $b'_{lj} = b_{lj}$, and $c'_{lj} = c_{lj}$,
\begin{enumerate}
\item $x_{la} + x'_{l'a} = z_l + z'_{l'}$
\item $x_{lb} + x_{lc} = x'_{l'b} + x'_{l'c} = 1$
\end{enumerate}
\end{enumerate}
\end{proposition}
\begin{proof}
The first two statements and 3(a) follow from the observation that distinct vertices must be in distinct triangles. For 3(b), note that if we take the $b,c$ edges from each $R_i$-triangle, the resulting graph is connected. Thus, each distinct such edge (which must come from a distinct triangle) after the first edge can only add one new vertex and the result follows. 3(c) can be proved analogously.

For the fourth statement, note that in this case all of the $b_{lj}$ and $b'_{lj}$ indices are equal to a single index $b$ and all of the $c_{lj}$ and $c'_{lj}$ indices are equal to a single index $c$. Thus, the number of distinct $a_{lj}$ is equal to the number of distinct $R_{l}$ and $R'_{l}$ triangles.
\end{proof}
With these bounds in mind, we define $x^{max}$ coefficients which represent the maximum number of distinct indices we can expect (given the structure of $A$ and $B$ and the values $z_i,z'_i$) and $\Delta$ coefficients which describe the discrepency between this maximum and the number of distinct indices which we actually have.
\begin{definition} \ 
\begin{enumerate}
\item For all $i < l$, 
\begin{enumerate}
\item If $P_{i+1} = P_{UV}$ then we define $x^{max}_{ia} = x^{max}_{ib} = z_i$. We define $\Delta_{ia} = x^{max}_{ia} - x_{ia}$, $\Delta_{ib} = x^{max}_{ib} - x_{ib}$, and $\Delta_{ic} = 0$.
\item If $P_{i+1} = P_{UW}$ then we define $x^{max}_{ia} = x^{max}_{ic} = z_i$. We define $\Delta_{ia} = x^{max}_{ia} - x_{ia}$, $\Delta_{ic} = x^{max}_{ic} - x_{ic}$, and $\Delta_{ib} = 0$.
\item If $P_{i+1} = P_{VW}$ then we define $x^{max}_{ia} = x^{max}_{ib} = z_i$. We define $\Delta_{ib} = x^{max}_{ib} - x_{ib}$, $\Delta_{ic} = x^{max}_{ic} - x_{ic}$, and $\Delta_{ia} = 0$.
\item If $P_{i+1} = P_{UVW}$ then we define $x^{max}_{ia} = x^{max}_{ib} = x^{max}_{ic} = z_i$. We define $\Delta_{ia} = x^{max}_{ia} - x_{ia}$, $\Delta_{ib} = x^{max}_{ib} - x_{ib}$, and $\Delta_{ic} = x^{max}_{ic} - x_{ic}$.
\end{enumerate}
\item For all $i < l'$, 
\begin{enumerate}
\item If $P'_{i+1} = P_{UV}$ then we define ${x'}^{max}_{ia} = {x'}^{max}_{ib} = z'_i$. We define $\Delta'_{ia} = {x'}^{max}_{ia} - {x'}_{ia}$, $\Delta'_{ib} = {x'}^{max}_{ib} - {x'}_{ib}$, and $\Delta'_{ic} = 0$.
\item If $P'_{i+1} = P_{UW}$ then we define ${x'}^{max}_{ia} = {x'}^{max}_{ic} = z'_i$. We define $\Delta'_{ia} = {x'}^{max}_{ia} - {x'}_{ia}$, $\Delta'_{ic} = {x'}^{max}_{ic} - {x'}_{ic}$, and $\Delta'_{ib} = 0$.
\item If $P'_{i+1} = P_{VW}$ then we define ${x'}^{max}_{ia} = {x'}^{max}_{ib} = z'_i$. We define $\Delta'_{ib} = {x'}^{max}_{ib} - {x'}_{ib}$, $\Delta'_{ic} = {x'}^{max}_{ic} - {x'}_{ic}$, and $\Delta'_{ia} = 0$.
\item If $P'_{i+1} = P_{UVW}$ then we define ${x'}^{max}_{ia} = {x'}^{max}_{ib} = {x'}^{max}_{ic} = z'_i$. We define $\Delta_{ia} = {x'}^{max}_{ia} - {x'}_{ia}$, $\Delta'_{ib} = {x'}^{max}_{ib} - {x'}_{ib}$, and $\Delta'_{ic} = {x'}^{max}_{ic} - {x'}_{ic}$.
\end{enumerate}
\item If $l' \neq l$, $b'_{l'j} \neq b_{lj}$, or $c'_{l'j} \neq c_{lj}$ then we define $x^{max}_{ll'a} = \min{\{z_l,z'_{l'}\}}$, we define $x^{max}_{lbc} = z_l+1$, and we define ${x'}^{max}_{lbc} = z'_{l'} + 1$. 
In the special case where $l' = l$, $b'_{l'j} = b_{lj}$, and $c'_{l'j} = c_{lj}$, we define 
$x^{max}_{ll'a} = z_l + z'_l$ and $x^{max}_{lbc} = {x'}^{max}_{l'bc} = 1$.
In both of these cases, we define $\Delta_{ll'a} = x^{max}_{ll'a} - x_{la} - x'_{l'a}$, $\Delta_{lbc} = x^{max}_{lbc} - x_{lb} - x_{lc}$, and $\Delta'_{l'bc} = {x'}^{max}_{l'bc} - x'_{l'b} - x'_{l'c}$.
\end{enumerate}
\end{definition}
\begin{definition}
We define 
$$\Delta = \Delta_{ll'a} + \Delta_{lbc} + \Delta'_{l'bc} + 
\sum_{i=0}^{l-1}{(\Delta_{ia} + \Delta_{ib} + \Delta_{ic})} + 
\sum_{i=0}^{l'-1}{(\Delta'_{ia} + \Delta'_{ib} + \Delta'_{ic})}$$
\end{definition}
We now reexpress our bound in terms of $\Delta$.
\begin{lemma}
For a given intersection pattern and choices for the equalities or inequalities between the $a_{ij},b_{ij},c_{ij}$ and $a'_{ij},b'_{ij},c'_{ij}$ indices, we can obtain a bound which is a product of 
$$\frac{{n_2}{n_3}}{\mu^2}\left(\frac{\mu}{\min{\{n_1,n_2,n_3\}}}\right)^{\Delta}r^{\text{\# of doubled hyperedges we bound with no free index} - \left(\sum_{i=0}^{l}{(2q-z_i)}+\sum_{i=0}^{l'}{(2q-z'_i)}\right)}$$
and terms of the form $\frac{r\mu^{\frac{3}{2}}\sqrt{n_1}\max{\{n_2,n_3\}}}{m}$, 
$\frac{r\mu^2\max{\{n_1,n_2,n_3\}}}{m}$, or $\frac{r\mu^3}{m}$
\end{lemma}
\begin{proof}
Recall that our bound was 
\begin{align*}&\left(\frac{{n_1}{n_2}{n_3}}{m}\right)^{\sum_{i=0}^{l}{(2q-z_i)}+\sum_{i=0}^{l'}{(2q-z'_i)}}
\left(\prod_{e}{v(e)}\right)\left(\frac{n_1}{\mu}\right)^{\text{\# of a indices}} \\
&\left(\frac{n_2}{\mu}\right)^{\text{\# of b indices}}
\left(\frac{n_3}{\mu}\right)^{\text{\# of c indices}}r^{\text{\# of doubled hyperedges we bound with no free index}}
\end{align*}
For all $i < l$, we consider the part of this bound which comes from $P_{i+1}$ and the indices $a_i,b_i,c_i$ which do not appear at a higher level. Similary, for all $i < l'$, we consider the part of this bound which comes from $P'_{i+1}$ and the indices $a'_i,b'_i,c'_i$ which do not appear at a higher level. Finally, we consider the part of this bound that comes from the $X$ hyperedges, the $R_{l}$-triangles, the $R'_{l'}$-triangles, and their indices.
\begin{enumerate}
\item If $P_{i+1} = P_{UV}$ then we can decompose the corresponding terms into the following parts:
\begin{enumerate}
\item $(\frac{r\mu^4}{{n^2_1}{n^2_2}})^q$ from the hyperedges.
\item $\left(\frac{{n_1}{n_2}}{\mu^2}\right)^q$ from the $q$ potential new $a$, $b$, and $c$ indices.
\item $\left(\frac{{n_1}{n_2}{n_3}}{m}\right)^q$ from the $q$ potential distinct triangles.
\item $\left(r \cdot \frac{\mu^2}{{n_1}{n_2}} \cdot \frac{{n_1}{n_2}{n_3}}{m}\right)^{q-z_i} = \left(\frac{r{\mu^2}n_3}{m}\right)^{q-z_i}$ from the actual number of distinct triangles, the corresponding reduced maximum number of potential new indices, and the factors of $r$ which we take from $r^{\text{\# of doubled hyperedges we bound with no free index}}$
\item $\left(\frac{\mu}{n_1}\right)^{\Delta_{ia}}\left(\frac{\mu}{n_2}\right)^{\Delta_{ib}}
\left(\frac{\mu}{n_3}\right)^{\Delta_{ic}} \leq 
\left(\frac{\mu}{\min{\{n_1,n_2,n_3\}}}\right)^{\Delta_{ia}+\Delta_{ib}+\Delta_{ic}}$ from the actual number of new indices which we have
\end{enumerate}
Putting everything together we obtain 
$$\left(\frac{r{\mu^2}n_3}{m}\right)^{2q-z_i}
\left(\frac{\mu}{\min{\{n_1,n_2,n_3\}}}\right)^{\Delta_{ia}+\Delta_{ib}+\Delta_{ic}}$$
Similar arguments apply if $P_{i+1} = P_{VW}$ or $P_{VW}$
\item If $P_{i+1} = P_{UVW}$ then we can decompose the corresponding terms into the following parts:
\begin{enumerate}
\item $(\frac{r\mu^6}{{n^2_1}{n^2_2}{n^2_3}})^q$ from the hyperedges.
\item $\left(\frac{{n_1}{n_2}{n_3}}{\mu^2}\right)^q$ from the $q$ potential new $a$, $b$, and $c$ indices.
\item $\left(\frac{{n_1}{n_2}{n_3}}{m}\right)^q$ from the $q$ potential distinct triangles.
\item $\left(r \cdot \frac{\mu^3}{{n_1}{n_2}{n_3}} \cdot \frac{{n_1}{n_2}{n_3}}{m}\right)^{q-z_i} = \left(\frac{r{\mu^3}}{m}\right)^{q-z_i}$ from the actual number of distinct triangles, the corresponding reduced maximum number of potential new indices, and the factors of $r$ which we take from $r^{\text{\# of doubled hyperedges we bound with no free index}}$
\item $\left(\frac{\mu}{n_1}\right)^{\Delta_{ia}}\left(\frac{\mu}{n_2}\right)^{\Delta_{ib}}
\left(\frac{\mu}{n_3}\right)^{\Delta_{ic}} \leq 
\left(\frac{\mu}{\min{\{n_1,n_2,n_3\}}}\right)^{\Delta_{ia}+\Delta_{ib}+\Delta_{ic}}$ from the actual number of new indices which we have
\end{enumerate}
Putting everything together we obtain 
$$\left(\frac{r{\mu^3}}{m}\right)^{2q-z_i}
\left(\frac{\mu}{\min{\{n_1,n_2,n_3\}}}\right)^{\Delta_{ia}+\Delta_{ib}+\Delta_{ic}}$$
Similar arguements holds for the $P'$ terms.
\item If $l' \neq l$, $b'_{lj} \neq b_{lj}$, or $c'_{lj} \neq c_{lj}$ then our remaining terms are as follows
\begin{enumerate}
\item $(\frac{r\mu^3}{{n_1}{n_2}{n_3}})^{2q}$ from the hyperedges.
\item $\frac{{n_2}{n_3}}{\mu^2}\left(\frac{{n_1}(\max{\{n_2,n_3\}})^2}{\mu^3}\right)^q$ from the $q$ potential $a$ indices and $2q+2$ potential $b$ or $c$ indices (which must have at least one $b$ index and at least one $c$ index).
\item $\left(\frac{{n_1}{n_2}{n_3}}{m}\right)^{2q}$ from the $2q$ potential distinct triangles.
\item $\left(r \cdot \frac{\mu^{\frac{3}{2}}}{\sqrt{n_1}\max{\{n_2,n_3\}}} \cdot \frac{{n_1}{n_2}{n_3}}{m}\right)^{2q-z_l-z'_l} \leq \left(\frac{r{\mu^{3/2}}\sqrt{n_1}\max{\{n_2,n_3\}}}{m}\right)^{2q-z_l-z'_l}$ from the actual number of distinct triangles, the corresponding reduced maximum number of potential new indices, and the factors of $r$ which we take from $r^{\text{\# of doubled hyperedges we bound with no free index}}$
\item $\left(\frac{\mu}{n_1}\right)^{\Delta_{ll'a}}\left(\frac{\mu}{\max{\{n_2,n_3\}}}\right)^{\Delta_{lbc} + \Delta'_{l'bc}} \leq 
\left(\frac{\mu}{\min{\{n_1,n_2,n_3\}}}\right)^{\Delta_{ll'a}+\Delta_{lbc}+\Delta'_{l'bc}}$ from the actual number of new indices which we have
\end{enumerate}
Putting everything together we obtain 
$$\frac{{n_2}{n_3}}{\mu^2}\left(\frac{r{\mu^{3/2}}\sqrt{n_1}\max{\{n_2,n_3\}}}{m}\right)^{4q-z_l-z'_l}
\left(\frac{\mu}{\min{\{n_1,n_2,n_3\}}}\right)^{\Delta_{ll'a}+\Delta_{lbc}+\Delta'_{l'bc}}$$
\item In the special case that $l' = l$, $b'_{lj} = b_{lj}$, and $c'_{lj} = c_{lj}$, we have the same terms except that now there is only one $b$ and $c$ index and there are $2q$ potential $a$ indices. Following similar logic we obtain a bound of 
$$\frac{{n_2}{n_3}}{\mu^2}\left(\frac{r\mu^{2}{n_1}}{m}\right)^{4q-z_l-z'_l}
\left(\frac{\mu}{\min{\{n_1,n_2,n_3\}}}\right)^{\Delta_{ll'a}+\Delta_{lbc}+\Delta'_{l'bc}}$$
\end{enumerate}
\end{proof}
With this lemma in hand, to show our bound it is sufficient to show that we can choose an ordering on the hyperedges such that the number of times we bound a doubled hyperedge without a free index is at most $\Delta + \sum_{i=0}^{l}{(2q-z_i)}+\sum_{i=0}^{l'}{(2q-z'_i)}$
\subsection{Choosing an ordering}\label{choosingordering}
In this section, we describe how to choose a good ordering for bounding the hyperedges.
\begin{lemma}
For any structure for $A$ and $B$ (including equalities or inequalities between $a_{ij},b_{ij},c_{ij}$ and $a'_{ij},b'_{ij},c'_{ij}$) and any intersection pattern, there is a way to double the hyperedges using the inequality $|ab| \leq \frac{x}{2}a^2 + \frac{1}{2x}b^2$ and then bound the doubled hyperedges one by one so that 
\begin{enumerate}
\item After doubling the hyperedges, every index is part of at least one hyperedge.
\item The number of times that we bound a doubled hyperedge without a free index is at most $\Delta + \sum_{i=0}^{l}{(2q-z_i)}+\sum_{i=0}^{l'}{(2q-z'_i)}$
\end{enumerate}
\end{lemma}
\begin{proof}
To double the $X$-hyperedges, we choose pairs of $X$-hyperedges corresponding to the same triangle. This guarantees us at least one doubled hyperedge for every triangle at level $0$. We double any remaining $X$-hyperedges arbitrarily.

We show by induction on $i$ that we cover all indices with these hyperedges. The base case $i=0$ is already done. If we have already covered all indices at level $i-1$ then consider the hyperedges corresponding to the projection operators $P_i$ and $P'_i$. All of these hyperedges go between a triangle at level $i-1$ and a triangle at level $i$. We double pairs of these hyperedges which correspond to the same triangle at level $i$. This guarantees that for every triangle at level $i$, there is at least one doubled hyperedge corresponding to it. This hyperedge may not cover all three of the vertices of the triangle, but if it misses one, this one must be equal to a vertex at the level below which was already covered by assumption. We double the remaining hyperedges corresponding to the projection operators $P_i$ and $P'_i$ arbitrarily.

When performing this doubling, whenever the two hyperedges $e_1$ and $e_2$ have the same base value, we use the inequality $|{e_1}{e_2}| \leq \frac{e^2_1 + e^2_2}{2}$. In the rare case when they have different base values, we use the inequality $|{e_1}{e_2}| \leq \frac{v(e_2)}{2v(e_1)}e^2_1 + \frac{v(e_1)}{2v(e_2)}e^2_2$ to preserve the product of the base values.
 
Note that by this construction, for every triangle at level $i \geq 1$, there is a doubled hyperedge corresponding to some $P_i$ or $P'_i$ which goes between this triangle and a lower triangle, but we don't know which one.

We now describe our ordering on the hyperedges. To find this ordering, we consider the following multi-graph.
\begin{definition}
We define the multi-graph $G$ to have vertex set $V(G) = \cup_{ij}{\{a_{ij},b_{ij},c_{ij},a'_{ij},b'_{ij},c'_{ij}\}}$ (with all equalities implied by the intersection pattern, the structure of the matrices $A$ and $B$, and the choices for equalities or inequalities between the primed indices and unprimed indices.). We take the edges of $G$ as follows. For all $i < l$ and for each distinct triangle $(a_{ij},b_{ij},c_{ij})$ or $(a'_{ij},b'_{ij},c'_{ij})$, we take the elements which do not appear in a higher level. If this is true for two of the three elements (which will be the case most of the time) we take the corresponding edge. If this is true for all three elements, we choose two of them to take as an edge, making this choice so that we take the same type of edge for all triangles at that level. If this is only true for one element, we take a loop on that element.

There are two cases for what happens with $i = l$
\begin{enumerate}
\item If $l' \neq l$, $b'_{lj} \neq b_{lj}$, or $c'_{lj} \neq c_{lj}$ then for every triangle $(a_{lj},b_{lj},c_{lj})$ we take the edge $(b_{lj},c_{lj})$. If $l' = l$ then for every triangle $(a'_{lj},b'_{lj},c'_{lj})$ we take the edge $(b'_{lj},c'_{lj})$
\item If $l' = l$, $b'_{lj} = b_{lj}$, and $c'_{lj} = c_{lj}$ then we take loops on every distinct element $a_{lj}$.
\end{enumerate}
\end{definition}
We analyze $\Delta$ in terms of this $G$. If we have a fixed budget of edges and want to maximize the number of vertices which we have, we want to have as many connected components as possible and we want each connected component to have the minimal number of edges. We define weights on the connected components of $G$ measuring how far they are from satisfying these ideals.
\begin{definition}
Given a connected component $C$ of $G$, we define $w_{edge}(C)$ to be the number of non-loop edges it contains plus $1$ minus the number of vertices it contains.
\end{definition}
\begin{definition}
Given a connected component $C$ of $G$, we define $w_{triangle}(C)$ as follows
\begin{enumerate}
\item If $C$ does not contain any $b_{lj}$, $c_{lj}$, $b'_{l'j}$, or $c'_{l'j}$ then we define $w_{triangle}(C)$ to be the number of distinct triangles whose corresponding edge in $G$ is in $C$ minus $1$. 
\item If $C$ is the connected component containing $b_{lj}$ and $c_{lj}$ for all $j$ then we set $w_{triangle}(C) = 0$
\item If $C$ is the connected component containing $b'_{l'j}$ and $c'_{l'j}$ for all $j$ then we set $w_{triangle}(C)$ to be the number of distinct $R_{l'}$ triangles $(a_{l'j},b_{l'j},c_{l'j})$ whose corresponding edge in $G$ is in $C$.
\item If $C$ is a connected component containing some $c'_{l'j}$ but no $b'_{l'j}$ (because all of the $b'_{l'j}$ appeared at a higher level) or vice versa, then we define $w_{triangle}(C)$ to be the number of distinct triangles whose corresponding edge in $G$ is in $C$ minus $1$.
\end{enumerate}
\end{definition}
\begin{definition}
We say that a vertex $a_{ij}$ is bad if 
\begin{enumerate}
\item The projector $P_{i+1}$ involves the vertex $a_{ij}$ (i.e. we do not have the constraint $a_{(i+1)j} = a_{ij}$ directly)
\item $a_{ij}$ appears at a higher level.
\end{enumerate}
We define badness similarly for the $a',b,b',c,c'$ indices. Note that we could have $a_i$ be bad while $a'_i$ is not bad even if $a'_i = a_i$ (in fact this equality must be true in this case).
\end{definition}
\begin{lemma}
$$\Delta \geq \sum_{C}{(w_{edge}(C)+w_{triangle}(C))}+\sum_{i<l:a_{ij},b_{ij},\text{ or } c_{ij} \text{ is bad}}{z_i}
+\sum_{i<l:a'_{ij},b'_{ij},\text{ or } c'_{ij} \text{ is bad}}{z'_i}$$
\end{lemma}
\begin{proof}
As discussed above, every time a connected component contains an extra edge above what it needs to be connected, this reduces the number of indices we can have by 1. Similarly, in the optimal case we have one connected component per triangle (with the exception of the $R_{l}$-triangles and perhaps the $R'_{l'}$-triangles), so every time a connected component contains an extra triangle (or rather the edge corresponding to that triangle), this reduces the number of connected components by 1. For the remaining terms, note that if there are bad vertices, our previous bounds assumed that we would have new indices of that type but we do not. The resulting difference in the bounds is the corresponding $z_i$ or $z'_i$. Note that this also works out in the special case that $a'_i = a_i$, $b'_i = b_i$, $c'_i = c_i$. Here we can view each $a$ index as being half $a_i$ and half $a'_i$ and similarly for the $b$ and $c$ indices.
\end{proof}
With this lemma in hand, our strategy is as follows. We choose an ordering on the hyperedges so that each time we fail to have a free index, we can attribute it to one of the terms described above. We first preprocess our doubled hyperedges so that each hyperedge appears with multiplicity exactly 2. This requires bounding $\sum_{i=0}^{l}{(2q-z_i)}+\sum_{i=0}^{l'}{(2q-z'_i)}$ doubled hyperedges with no free index. At this point, there is a one to one correspondence between our doubled hyperedges and edges of $G$. Note that this correspondence is somewhat strange, we only know that each edge in $G$ is part of the upper level triangle for its corresponding hyperedge.

We now describe our procedures for ordering the hyperedges
\begin{definition}
We say that a vertex $v$ is an anchor for an edge $e$ of $G$ if either
\begin{enumerate}
\item $v,e \subseteq \{a_{ij},b_{ij},c_{ij}\}$ for some $i$ and $j$ and 
$v$ appears at a higher level.
\item $v,e \subseteq \{a'_{ij},b'_{ij},c'_{ij}\}$ for some $i$ and $j$ and 
$v$ appears at a higher level.
\end{enumerate}
\end{definition}
\begin{definition}
For an anchor vertex $v_{anchor}$, define $E_i(v_{anchor})$ to be the set of all edges at level $i$ which have $v_{anchor}$ as an anchor vertex.
\end{definition}
\begin{definition} 
We say that a vertex $v$ or edge $e$ is uncovered if it is not incident with any hyperedges between its level and the level above and covered otherwise. For a vertex $v$ which is not part of $G$ at level $i$, we say that $v$ is uncovered at level $i$ if there is no $j \geq 0$ such that $v$ incident with a hyperedge between level $i+j$ and $i+j+1$.
\end{definition}
\begin{definition}
We say that a vertex $v$ is released at level $i$ if there are no hyperedges remaining between level $i$ and $i-1$ whose upper and lower triangles both contain $v$.
\end{definition}
Our main recursive procedure is as follows. We are considering a collection of connected component of the graph at level $i$ where everything  is uncovered except possibly for one edge $e_r$. If an edge $e_r$ is covered and has anchor vertex $v_{anchor}$ then we assume that this collection contains all of $E_{i}(v_{anchor})$ and that $v_{anchor}$ is uncovered at level $i$.

We first consider the case when there are no bad vertices (we will consider the cases where we have bad vertices afterwards). If $G$ contains a cycle, we can delete an edge and its corresponding hyperedge to break the cycle, accounting for this by decreasing $w_{edge}(C)$. Otherwise, unless $C$ is just the single edge $e_r$, there must be a vertex $v$ and edge $e$ in $C$ such that $e \neq e_r$ and $e$ is the only edge incident with $v$.

We now consider the hyperedge corresponding to $e$. If $v$ is part of this hyperedge then we can delete $e$ and this hyperedge and continue. Otherwise, $v$ must be an anchor vertex for many edges at the level below. Moreover, $v$ is uncovered at level $i-1$. We now consider $E_{i-1}(v)$. If $E_{i-1}(v)$ and everything connected to it is uncovered except for the edge $e'_r$ which is the bottom edge of the hyperedge corresponding to $e$, then we can apply our procedure recursively on $E_{i-1}(v)$ and everything connected to it. Otherwise, $E_{i-1}(v)$ must be connected to $E_{i-1}(v'_{anchor})$ for some other anchor vertex $v'_{anchor}$ which has not yet been released at level $i$. Note that since there are no bad vertices, $E_{i-1}(v) \cap E_{i-1}(v'_{anchor}) = \emptyset$. Thus, there is a contribution of at least $1$ to $w_{triangle}$ of one of these connected components from the connection between $E_{i-1}(v)$ and $E_{i-1}(v'_{anchor})$. Using this contribution, we can delete $e$ and continue. After doing this, $v$ is released at level $i$.
\begin{remark}
Whenever we have a connection between $E_{i-1}$ for two anchor vertices, we relase one of them at level $i$ immediately after taking this connection into account. This ensures that we do not double count contributions to $w_{triangle}$.
\end{remark} 

If we are left with the single edge $e_r$ then there are several cases. Letting $v$ be the anchor vertex for $e_r$, if $v$ goes down to the level below then consider the hyperedge corresponding to $e_r$ and let $e'_r$ be its bottom edge. Since we have deleted all edges in $E_i(v)$ except for $e_r$, either all of $E_{i-1}(v)$ except for $e'_r$ is uncovered or $E_{i-1}(v)$ is connected to $E_{i-1}(v'_{anchor})$ for a different anchor vertex $v'_{anchor}$ which has not yet been released at level $i$. In the first case, we can apply our recursive procedure on $E_{i-1}(v)$ and all edges connected to it. In the second case, we instead delete $e_r$ as before and go back to the level above. Again, after doing this, $v$ is now released at level $i$.

If $v$ does not go down to the level below (or we are already at the bottom) then the hyperedge coresponding to $e_r$ contains $v$. Moreover, by our assumption $v$ is uncovered at level $i$. Thus, $v$ is a free index for $e_r$ so we can delete $e_r$ and go back to the level above.

This procedure will succeed in the case that there are no bad vertices. We now handle bad vertices by reducing to the case where there are no bad vertices.

We consider the case where are below level $l'$ and we do not haave that $a'_{ij} = a_{ij}$, $b'_{ij} = b_{ij}$, and $c'_{ij} = c_{ij}$. We will handle these cases separately. 

If the $a'_{ij}$ are bad vertices, this must be because of equalities $a'_{ij} = a_{ij}$. We handle this by replacing each $a'_{ij}$ with a new vertex and running our procudure on this altered graph. This will cause failures when we try to use $a'_{ij}$ or $a_{ij}$ as a free index. That said, once we've tried to use all but one of a set of equal vertices, the final one will succeed, so the number of additional failures is at most $z'_i$. We can account for this using the term $\sum_{i<l:a'_{ij},b'_{ij},\text{ or } c'_{ij} \text{ is bad}}{z'_i}$. We handle bad $a_{ij},b_{ij},b'_{ij},c_{ij},c'_{ij}$ vertices in a similar manner.

In the case that $a'_{ij} = a_{ij}$, $b'_{ij} = b_{ij}$, and $c'_{ij} = c_{ij}$, if the $a'_{ij}$ and $b_{ij}$ are bad vertices, this must be because of the equalities $a'_{ij} = a_{ij}$ and $b'_{ij} = b_{ij}$. We handle this by creating a new vertex for each $a'_{ij}$, having the hyperedges between levels $i$ and $i+1$ use the old vertices, and having the hyperedges at lower levels use the new vertices. We modify $G$ so that instead of loops at level $i$, the edges involve these new vertices. This makes it so that the only anchor vertices for edges at level $i$ are the vertices $b'_{(i+1)j}$. Since all edges of $G$ now have a unique anchor, the recursive procedure succeeds. We can accomplish this with the terms $\sum_{i<l:a_{ij},b_{ij},\text{ or } c_{ij} \text{ is bad}}{z_i}
+\sum_{i<l:a'_{ij},b'_{ij},\text{ or } c'_{ij} \text{ is bad}}{z'_i}$.

We consider level $l'$ separately. If the bottom of level $l'$ contains bad vertices, we cannot make these vertices distinct. However, if this happens then we have loops in $G$ for the bottom triangles at level $l'$. These triangles are distinct from the triangles on top at level $l'$. 

We handle this by using $w_{triangle}$ to delete edges from $G$ at this level so that each component contains at most one loop. When we run the procedure, we can use $w_{triangle}$ when $E_{i-1}(v)$ is connected to a loop as well as when it is connected to $E_{i-1}(v'_{anchor})$ for some other anchor vertex $v'_{anchor}$ which has not been released at level $i$. This allows us to process each component of $G$ at level $i$ until we are left with either a covered edge or a single loop, both of which can be handled by our procedure.
\end{proof}
\subsection{Counting intersection patterns and random partitioning}\label{randompartitionsection}
There are two pieces left to add. First, all of our analysis so far was for a given intersection pattern. We must sum over all intersection patterns.
\begin{lemma}
For all $i$, there are at most $\binom{2q}{z_i}(z_i)^{2q-z_i} \leq 2^{2q}(2q)^{2q-z_i}$ choices for which $R_{i}$-triangles are equal to each other.
\end{lemma}
\begin{proof}
To specify a partition of the 2q $R_{i}$-triangles into $z_i$ parts, we specify which triangles are distinct from all previous triangles. There are $\binom{2q}{z_i}$ choices for which triangles these are. For the remaining triangles, we specify which previous triangle they are equal to. There are at most $(z_i)^{2q-z_i}$ choices for this.
\end{proof}
In our bound, we can group this with the other factors corresponding to the $R_{i}$-triangles. Since we take $q$ to be $O(logn)$, this is fine as $m$ has a $log(n)$ factor.

Second, we justify our assumption that 
\begin{enumerate}
\item For all $i \in [0,l']$, we either have that $a'_{ij} = a_{ij}$ for all $j \in [1,2q]$ or $a'_{ij} \neq a_{ij}$ for all $j \in [1,2q]$.
\item For all $i \in [0,l']$, we either have that $b'_{ij} = b_{ij}$ for all $j \in [1,2q]$ or $b'_{ij} \neq b_{ij}$ for all $j \in [1,2q]$.
\item For all $i \in [0,l']$, we either have that $c'_{ij} = c_{ij}$ for all $j \in [1,2q]$ or $c'_{ij} \neq c_{ij}$ for all $j \in [1,2q]$.
\end{enumerate}
To achieve this, instead of looking at the entire matrix $\sum_{a}{A_a \otimes B^T_a}$, we split it into parts based on the equalities/inequalities we're looking at. To obtain the case where indices $a$ and $a'$ are always equal,we just restrict ourselves in $\sum_{a}{A_a \otimes B^T_a}$ to the terms where this is the case. To obtain the case where indices $a$ and $a'$ are never equal, we choose a random partition $V,V^c$ of the indices and restrict ourselves in $\sum_{a}{A_a \otimes B^T_a}$ to the terms where $a \in V$ and $a' \in V^c$. If there are multiple indices that we wish to fork over, we apply this argument to each one (choosing the vertex partitions independently).

This construction has the property that if we take the expectation over all the possible vertex partitions, we obtain a constant times the part of $\sum_{a}{A_a \otimes B^T_a}$ we are interested in. Using this, it can be shown that probabilistic nrom bounds on these restricted matrices imply probabilistic norm bounds on the original matrix. For details, see Lemma 27 of ``Bounds on the Norms of Uniform Low Degree Graph Matrices''. From the above subsections, we have probabilistic norm bounds on the restricted matrices and the result follows.
\subsection{Other Cross Terms}
In this subsection, we sketch how the argument differs when $B = X$ rather than $B = \bar{R}_{\Omega_0}X$ or $B = P'_0\bar{R}_{\Omega_0}X$. 
\begin{theorem}\label{xcrosstermnormboundtheorem}
There is an absolute constant $C$ such that for any $\alpha > 1$ and $\beta > 0$, 
$$Pr\left[||\sum_{a}{A_a \otimes X^T}|| > \alpha^{-(l+1)}\right] < n^{-\beta}$$
as long as 
\begin{enumerate}
\item $r\mu \leq \min{\{n_1,n_2,n_3\}}$
\item $m > C\alpha\beta\mu^{\frac{3}{2}}r\sqrt{n_1}\max{\{n_2,n_3\}}log(\max{\{n_1,n_2,n_3\}})$
\item $m > C\alpha\beta\mu^{2}r\max{\{n_1,n_2,n_3\}}log(\max{\{n_1,n_2,n_3\}})$
\end{enumerate}
\end{theorem}
\begin{proof}[Proof sketch:]
The terms from $X$ directly are 
$$\prod_{j=1}^{2q}{X_{a'_{0j}b'_{0j}c'_{0j}}} = \prod_{j=1}^{2q}{\left(\sum_{i_{j}}u_{{i_j}a'_{0j}}v_{{i_j}b'_{0j}}w_{{i_j}c'_{0j}}\right)}$$

Note that the $R_{\Omega}$ factors are completely independent of the $b'_{0j}$ and $c'_{0j}$ indices. Thus, we can sum over the $b'_{0j}$ and $c'_{0j}$ indices first. When we do, this zeros out all terms except the ones where all of the $i_j$ are equal. Moreover, all of the $v$ and $w$ terms sum to $1$. The $u_{i_j}$ terms can be bounded by $\left(\frac{\mu}{n_1}\right)^q$. We now compare the bound we had before with the bound we have here.

For $\bar{R}_{\Omega_0}X$ we had factors
\begin{enumerate}
\item $(\frac{r\mu^3}{{n_1}{n_2}{n_3}})^{2q}$ from the $X$-hyperedges.
\item $\frac{{n_2}{n_3}}{\mu^2}\left(\frac{{n_1}(\max{\{n_2,n_3\}})^2}{\mu^3}\right)^q$ from the $q$ potential $a$ indices and $2q+2$ potential $b$ or $c$ indices.
\item $\left(\frac{{n_1}{n_2}{n_3}}{m}\right)^{2q}$ from the $2q$ potential distinct triangles.
\item $\left(r \cdot \frac{\mu^{\frac{3}{2}}}{\sqrt{n_1}\max{\{n_2,n_3\}}} \cdot \frac{{n_1}{n_2}{n_3}}{m}\right)^{2q-z_l-z'_l} \leq \left(\frac{r{\mu^{3/2}}\sqrt{n_1}\max{\{n_2,n_3\}}}{m}\right)^{2q-z_l-z'_l}$ from the actual number of distinct triangles, the corresponding reduced maximum number of potential new indices, and the factors of $r$ which we take from $r^{\text{\# of doubled hyperedges we bound with no free index}}$
\item $\left(\frac{\mu}{n_1}\right)^{\Delta_{ll'a}}\left(\frac{\mu}{\max{\{n_2,n_3\}}}\right)^{\Delta_{lbc} + \Delta'_{l'bc}} \leq 
\left(\frac{\mu}{\min{\{n_1,n_2,n_3\}}}\right)^{\Delta_{ll'a}+\Delta_{lbc}+\Delta'_{l'bc}}$ from the actual number of new indices which we have
\end{enumerate}
We now have the following factors instead:
\begin{enumerate}
\item $(\frac{r\mu^3}{{n_1}{n_2}{n_3}})^{q}r(\frac{\mu}{n_1})^{q}$ from the $X$-hyperedges.
\item $\frac{\max{\{n_2,n_3\}}}{\mu}\left(\frac{{n_1}\max{\{n_2,n_3\}}}{\mu^2}\right)^q$ from the $q$ potential $a$ indices and $q+1$ potential $b$ or $c$ indices.
\item $\left(\frac{{n_1}{n_2}{n_3}}{m}\right)^{q}$ from the $q$ potential distinct triangles.
\item $\left(r \cdot \frac{\mu^{2}}{n_1\max{\{n_2,n_3\}}} \cdot \frac{{n_1}{n_2}{n_3}}{m}\right)^{q-z_l} \leq \left(\frac{r{\mu^{2}}\max{\{n_2,n_3\}}}{m}\right)^{q-z_l}$ from the actual number of distinct triangles, the corresponding reduced maximum number of potential new indices, and the factors of $r$ which we take from $r^{\text{\# of doubled hyperedges we bound with no free index}}$
\item $\left(\frac{\mu}{n_1}\right)^{\Delta_{ll'a}}\left(\frac{\mu}{\max{\{n_2,n_3\}}}\right)^{\Delta_{lbc} + \Delta'_{l'bc}} \leq 
\left(\frac{\mu}{\min{\{n_1,n_2,n_3\}}}\right)^{\Delta_{ll'a}+\Delta_{lbc}+\Delta'_{l'bc}}$ from the actual number of new indices which we have
\end{enumerate}
The difference is in the first three terms, grouping these terms together gives 
$$\frac{r\max{\{n_2,n_3\}}}{\mu}\left(\frac{r\mu^2\max{\{{n_2},{n_3}\}}}{m}\right)^q$$
By our assumption, $m \geq Cr\mu^2\max{\{{n_2}{n_3}\}}log(n)^2$ so we are fine.
\end{proof}

\end{document}